\documentclass[11pt,english]{article}
\usepackage[T1]{fontenc}
\usepackage[latin9]{inputenc}
\usepackage[a4paper]{geometry}
\usepackage{color}
\usepackage{babel}
\usepackage{mathtools}
\usepackage{amsmath}
\usepackage{amsthm}
\usepackage{amssymb}
\usepackage{booktabs}
\usepackage[margin=1cm,font={small}]{caption}
\usepackage{adjustbox}
\usepackage{tabularx}
\newcolumntype{Y}{>{\centering\arraybackslash}X}
\usepackage[square, numbers]{natbib}
\usepackage[unicode=true,pdfusetitle,
 bookmarks=true,bookmarksnumbered=false,bookmarksopen=false,
 breaklinks=true,pdfborder={0 0 0},pdfborderstyle={},backref=false,colorlinks=true]
 {hyperref}
 \usepackage{qtree}
 \usepackage[edges,linguistics]{forest}
 \usetikzlibrary{shadows.blur,shapes.multipart}
\tikzset{grimsel/.style={rectangle split,rectangle split parts=1,draw,
    fill=white,blur shadow,shadow scale = 0.95,rounded corners,align=center}}

\usepackage{soul}

\usepackage{authblk}

 \usepackage{todonotes}

\makeatletter
\theoremstyle{plain}
\newtheorem{prop}{\protect\propositionname}
\theoremstyle{plain}
\newtheorem{lem}{\protect\lemmaname}
\theoremstyle{plain}
\newtheorem{thm}{\protect\theoremname}
\usepackage{amsmath, amsthm, amssymb} 
\DeclareMathOperator*{\argmin}{argmin}

\newcommand{\til}[1]{\widetilde{#1}}
\newcommand{\m}[1]{\mathbf{#1}}
\newcommand{\mr}[1]{\mathrm{#1}}
\newcommand{\mc}[1]{\mathcal{#1}}
\newcommand{\mb}[1]{\mathbb{#1}}
\newcommand{\mbs}[1]{\boldsymbol{#1}}
\newcommand{\rank}{k_{\mathrm{max}}}
\newcommand{\OP}{O_{\mathbb{P}}}
\newcommand{\Adjfrac}[2]{\dfrac{\vphantom{\Big(}#1}{\vphantom{\Big(}#2}}
\newcommand{\adjfrac}[2]{\dfrac{\vphantom{\big(}#1}{\vphantom{\big(}#2}}
\newtheorem{assump}{}

\makeatother

\providecommand{\lemmaname}{Lemma}
\providecommand{\propositionname}{Proposition}
\providecommand{\theoremname}{Theorem}

\begin{document}

\renewcommand{\arraystretch}{1.25}
\newcommand{\NodeTitle}[2][]{#2\nodepart[align=left,text width={width("#2")}]{two}}

\title{Generalisation and benign over-fitting for linear regression onto random functional covariates}
\author[1]{Andrew Jones}
\author[2]{Nick Whiteley}
\affil[1]{School of Mathematics, University of Edinburgh}
\affil[2]{School of Mathematics, University of Bristol}

\maketitle
\abstract{We study  theoretical predictive performance of ridge and ridge-less least-squares regression when covariate vectors arise from evaluating $p$ random, means-square continuous functions over a latent metric space at $n$ random and unobserved locations, subject to additive noise. This leads us away from the standard assumption of i.i.d. data to a setting in which the $n$ covariate vectors are exchangeable but not independent in general. Under an assumption of independence across dimensions, $4$-th order moment, and other regularity conditions, we obtain probabilistic bounds on a notion of predictive excess risk adapted to our random functional covariate setting, making use of recent results of \citet{barzilai2023generalization}. We derive convergence rates in regimes where $p$ grows suitably fast relative to $n$,  illustrating interplay between ingredients of the model in determining convergence behaviour and the role of additive covariate noise in benign-overfitting.}

\section{Introduction}
In studies of the theoretical predictive performance of supervised Machine Learning methods it is very commonly assumed that training data in the form of covariate/response pairs $\{(\mathbf{x}_{i},y_{i});i=1,\ldots,n\}$ and test data $(\mathbf{x}_{test},y_{test})$  are i.i.d.   For ridge regression and ridge-less least squares regression, the assumption of i.i.d. data is standard in studies of the classical regime where $n\to\infty$, with $\mathbf{x}_i\in\mathbb{R}^p$ for $p<n$  or with covariate vectors valued in a possibly infinite-dimensional Hilbert space \citep{zhang2005learning,smale2007learning, caponnetto2007optimal,hsu2014random,mourtada2022elementary}. The i.i.d. assumption is also prominent in the emerging literature on the `benign overfitting' phenomenon \cite{bartlett2020benign, hastie2022surprises, belkin2020two,tsigler2023benign,cheng2022dimension,cheng2023theoretical,barzilai2023generalization}, where, under conditions on the eigenvalues of the covariance matrix of covariate vectors, least-norm solutions of the over-parameterised least squares problem which interpolate --- i.e., fit training data exactly --- when $p\gg n$, can have good predictive performance. These high-dimensional settings do not involve sparsity as in the very well known $\ell_1$-penalised settings of, e.g., \cite{CandesAndTao,vandeGeer,buhlmann2011statistics,hastie2015statistical}. 

Studies of non-sparse, high-dimensional linear regression when there is some form of dependence across samples have appeared recently: \cite{nakakita2022benign} consider benign over-fitting for linear regression in a time series model with stationary, centered, Gaussian covariates and noise; \cite{moniri2024asymptotics} consider ridge regression with linearly-dependent, non-Gaussian data in a setting where $n$ and $p$ grow proportionally;  \cite{atanasov2024risk} assume covariates are generated from a Gaussian process with spatio-temporal covariance; \cite{luo2024roti} assume covariates follow a right-rotationally invariant  distribution. Motivated by desire to understand the double-descent phenomena in neural networks, some recent studies \citep{hastie2022surprises,mei2022generalization} have considered a situation in which $(\mathbf{x}_i,y_i)_{i\geq 1}$ are i.i.d., but least squares linear regression of $y_i$ is performed on to mapped covariate vectors $\sigma(\mathbf{W}\mathbf{x}_i)$, where  $\mathbf{W}$ is a random matrix with i.i.d. elements, and $\sigma(\cdot)$ is some element-wise nonlinearity. This is a simple model of a neural network, with a single nonlinear layer with random weights, and a linear output layer. The mapped vectors $\sigma(\mathbf{W}\mathbf{x}_i)$, $i=1,2,\ldots,$ are exchangeable but not independent of each other, because of the presence of the random matrix $\mathbf{W}$.  

The overall aim of the present work is to study the generalisation performance of ridge and ridge-less least squares regression when simultaneously $p/n \to\infty$ and $n\to\infty$, in  the setting of the Latent Metric Model (LMM), a general form of model for high-dimensional data explored in the forthcoming JRSSB discussion paper of \citet{whiteley2022statistical}.  The interest in the LMM is that it serves as a general alternative to the assumption of i.i.d. data with rich behaviour: \citet{whiteley2022statistical} illustrated how intrinsically low-dimensional nonlinear structure can emerge in high-dimensional data from the LMM when there is independence across dimensions,  providing a statistical grounding for the so-called Manifold Hypothesis \cite{cayton2005algorithms,bengio2013representation,fefferman2016testing}. Moreover the Gaussian Process Latent Variable model of \citet{lawrence2003gaussian,lawrence2005probabilistic} is a special case of the LMM.

Under the LMM, data vectors are exchangeable but not independent in general, and arise from evaluation of a collection of $p$ random functions at $n$ random, latent, metric space-valued locations, subject to additive noise. We introduce a novel regression setup  tailored to the structure of the LMM: the domain of the unknown regression function is the latent metric space in the LMM; we assume that we do not have access to the latent variables associated with the training or test data; and there is dependence between training and test covariates arising from the random functions in the LMM.  We also work under mild finite $4$th-moment conditions which are much weaker than sub-Gaussian assumptions which are common in studies of benign overfitting, e.g., \cite{bartlett2020benign,nakakita2022benign}.

Our study is largely inspired and stimulated by numerous recent contributions regarding benign overfitting, \cite{bartlett2020benign, hastie2022surprises, belkin2020two,tsigler2023benign,cheng2022dimension,cheng2023theoretical,barzilai2023generalization}, and in particular we rely heavily on modifications to some results of \citep{barzilai2023generalization}, which are a key building block for us. The connection to \citep{barzilai2023generalization} which we explore is that, when $p$ is large and there is independence across dimensions of the random functions and additive noise in the LMM, the predictions obtained from ridge/ridge-less least squares regression can be viewed as perturbations of predictions from kernel ridge/ridge-less regression, where the kernel is what we shall call the \emph{implicit kernel}, defined by the ingredients of the LMM.  Moreover, we show how implicit regularisation of the regression problem can arise from the ingredients of the LMM as $p\to\infty$. The results of \citet{barzilai2023generalization} are also  useful for us because they rely on realistic assumptions on the kernel in question. In particular we note that \cite[Sec 2.2]{barzilai2023generalization} sets out in detail the ways in which their setup substantially loosens restrictive assumptions in prior work.  

Perturbations of kernel ridge regression have been considered in the literature on Random Fourier Features (RFF), e.g.,  \citep{rahimi2007random,avron2017random, li2021towards}. RFF is an approach to ameliorating the cost of kernel methods using randomised functional  approximations to kernel Gram matrices. Both the mathematical details and motivation of RFF are different to our setting: in RFF the Gram matrix approximations arise by the user sampling from the spectral measure (or some closely related measure over frequencies) in the Bochner's theorem representation of the kernel (the kernel is chosen to be shift invariant), with the aim of controlling computational cost. In our setting we make no assumptions on the functional form of the kernel. Instead, we take the perspective that the LMM is nature's data generating mechanism, rather than being an engineered sampling device which the user controls, and the kernel we consider is defined implicitly through the LMM and not something the user of the regression method is free to choose.

The rest of this article is structured as follows.
\begin{itemize}
\item In Section \ref{sec:model_and_ass} we introduce the LMM and some special cases thereof. We explain key differences in our regression setup and definition of prediction error compared to the standard setup of i.i.d. data.  
\item In Section \ref{sec:main_theorems} we present our prediction error decomposition and main results, Theorems \ref{thm:VB_overall_bound}, \ref{thm:VB_finite_rank} and \ref{thm:R_overall_bound} which give bounds on the terms in the decomposition. The proofs of Theorems \ref{thm:VB_overall_bound} and \ref{thm:VB_finite_rank} hinge on carefully modifying proofs of state-of-the-art  error bounds for kernel ridge regression due to \citet{barzilai2023generalization}.  
\item In Section \ref{sec:interpretation} we interpret and apply Theorems \ref{thm:VB_overall_bound}, \ref{thm:VB_finite_rank} and \ref{thm:R_overall_bound}, across a taxonomy of six regularisation scenarios -- see Figure \ref{fig:tree_intro}. Across these scenarios, regularisation can arise implicitly from the additive noise component of the LMM, or explicitly from the ridge regression objective function. For each of these  scenarios, we derive simplified expressions for convergence rates which exhibit trade-offs between $n$, $p$ and various other parameters.
\item In Section \ref{sec:numerical_results} we provide empirical examples to support our theoretical bounds, by translating a scenario presented by Tsigler and Bartlett \citet{tsigler2023benign} into the LMM setting, and demonstrating the asymptotic decay of the prediction error in this example under both implicit and explicit regularisation regimes.  Finally, we provide a real-world example using temperature time series data from the Berkeley Earth project \cite{berkelyearth}, in which we conducting a regression analysis to predict the latitude of a number of cities, observing increased predictive performance as the length of the time series increases.
\end{itemize}
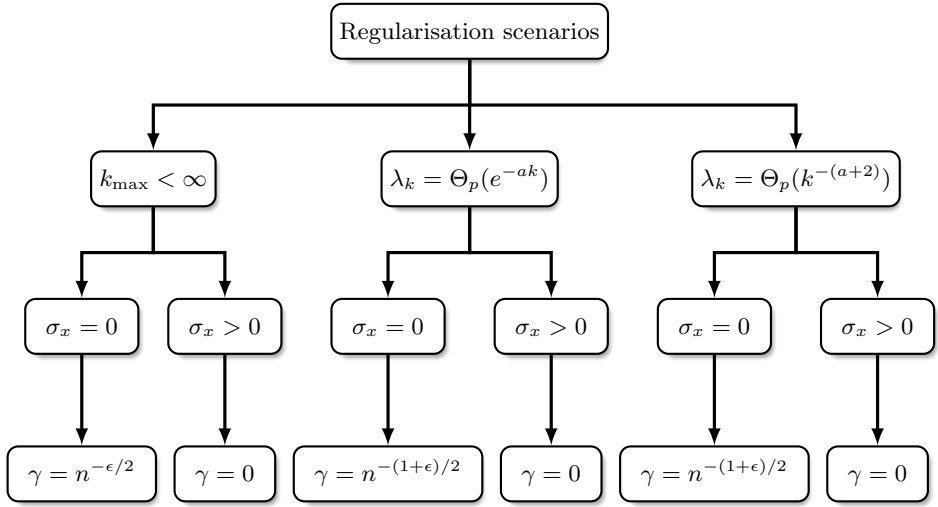
\begin{figure}[h!]
\centering
{\footnotesize
\begin{forest}
  forked edges,
    for tree={grimsel,thick,edge={-latex,very thick},l sep = 12mm,fork sep=6mm}
  [{Regularisation scenarios} 
    [{$\rank<\infty$}
        [{$\enspace\sigma_x = 0\enspace$}[{$\enspace\gamma=n^{-\epsilon/2}\enspace$}]]
        [{$\enspace\sigma_x > 0\enspace$}[{$\enspace\gamma=0\enspace$}]]
    ]
    [{$\lambda_k=\Theta_p(e^{-ak})$},calign with current 
        [{$\enspace\sigma_x=0\enspace$}[{$\enspace\gamma= n^{-(1+\epsilon)/2}\enspace$}]]
        [{$\enspace\sigma_x>0\enspace$}[{$\enspace\gamma= 0\enspace$}]]
    ]
    [{$\lambda_k=\Theta_p(k^{-(a+2)})$} 
        [{$\enspace\sigma_x=0\enspace$}[{$\enspace\gamma= n^{-(1+\epsilon)/2}\enspace$}]]
        [{$\enspace\sigma_x>0\enspace$}[{$\enspace\gamma= 0\enspace$}]]
    ]
  ]
\end{forest}
}
\caption{The regularisation scenarios explored in Section \ref{sec:interpretation}. Here $\rank$ is the rank of the implicit kernel in the LMM, $(\lambda_k)_{k\geq 1}$ are its eigenvalues, $\sigma_x$ is the level of additive covariate noise in the LMM, $\gamma$  is the explicit ridge regularisation parameter and $\Theta_p(\cdot)$ means asymptotically equivalent with constants independent of dimension $p$. In each of these scenarios and under corresponding assumptions on how quickly $p$ grows with $n$, we obtain convergence rates for the mean-square prediction error of ridge regression. In the scenarios shown where $\gamma=0$, we shall see that benign overfitting occurs.}\label{fig:tree_intro}
\end{figure}

\section{Model and assumptions}\label{sec:model_and_ass}

\subsection{Basics of ridge regression with i.i.d. data}\label{subsec:iid}
In order to highlight the unusual aspects of our setup and fix notation, we first recall some elementary aspects of ridge regression. Given covariate-response
pairs $\{(\mathbf{x}_{i},y_{i});i=1,\ldots,n\}$, where   $\mathbf{x}_{i}\in\mathbb{R}^{p}$
and $y_{i}\in\mathbb{R}$, written in matrix-vector form $\mathbf{X}\equiv[\mathbf{x}_{1}|\cdots|\mathbf{x}_{n}]^{\top}\in\mathbb{R}^{n\times p}$
and $\mathbf{y}=[y_{1}\,\cdots\,y_{n}]^{\top}\in\mathbb{R}^{n},$
ridge regression involves solving
\begin{equation}
\argmin_{\beta\in\mathbb{R}^{p}}\frac{1}{n}\|\mathbf{y}-\mathbf{X}\beta\|_{2}^{2}+\gamma\|\beta\|_{2}^{2}.\label{eq:vanilla_ridge}
\end{equation}
The solution in the case $\gamma>0$ is: 
\[
\hat{\beta}\coloneqq\mathbf{X}^{\top}\left(\mathbf{X}\mathbf{X}^{\top}+n\gamma\right)^{-1}\mathbf{y}=\left(\mathbf{X}^{\top}\mathbf{X}+n\gamma\right)^{-1}\mathbf{X}^{\top}\mathbf{y}.
\]
In the unregularised, a.k.a. ridge-less, case where $\gamma=0$, if $\mathbf{X}\mathbf{X}^\top$ is invertible then the vector $\mathbf{X}^{\top}(\mathbf{X}\mathbf{X}^\top)^{-1}\mathbf{y}$ has the minimum $\|\cdot\|_2$ norm across all vectors $\beta$ such that $\mathbf{y}=\mathbf{X}\beta$, i.e. it is the least-norm interpolating solution. 

In analyses of the predictive performance of ridge regression it is typically assumed
that the pairs $(\mathbf{x}_{i},y_{i})_{i\geq1}$, are i.i.d.    \cite{zhang2005learning,smale2007learning, caponnetto2007optimal,hsu2014random,mourtada2022elementary}. For a response model of the form: 
\begin{equation}
y_{i}=\mathbf{x}_{i}^{\top}\beta^{*}+\epsilon_{i},\label{eq:iid_response_model}
\end{equation}
where $\beta^{*}$ is the true parameter value and the $\boldsymbol{\epsilon}_{i}$ are i.i.d. and zero mean, the
study of generalisation proceeds by introducing independent copies
$\mathbf{x}_{test},\epsilon_{test}$ of $\mathbf{x}_{i},\epsilon_{i},$
setting $y_{test}=\mathbf{x}_{test}^{\top}\beta^{*}+\epsilon_{test}$,
and comparing the prediction $\mathbf{x}_{test}^{\top}\hat{\beta}$
to the ideal value $\mathbf{x}_{test}^{\top}\beta^{*}$ in terms of
the excess risk:
\begin{multline}
\mathbb{E}_{\mathbf{x}_{test},\epsilon_{test},\mathbf{\epsilon}}\left[\left|y_{test}-\mathbf{x}_{test}^{\top}\hat{\beta}\right|^{2}\right]-\mathbb{E}_{\mathbf{x}_{test},\epsilon_{test},\mathbf{\epsilon}}\left[\left|y_{test}-\mathbf{x}_{test}^{\top}\beta^{*}\right|^{2}\right]\\
=\mathbb{E}_{\mathbf{x}_{test},\mathbf{\epsilon}}\left[\left|\mathbf{x}_{test}^{\top}\hat{\beta}-\mathbf{x}_{test}^{\top}\beta^{*}\right|^{2}\right]\label{eq:excess_risk_iid}
\end{multline}
where $\mathbb{E}_{\mathbf{x}_{test},\epsilon_{test},\mathbf{\epsilon}}[\cdot]$
denotes conditional expectation given everything \emph{except $\mathbf{x}_{test}$,
$\epsilon_{test}$} and $\boldsymbol{\epsilon}_{1},\cdots,\boldsymbol{\epsilon}_{n}$,
and $\mathbb{E}_{\mathbf{x}_{test},\mathbf{\epsilon}}[\cdot]$ is
defined similarly. This conditional expectation is thus a function
of the random covariates $\mathbf{x}_{1},\ldots,\mathbf{x}_{n}$.

\subsection{Definition of the Latent Metric Model and the regression problem}\label{sec:LMM_and_reg_model}

The Latent Metric Model \citep{whiteley2022statistical} for random vectors $\mathbf{x}_1,\ldots,\mathbf{x}_n\in\mathbb{R}^p$ is:
\begin{equation}
\mathbf{x}_{i}=\boldsymbol{\psi}(z_{i})+\sigma_{x}\mathbf{e}_{i},\qquad i=1,\ldots,n,\label{eq:x_i_model}
\end{equation}
comprising three independent sources of randomness:
\begin{itemize}
\item $z_{1},\ldots,z_{n}$ are unobserved, i.i.d. random elements of a
latent metric space $\mathcal{Z}$, with common distribution $\mu$
which is a Borel probability measure whose support is $\mathcal{Z}$. 
\item $\boldsymbol{\psi}(\cdot)=[\psi_{1}(\cdot)\,\cdots\,\psi_{p}(\cdot)]^{\top}$,
where each $\psi_{j}(\cdot)$ is an $\mathbb{R}$-valued random function
with domain $\mathcal{Z}$; that is for each $z\in\mathcal{Z}$ ,
$\psi_{j}(z)$ is a random variable.
\item $\mathbf{e}_{1},\ldots,\mathbf{e}_{n}$ are i.i.d. random vectors
in $\mathbb{R}^{p}$, the elements of each $\mathbf{e}_{i}$ are independent,
zero-mean and unit variance, and $\sigma_{x}\geq0$.
\end{itemize}
In this setting the $\mathbf{x}_{i}$ are exchangeable across $i$,
but not independent in general because the random functions $\psi_{j}$
appear in the definitions of all the $\mathbf{x}_{i}$. We shall write
(\ref{eq:x_i_model}) in matrix form
\begin{equation}
\mathbf{X}=\boldsymbol{\Psi}+\sigma_{x}\mathbf{E},\label{eq:X_model}
\end{equation}
with $n\times p$ matrices $\mathbf{X}\equiv[\mathbf{x}_{1}|\cdots|\mathbf{x}_{n}]^{\top}$,
$\boldsymbol{\Psi}\equiv[\boldsymbol{\psi}(z_{1})|\cdots|\boldsymbol{\psi}(z_{n})]^{\top}$
and $\mathbf{E}\equiv[\mathbf{e}_{1}|\cdots|\mathbf{e}_{n}]^{\top}$. 

In the remainder of the present work we assume, in contrast to  Section \ref{subsec:iid}, that covariate vectors $\mathbf{x}_i$ follow the Latent
Metric Model and the response variables $y_i$ follow
\begin{equation}
y_{i}=g(z_{i})+\epsilon_{i},\label{eq:response_model}
\end{equation}
where $g:\mathcal{Z}\to\mathbb{R}$ is an unknown function, $\epsilon_{1},\ldots,\epsilon_{n}$ are i.i.d., zero-mean
random variables with finite variance $\sigma_{y}^2\geq0$. We shall write $\boldsymbol{\epsilon}\equiv[\epsilon_{1}\,\cdots\,\epsilon_{n}]^{\top}$.

We consider a prediction problem in the setting (\ref{eq:X_model})-(\ref{eq:response_model}),
where we have access to $(\mathbf{X},\mathbf{y})$ as training data, but $z_{1},\ldots,z_{n}$
are hidden from us. Our predictions are defined in terms of ridge/ridge-less regression of $\mathbf{y}$
onto $\mathbf{X}$, that is  with $\gamma\geq0$ we consider the penalised least squares problem, 
\begin{equation}
\argmin_{\beta\in\mathbb{R}^{p}}\frac{1}{n}\|\mathbf{y}-p^{-1/2}\mathbf{X}\beta\|_{2}^{2}+\gamma\|\beta\|_{2}^{2},\label{eq:theta_optimisation-1}
\end{equation}
whose solution is: 
\begin{equation}\label{eq:hat_beta_defn_lmm}
\hat{\beta}(\mathbf{y})\coloneqq p^{-1/2}\mathbf{X}^{\top}\left(p^{-1}\mathbf{X}\mathbf{X}^{\top}+n\gamma\right)^{-1}\mathbf{y}.
\end{equation}
The dependence of $\hat{\beta}(\mathbf{y})$ on $\mathbf{y}$ is a
notational convenience for use in our proofs and the significance
of the $p^{-1/2}$ scaling in (\ref{eq:theta_optimisation-1}) will
become clear later. Although we have used the same notation $\hat{\beta}$
for the solution of the penalised least-squares problem as in section
\ref{subsec:iid}, in the setup (\ref{eq:X_model})-(\ref{eq:response_model})
  we stress that we do not define a true parameter ``$\beta^{*}$''
(although we shall introduce a different notion of true parameter related to the function $g$ in Assumption \ref{ass:g} below).  Nevertheless we shall use $\hat{\beta}$ to define predictions and
the following notion of generalisation error: we introduce $z_{test}$
and $\mathbf{e}_{test}$ which are independent copies of respectively
$z_{i}$ and $\mathbf{e}_{i}$ and define:
\[
\mathbf{x}_{test}\coloneqq\boldsymbol{\psi}(z_{test})+\sigma_{x}\mathbf{e}_{test}.
\]
We stress here that $\boldsymbol{\psi}(\cdot)$ is the same vector
of random functions as appears in (\ref{eq:x_i_model}), so that $\mathbf{x}_{test}$
is exchangeable with $\mathbf{x}_{1},\ldots,\mathbf{x}_{n}$, but
not independent of them in general. Upon observing $\mathbf{x}_{test}$ in addition to $(\mathbf{X},\mathbf{y})$, our objective is to predict $g(z_{test})$, where $z_{test}$ is not observed.  We take this prediction to be $p^{-1/2}\mathbf{x}_{test}^{\top}\hat{\beta}(\mathbf{y})$,
and we consider the following measure of excess risk as the analogue
of (\ref{eq:excess_risk_iid}):
\begin{equation}
\mathbb{E}_{z_{test},\mathbf{e}_{test},\boldsymbol{\epsilon}}\left[\left|p^{-1/2}\mathbf{x}_{test}^{\top}\hat{\beta}(\mathbf{y})-g(z{}_{test})\right|^{2}\right],\label{eq:excress_risk_LMM}
\end{equation}
where $\mathbb{E}_{z_{test},\mathbf{e}_{test},\boldsymbol{\epsilon}}[\cdot]$
denotes conditional expectation given everything \emph{except }$z_{test}$
, $\mathbf{e}_{test}$ and $\boldsymbol{\epsilon}$.

In summary of the unusual features of this setup:
\begin{itemize}
\item The covariate vectors we observe, $\mathbf{x}_{1},\ldots,\mathbf{x}_{n}$,
are exchangeable but not independent in general, because in the LMM all these
vectors are defined in terms of evaluations of the random functions
$\psi_{1},\ldots,\psi_{p}$.
\item Our regression function $g$ is assumed to be a function only of the
latent variable $z_{i}$, rather than the observed covariate
vectors $\mathbf{x}_{i}$.
\item Nevertheless, we perform ridge or ridge-less regression of $\mathbf{y}$ onto $p^{-1/2}\mathbf{X}$.
\item Our test covariate $\mathbf{x}_{test}$ is exchangeable with $\mathbf{x}_{1},\ldots,\mathbf{x}_{n}$
but not independent of them, because $\mathbf{x}_{test}$ is
defined in terms of evaluations of the random functions $\psi_{1},\ldots,\psi_{p}$. 
\item We consider the simple linear form $p^{-1/2}\mathbf{x}_{test}^{\top}\hat{\beta}(\mathbf{y})$
as a prediction of $g(z_{test})$, where $z_{test}$ is not observed. 
\end{itemize}

At first glance, it may seem strange that $p^{-1/2}\mathbf{x}_{test}^{\top}\hat{\beta}(\mathbf{y})$ could serve as a useful prediction of $g(z_{test})$. To get a first look at why this might work,  we write out from definitions:
$$
p^{-1/2}\mathbf{x}_{test}^{\top}\hat{\beta}(\mathbf{y}) = p^{-1}\mathbf{x}_{test}^{\top}\mathbf{X}^{\top}\left(p^{-1}\mathbf{X}\mathbf{X}^{\top}+n\gamma\right)^{-1}\mathbf{y}.
$$
Here the elements of the vector $\mathbf{x}_{test}^{\top}\mathbf{X}^{\top}$ are of the form $\mathbf{x}_{test}^{\top}\mathbf{x}_i$ and the elements of the matrix $\mathbf{X}\mathbf{X}^{\top}$ are of the form $\mathbf{x}_{i}^{\top}\mathbf{x}_j$, for $1\leq i, j\leq n$. Therefore the only way in which the training and test covariate vectors enter into the prediction is through pairwise inner-products, scaled by $p^{-1}$. Our theoretical results entail studying the behaviour of these re-scaled inner products as $p\to\infty$. Informally stated, we shall see that in this regime the stochasticity in the random functions $\psi_1,\ldots,\psi_p$ and in the noise disturbances $\mathbf{e}_1,\ldots,\mathbf{e}_n$ `averages out', and from this  inner-products amongst $\phi(z_{test})$ and  $\phi(z_1),\ldots,\phi(z_n)$ emerge, where $\phi$ is a certain feature map whose definition is given in Section \ref{sec:assump_and_prop}.

\subsection{Assumptions
and properties of the Latent Metric Model}\label{sec:assump_and_prop}
In this section we introduce Assumptions \ref{ass:compact_and_cont}-\ref{ass:eig_decay_basic}, which are taken to hold throughout the remainder of this work, and explain their significance.

\begin{assump}\label{ass:compact_and_cont}$\mathcal{Z}$ is compact
and for each $j=1,\ldots,p$,  $\psi_{j}$ is pointwise square
integrable 
and mean-square continuous, that is for all $z$, $\mathbb{E}[|\psi_{j}(z)|^{2}]<\infty$ and $\lim_{z^{\prime}\to z}\mathbb{E}\left[\left|\psi_{j}(z)-\psi_{j}(z^{\prime})\right|^{2}\right]=0$.\end{assump}

Under \ref{ass:compact_and_cont}, the Cauchy-Schwartz inequality can be used to show that the following positive definite function with domain $ \mathcal{Z}\times \mathcal{Z}$ is continuous:
\[
(z,z^{\prime}) \mapsto\frac{1}{p}\mathbb{E}\left[\left\langle \boldsymbol{\psi}(z),\boldsymbol{\psi}(z^{\prime})\right\rangle_2 \right].
\]
 We shall refer to this positive definite function as the \emph{implicit kernel} associated with the LMM. For $x=[x_1\,x_2\,\cdots]^\top\in\mathbb{R}^{\mathbb{N}}$, denote $\|x\|_2\coloneqq \left(\sum_{k\geq 1}|x_k|^2\right)^{1/2} $ and  $\ell_2\coloneqq \{x \in\mathbb{R}^{\mathbb{N}}:\|x\|_2 <\infty\}$.
Mercer's theorem \citep[Thm. 4.49]{steinwart2008support}  applied to the implicit kernel and the measure $\mu$ yields:
\begin{equation}
\frac{1}{p}\mathbb{E}\left[\left\langle \boldsymbol{\psi}(z),\boldsymbol{\psi}(z^{\prime})\right\rangle_2 \right]=\left\langle \phi(z),\phi(z^{\prime})\right\rangle _{2}=\sum_{k\geq1}\lambda_{k}u_{k}(z)u_{k}(z^{\prime}).\label{eq:mercers}
\end{equation}
Here $\phi:\mathcal{Z}\to\ell_2$ is given by $\phi(z)\coloneqq[\lambda_{1}^{1/2}u_{1}(z)\;\lambda_{2}^{1/2}u_{2}(z)\;\cdots\;]^{\top}$,
$(u_{k};k\geq1)$ are an orthonormal basis of eigenfunctions in $L_{2}(\mu)$
associated with the kernel integral operator, $\lambda_{1}\geq\lambda_{2}\geq\cdots\geq 0$
are corresponding nonnegative eigenvalues, and in (\ref{eq:mercers}) the convergence
is absolute and uniform.  We denote by $\rank$ the number of non-zero eigenvalues, with $\rank\coloneqq \infty$ if all the eigenvalues are strictly positive.

We next introduce an assumption about the regression function in \eqref{eq:response_model}.
\begin{assump}\label{ass:g}There exists $\theta^{*}=[\theta_1^*\;\theta_2^*\;\cdots\;]^\top$ with $\|\theta^{*}\|_{2}<\infty$
such that $g(\cdot)=\langle\phi(\cdot),\theta^{*}\rangle_2$, with the convention that if $\rank<\infty$, then $\theta^*_k=0$ for $k>\rank$. \end{assump}

Writing out the inner product   in  \ref{ass:g} gives:
$$
g(z)=\langle\phi(z),\theta^{*}\rangle_2=\sum_{k\geq 1}\lambda_k^{1/2}\theta^*_k u_k(z),
$$
so \ref{ass:g} says that $g$ can be expanded onto the basis $(u_k)_{k\geq 1}$, with expansion coefficients which decay suitably quickly. This is a standard type of assumption in studies of kernel ridge regression, e.g., \citep{barzilai2023generalization}. In abstract terms, this means $g$ is a member of the Reproducing Kernel
Hilbert Space associated with $\phi$, although we shall not need
to introduce explicit details of this Hilbert space. 

Note that the implicit kernel, hence $(\lambda_{k},u_{k})_{k\geq1}$
and in turn $\phi$, depend on $p$ in general, but this is not shown in the notation.
Similarly if $g$ is considered to be fixed across $p$ then $\theta^{*}$
depends on $p$ in general. It may be useful to keep in mind the special case in which 
the random functions $(\psi_{j})_{j\geq1}$ are identically distributed; in this situation  $(\lambda_{k},u_{k})_{k\geq1}$ and $\phi$ do not depend on $p$.

In order to sketch the ideas underlying our analysis of the excess risk \eqref{eq:excress_risk_LMM} and motivate our remaining assumptions \ref{ass:fourth_moment} and \ref{ass:ind} below, 
let us consider the property of the LMM that:
\begin{equation}\label{eq:cond_exp_id}
\frac{1}{p}\mathbb{E}\left[\langle \mathbf{x}_i,\mathbf{x}_j \rangle_2|z_i,z_j\right] = \left\langle \phi(z_i),\phi(z_j)\right\rangle_{2} +  \mathbf{I}[i=j] \sigma_x^2.
\end{equation} In this sense the inner-products between feature-mapped latent variables $\phi(z_i),\phi(z_j)$ determine the (conditional) expected inner-products between data vectors $\mathbf{x}_i,\mathbf{x}_j$, up to some  distortion in the case $i=j$ depending on the noise level $\sigma_x$.  We shall write the relation \eqref{eq:cond_exp_id} in matrix form  as
\begin{equation}\label{eq:XX^T_cond_exp}
\frac{1}{p}\mathbb{E}[\mathbf{X
}\mathbf{X
}^\top |z_1,\ldots,z_n ] = \boldsymbol{\Phi}\boldsymbol{\Phi}^\top +  \mathbf{I}_n \sigma_x^2,
\end{equation}
where $\boldsymbol{\Phi}\equiv[\phi(z_{1})|\cdots|\phi(z_{n})]^{\top}$ and $\mathbf{I}_n$ is the $n\times n$ identity matrix.  Recalling the matrix $p^{-1}\mathbf{X
}\mathbf{X
}^\top$ appears in \eqref{eq:hat_beta_defn_lmm}, part of our proofs will entail showing that this matrix is concentrated about its conditional expectation \eqref{eq:XX^T_cond_exp}. The two following assumptions will be used to establish this.
\begin{assump}\label{ass:fourth_moment}$\sup_{j\geq1}\sup_{z\in\mathcal{Z}}\mathbb{E}\left[\left|\psi_{j}(z)\right|^{4}\right]<\infty$
and $\sup_{i,j\geq1}\mathbb{E}\left[\left|\mathbf{E}_{ij}\right|^{4}\right]<\infty$.\end{assump}

\begin{assump}\label{ass:ind}The random functions $(\psi_{j})_{j\geq1}$
are mutually independent. \end{assump}
We stress that we do \emph{not} require that the $\psi_j$ have zero mean functions, that is, for each $j$ and $z$, we do not require $ \mathbb{E}[\psi_j(z)]=0$. Thus it may be useful to think of \ref{ass:ind} as meaning that for each $z\in\mathcal{Z}$, $\psi_j(z)=\mathbb{E}[\psi_j(z)] + \Delta_j^{\psi}(z)$, for random functions $\Delta_j^{\psi}(\cdot)$ which are independent across $j$ and satisfy $\mathbb{E}[\Delta_j^{\psi}(z)]=0$ for all $z\in\mathcal{Z}$.

 The independence in \ref{ass:ind} will allow us to decompose $p^{-1}\mathbf{X}\mathbf{X}^\top$ as the arithmetic mean of $p$ conditionally independent, rank-one matrices. The mild uniform moment Assumption \ref{ass:fourth_moment} will guarantee that the variance of the elements of these rank-one matrices is finite.

In order to decompose \eqref{eq:excress_risk_LMM} we shall exploit a relationship between $\mathbf{X
}\equiv[\mathbf{x}_1|\cdots|\mathbf{x}_n]^\top$ and $\boldsymbol{\Phi}$ which we shall now explain. Let $\mathbf{W}$ be the random matrix with $p$ rows and infinitely many columns
defined in the case $\rank=\infty$ by:
\begin{equation}\label{eq:W_jk_defn}
\mathbf{W}_{jk}\coloneqq
(\lambda_{k}p)^{-1/2}\int_{\mathcal{Z}}\psi_{j}(z)u_{k}(z)\mu(\mathrm{d}z),\
\end{equation}
for $k\in\mathbb{N}$,  and defined in the case $\rank<\infty$ by the same expression for $k\leq \rank$ and $\mathbf{W}_{jk}\coloneqq 0$ for $k>\rank$. 

The following proposition is a refinement of \citet[Prop.1]{whiteley2022statistical}, invoking Assumption \ref{ass:eig_decay_basic} in order to establish the desired almost sure inequality.
\begin{assump}\label{ass:eig_decay_basic}The eigenvalues of the implicit kernel satisfy $\sum_{k\geq 1}k\lambda_k <\infty $. \end{assump}
\begin{prop}
\label{prop:KL_expansion}When assumptions \ref{ass:compact_and_cont} and \ref{ass:eig_decay_basic} hold,
\[
\mathbf{X}\stackrel{a.s.}{=}p^{1/2}\boldsymbol{\Phi}\mathbf{W}^{\top}+\sigma_{x}\mathbf{E},\qquad\mathbf{x}_{test}\stackrel{a.s.}{=}p^{1/2}\mathbf{W}\phi(z_{test})+\sigma_{x}\mathbf{e}_{test},\qquad\mathbb{E}[\mathbf{W}^{\top}\mathbf{W}]=\mathbf{I}_{r}.
\]
\end{prop}
The proof is in Section \ref{sec:proof_of_KL_expansion}. The `$a.s.$' qualifications in this proposition mean that the infinite sums appearing in the matrix-matrix and matrix-vector products $\boldsymbol{\Phi}\mathbf{W}$ and $\mathbf{W}\phi(z_{test})$ converge almost surely. Proposition \ref{prop:KL_expansion} tells us that  $\mathbf{x}_{i}$ can be regarded as a random projection of $p^{1/2}\phi(z_i)$, plus additive noise, where the term ``random projection'' refers to the fact that the matrix $\mathbf{W}$ satisfies the expectation equality in the proposition and is independent of $z_1,\ldots,z_n$ and $\mathbf{e}_1,\ldots,\mathbf{e}_n$. We shall exploit this representation of data from the LMM when we decompose the error associated with our regression problem -- see Lemma \ref{lem:decomp} below.

\subsection{Notation}

We shall write $u(z)\coloneqq[u_{1}(z)\,u_{2}(z)\,\cdots\,]^{\top}$,
and in matrix form, $\boldsymbol{\Phi}\equiv[\phi(z_{1})|\cdots|\phi(z_{n})]^{\top}$,
$\mathbf{U}\equiv[u(z_{1})|\cdots|u(z_{n})]^{\top}$, $\boldsymbol{\Lambda}\equiv\text{diag}(\lambda_{1},\lambda_{2},\ldots)$. For any $k\geq1$ we write $\boldsymbol{\Phi}_{\leq k}$ and $\text{\ensuremath{\boldsymbol{\Phi}_{>k}}}$
for the submatrices consisting of respectively the first $k$ and
the remaining columns of $\boldsymbol{\Phi}$, so that $\mathbf{K}\coloneqq\boldsymbol{\Phi}\boldsymbol{\Phi}^{\top}=\boldsymbol{\Phi}_{\leq k}\boldsymbol{\Phi}_{\leq k}^{\top}+\boldsymbol{\Phi}_{>k}\boldsymbol{\Phi}_{>k}^{\top}=:\mathbf{K}_{\leq k}+\mathbf{K}_{>k}$,
and $\boldsymbol{\Lambda}_{>k}\coloneqq\text{diag}(\lambda_{k+1},\lambda_{k+2},\ldots)$. If $\rank<\infty$ and $k\geq\rank$, then under these definitions  $\boldsymbol{\Phi}_{>k}$, and $\mathbf{K}_{>k}$ and $\boldsymbol{\Lambda}_{>k}$ are matrices of zeros.

The $k$th largest eigenvalue of a symmetric matrix $\mathbf{B}$
is denoted $\mu_{k}(\mathbf{B})$. The identity matrix with $s\in\mathbb{N}\cup\{\infty\}$
rows and columns is denoted $\mathbf{I}_{s}$. The Frobenius and spectral
norms of a matrix $\mathbf{B}$ are denoted $\|\mathbf{B}\|_{F}$
and $\|\mathbf{B}\|_{2}$ respectively. For a function $f:\mathcal{Z}\to\mathbb{R}$,
$\|f\|_{L_{2}(\mu)}\coloneqq\left(\int_{\mathcal{Z}}|f(z)|^{2}\mu(\mathrm{d}z)\right)^{1/2}$.
For a generic vector $v$ and $k>1$, $v_{\leq k}$ and $v_{>k}$
denote respectively the first $1,\ldots,k$ and $k+1,k+2,\ldots$
coordinates of $v$. When $\mathbf{B}$ is a positive semidefinite matrix
we define $\|v\|_{\mathbf{B}}\coloneqq(v^{\top}\mathbf{B}v)^{1/2}$.
The complement of an event $C$ is denoted $\overline{C}$.

\section{Main results about prediction error}\label{sec:main_theorems}

The following lemma presents our main decomposition of the prediction
error and excess risk, in terms of:
\[
\hat{\theta}(\tilde{\mathbf{y}})\coloneqq\boldsymbol{\Phi}^{\top}\mathbf{A}^{-1}\tilde{\mathbf{y}},\qquad\tilde{\mathbf{y}}\in\mathbb{R}^{n}.
\]

\begin{lem}
\label{lem:decomp}
\begin{align*}
p^{-1/2}\mathbf{x}_{test}^{\top}\hat{\beta}(\mathbf{y})-g(z_{test}) & =\phi(z_{test})^{\top}\hat{\theta}(\mathbf{y})-\phi(z_{test})^{\top}\mathbf{\theta}^{*}\\
 & \quad+\left(p^{-1}\mathbf{x}_{test}^{\top}\mathbf{X}^{\top}-\phi(z_{test})^{\top}\boldsymbol{\Phi}^{\top}\right)\mathbf{A}^{-1}\mathbf{y}
\end{align*}
where
\begin{align}
\mathbf{A} & \coloneqq p^{-1}\mathbf{X}\mathbf{X}^{\top}+n\gamma\mathbf{I}_{n}=\mathbf{K}+\left(\sigma_{x}^{2}+n\gamma\right)\mathbf{I}_{n}+\boldsymbol{\Delta},\label{eq:A_defn}\\
\boldsymbol{\Delta} & \coloneqq\boldsymbol{\Phi}\left(\mathbf{W}^{\top}\mathbf{W}-\mathbf{I}_{r}\right)\boldsymbol{\Phi}^{\top}+p^{-1/2}\sigma_{x}\left(\boldsymbol{\Phi}\mathbf{W}^{\top}\mathbf{E}^{\top}+\mathbf{E}\mathbf{W}\boldsymbol{\Phi^{\top}}\right)+p^{-1}\sigma_{x}^{2}\left(\mathbf{E}\mathbf{E}^{\top}-p\mathbf{I}_{n}\right),\label{eq:Delta_defn}
\end{align}
and 
\[
\frac{1}{4}\mathbb{E}_{z_{test},\mathbf{e}_{test},\boldsymbol{\epsilon}}\left[\left|p^{-1/2}\mathbf{x}_{test}^{\top}\hat{\beta}(\mathbf{y})-g(z{}_{test})\right|^{2}\right]\leq B+V+\sum_{i=1}^{3}S_{i}
\]
where 
\begin{align*}
 & B\coloneqq\left\Vert \hat{\theta}(\boldsymbol{\Phi}\theta^{*})-\theta^{*}\right\Vert _{\boldsymbol{\Lambda}}^{2},\\
 & V\coloneqq\mathbb{E}_{\boldsymbol{\epsilon}}\left[\left\Vert \hat{\theta}(\boldsymbol{\epsilon})\right\Vert _{\boldsymbol{\Lambda}}^{2}\right],\\
 & S_{1}\coloneqq\mathbb{E}_{z_{test},\boldsymbol{\epsilon}}\left[\left|\phi(z_{test})^{\top}\left(\mathbf{W}^{\top}\mathbf{W}-\mathbf{I}_{r}\right)\hat{\theta}(\mathbf{y})\right|^{2}\right],\\
 & S_{2}\coloneqq\frac{\sigma_{x}^{2}}{p}\mathbb{E}_{z_{test},\boldsymbol{\epsilon}}\left[\left|\phi(z_{test})^{\top}\mathbf{W}^{\top}\mathbf{E}^{\top}\mathbf{A}^{-1}\mathbf{y}\right|^{2}\right],\\
 & S_{3}\coloneqq\frac{\sigma_{x}^{2}}{p^{2}}\mathbb{E}_{\mathbf{e}_{test},\boldsymbol{\epsilon}}\left[\left|\mathbf{e}_{test}^{\top}\mathbf{X}^{\top}\mathbf{A}^{-1}\mathbf{y}\right|^{2}\right].
\end{align*}
\end{lem}
We present the proof of this result in Section \ref{sec:proof_of_decomp}.  To put Lemma \ref{lem:decomp} in context, we provide the following interpretation:
\begin{itemize}
\item If $\boldsymbol{\Delta}$ was equal to the matrix of zeros,
then we would have $\mathbf{A}=\mathbf{K}+\sigma_x^2+n\gamma$, and the function
$z\mapsto\phi(z)^{\top}\hat{\theta}(\mathbf{y})$ would be the solution
of kernel ridge regression problem: 
\[
\argmin_{f\in\mathcal{H}}\frac{1}{n}\sum_{i=1}^{n}\left(y_{i}-f(z_{i})\right)^{2}+\left(\frac{\sigma_x^2}{n}+\gamma\right)\|f\|_{\mathcal{H}}^{2}
\]
where $\mathcal{H}$ is the RKHS associated with $\phi,$ and the
term $\phi(z_{test})^{\top}\hat{\theta}(\mathbf{y})-\phi(z_{test})^{\top}\mathbf{\theta}^{*}$
would be the associated prediction error. This hints that when $\boldsymbol{\Delta}$ is small, predictive performance may be similar to that of kernel ridge regression with regularisation: $\frac{\sigma^2_x}{n}+\gamma$. Indeed  controlling the probability
of the event
\[
\left\{ 2\|\boldsymbol{\Delta}\|_{2}\geq\mu_{n}(\mathbf{K})+\sigma_{x}^{2}+n\gamma\right\} 
\]
will be one of the main ingredients in arguing that that the $B$ and $V$ terms in the lemma above bound the excess risk associated
with this kernel ridge regression problem. 
\item Recalling from Proposition \ref{prop:KL_expansion} that $\mathbb{E}\left[\mathbf{W}^{\top}\mathbf{W}\right]=\mathbf{I}_{r}$,
controlling the term $S_{1}$ will involve showing, in a particular sense when $p$ is large,
that $\mathbf{W}^{\top}\mathbf{W}-\mathbb{E}\left[\mathbf{W}^{\top}\mathbf{W}\right]\approx0$.
The intuition here is that if we write $\mathbf{W}\equiv[W_{1}|\cdots|W_{p}]^{\top}$,
then $\mathbf{W}^{\top}\mathbf{W}=\sum_{j=1}^{p}W_{j}W_{j}^{\top}$,
and \ref{ass:ind} implies the vectors $(W_{j})_{j\geq1}$
are independent. 
\end{itemize}
The term within the conditional expectation $S_2$ can be written out in terms of inner-products between the zero-mean, independent vectors $\mathbf{e}_1,\ldots,\mathbf{e}_n\in\mathbb{R}^p$ and certain other $p$-dimensional vectors. Moment inequalities will be used to show that when re-scaled by $p$ these inner products are small. The term $S_3$ will be controlled in a similar manner.

\subsection{Bounding \texorpdfstring{$B$}{V} and \texorpdfstring{$V$}{B}}

The following definitions are taken from \citep{barzilai2023generalization}.
As remarked there but written in our notation, $\mathbb{E}[\|u(z_{1})_{\leq k}\|_{2}^{2}]=k$,
$\mathbb{E}[\|\phi(z_{1})_{>k}\|_{2}^{2}]=\mathrm{tr}(\boldsymbol{\Lambda}_{>k})$
and $\mathbb{E}[\|\boldsymbol{\Lambda}_{>k}^{1/2}\phi(z_{1})_{>k}\|_{2}^{2}]=\mathrm{tr}(\boldsymbol{\Lambda}_{>k}^{2})$,
so that the following coefficients, $\alpha_{k}$ and \textbf{$\beta_{k}$,
}quantify deviation from these expected values.\textbf{ }For any $k<\rank$,
\begin{align}
\alpha_{k}&\coloneqq\inf_{z\in\mathcal{Z}}\frac{\|\phi(z)_{>k}\|_{2}^{2}}{\mathrm{tr}(\boldsymbol{\Lambda}_{>k})}\nonumber \\
\beta_{k}&\coloneqq\sup_{z\in\mathcal{Z}}\max\left\{ \frac{\|u(z)_{\leq k}\|_{2}^{2}}{k},\frac{\|\phi(z)_{>k}\|_{2}^{2}}{\mathrm{tr}(\boldsymbol{\Lambda}_{>k})},\frac{\|\boldsymbol{\Lambda}_{>k}^{1/2}\phi(z)_{>k}\|_{2}^{2}}{\mathrm{tr}(\boldsymbol{\Lambda}_{>k}^{2})}\right\} \label{eq:beta_k_defn}
\\
r_{k}&\equiv r_{k}(\boldsymbol{\Lambda})\coloneqq\frac{\mathrm{tr}(\boldsymbol{\Lambda}_{>k})}{\|\boldsymbol{\Lambda}_{>k}\|_{2}}
\\
R_{k}&\equiv R_{k}(\boldsymbol{\Lambda})\coloneqq\frac{\mathrm{tr}(\boldsymbol{\Lambda}_{>k})^{2}}{\mathrm{tr}(\boldsymbol{\Lambda}_{>k}^{2})}.
\end{align}
In \citep{barzilai2023generalization}[App. H and G] the calculation of bounds on $\alpha_k$ and $\beta_k$ is discussed for well-known families of kernels, such as dot-product, radial basis function, shift-invariant and kernels on the hypercube.

The following definition is very similar to one introduced by \citet{barzilai2023generalization} but incorporates the covariate noise level $\sigma_x$ from the LMM.
For $k< \rank$, we define the \emph{concentration coefficient}
\begin{equation}\label{eq:rho_kn_defn}
\rho_{k,n}\coloneqq\frac{\|\boldsymbol{\Lambda}_{>k}\|_{2}+\mu_{1}(\frac{1}{n}\mathbf{K}_{>k})+\sigma_{x}^{2}/n+\gamma}{\mu_{n}(\frac{1}{n}\mathbf{K}_{>k})+\sigma_{x}^{2}/n+\gamma}
\end{equation}
Let us introduce the matrix:
\begin{equation}
\mathbf{A}_{k}\coloneqq\mathbf{K}_{>k}+\left(\sigma_{x}^{2}+n\gamma\right)\mathbf{I}_{n}+\boldsymbol{\Delta},\label{eq:A_k_defn}
\end{equation}
and the events, for $0\leq k \leq n$,
\begin{align}
C_{p,n}^{(k)} & \coloneqq\left\{ 2\|\boldsymbol{\Delta}\|_{2}\geq\mu_{n}(\mathbf{K}_{>k})+\sigma_{x}^{2}+n\gamma\right\} ,\label{eq:C_k_p_n_defn}\\
D_{n}^{(k)} & \coloneqq\left\{ \mu_{n}(\mathbf{K}_{>k})+\sigma_{x}^{2}+n\gamma=0\right\}, \label{eq:D_k_n_defn}
\end{align}
with the convention that $\mathbf{K}_{>0}\equiv\mathbf{K}$.
\begin{thm}
\label{thm:VB_overall_bound}There exist absolute constants $c,c^{\prime},c_{1},c_{2}$
such that for any $k<\rank$ with $c\beta_{k}k\log k\leq n$
and any $\delta>0$, we have that with probability at least $1-\delta-16\exp\left(-\frac{c^{\prime}}{\beta_{k}^{2}}\frac{n}{k}\right)-2\mathbb{P}(C_{p,n}^{(k)})-2\mathbb{P}(D_{n}^{(k)})$
the following two bounds hold simultaneously:
\begin{align*}
 & V\leq c_{1}\rho_{k,n}^{2}\sigma_{y}^{2}\left[\frac{k}{n}+\min\left\{\frac{r_{k}(\boldsymbol{\Lambda}^{2})}{n},\left(\frac{n}{R_{k}(\boldsymbol{\Lambda})}\frac{\mathrm{tr}(\boldsymbol{\Lambda}_{>k})^2}{\left(\alpha_k\mathrm{tr}(\boldsymbol{\Lambda}_{>k})+\sigma_x^2+n\gamma\right)^2}\right)\right\}\right],\\[1em]
 & B\leq c_{2}\rho_{k,n}^{3}\left[\frac{1}{\delta}\|\theta_{>k}^{*}\|_{\boldsymbol{\Lambda}_{>k}}^{2}+\frac{\|\theta_{\leq k}^{*}\|_{\boldsymbol{\Lambda}_{\leq k}^{-1}}^{2}}{n^2}\left(\beta_{k}\mathrm{tr}(\boldsymbol{\Lambda}_{>k})+\sigma_{x}^{2}+n\gamma\right)^{2}\right].
\end{align*}
\end{thm}
The proof of Theorem \ref{thm:VB_overall_bound} (presented in Section \ref{sec:proof_of_VB_overall_bound}) involves the application of several results of \citep[Theorem 2]{barzilai2023generalization},
subject to some subtle modifications. Recalling the definitions of
$\mathbf{A}$ and $\mathbf{A}_{k}$ from (\ref{eq:A_defn}) and (\ref{eq:A_k_defn}),
we have $\mathbf{A}=\mathbf{K}_{\leq k}+\mathbf{A}_{k}$. To connect
to the setup of \citet{barzilai2023generalization}, we note that
if it were the case that $\sigma_{x}^{2}=0$ and $\boldsymbol{\Delta}=\mathbf{0}$,
then $\mathbf{A}_{k}$ would be of exactly the same form as the matrix
$A_{k}$ defined in \citep[App. D.2]{barzilai2023generalization}.
The main observation which allows most of the reasoning of \citep[proof of Theorem 2]{barzilai2023generalization}
to be transferred to the present context is that much of their analysis
applies with our definition of $\mathbf{A}_{k}$ (\ref{eq:A_k_defn})
in force, even when $\sigma_{x}^{2}>0$ and/or $\boldsymbol{\Delta}\neq\mathbf{0}$,
as long as it can be shown that $\mathbf{A}_{k}$ is positive-definite,
i.e., $\mu_{n}(\mathbf{A}_{k})>0$, and the condition $2\|\boldsymbol{\Delta}\|_{2}\leq\mu_{n}(\mathbf{K}_{>k})+\sigma_{x}^{2}+n\gamma$
holds. This is achieved by restricting our attention to the intersection
of the complements of the events $C_{p,n}^{(k)}$ and $D_{n}^{(k)}$,
contributing to the probability which is quantified in the statement
of Theorem \ref{thm:VB_overall_bound}.

The following theorem is a variant of Theorem \ref{thm:VB_overall_bound} addressing the special case $\rank<\infty$, whose proof we present in Section \ref{sec:proof_of_VB_finite_rank}.

\begin{thm}\label{thm:VB_finite_rank}
There exist absolute constants $c$, $c^{\prime}$, $c_{1}$, $c_{2}$
such that if $\rank<\infty$ and $\max(\sigma_{x},\gamma)>0$, then
for any $n\geq c\beta_{\rank}\rank\log \rank$, we
have that with probability at least $1-16\exp\left(-\frac{c^{\prime}n}{\beta_{\rank}^{2}\rank}\right)-2\mathbb{P}(C_{p,n}^{(\rank)})$,
\begin{align*}
V\leq c_{1}\sigma_{y}^{2}\left(\frac{k_{max}}{n}\right),\qquad
B\leq 
c_{2}\left(\frac{\|\theta_{\leq k_{max}}^{*}\|_{\boldsymbol{\Lambda}_{\leq k}^{-1}}^{2}}{n^2}\left(\sigma_{x}^{2}+n\gamma\right)^{2}\right).
\end{align*}    
\end{thm}

\subsection{Bounding \texorpdfstring{$S_{1}$, $S_{2}$, $S_{3}$}{S1, S2, S3}}

We now seek to bound the residual terms $S_i$ arising from the LMM.  To this end, define 
\begin{equation}
v_{1}\coloneqq\sup_{z,z^{\prime}\in\mathcal{Z}}\frac{1}{p}\sum_{j=1}^{p}\mathrm{Var}\left[\psi_{j}(z)\psi_{j}(z^{\prime})\right]\label{eq:v1_defn}
\end{equation}
\begin{equation}
v_{2}\coloneqq\sup_{z\in\mathcal{Z}}\frac{1}{p}\sum_{j=1}^{p}\mathbb{E}\left[\left|\psi_{j}(z)\right|^{2}\right]\label{eq:v2_defn}
\end{equation}

Then we have the following result, whose proof we present in Section \ref{sec:proof_of_R_overall_bound}
\begin{thm}\label{thm:R_overall_bound}
For any $\delta_{i}>0$, $i=1,2,3$, with probability at least $1-\sum_{i=1}^3\delta_i$,
\[
S_{1}+S_{2}+S_{3}\leq\frac{n^2}{p}\left(\frac{v_1}{\delta_1} +\frac{\sigma_x^2 v_2}{\delta_2} + \frac{\sigma_x^2(v_2+\sigma_x^2)}{\delta_3}\right)\frac{\left(\sup_{z}|g(z)|^{2}+\sigma_{y}^{2}/n\right)}{\left(\mu_{n}(p^{-1}\mathbf{X}\mathbf{X}^{\top})+n\gamma\right)^{2}},
\]
and with probability at least $1-\sum_{i=1}^3\delta_i-\mathbb{P}(C_{p,n}^{(0)})-\mathbb{P}(D_{n}^{(0)})$,
\[
S_{1}+S_{2}+S_{3}\leq\frac{4n^2}{p}\left(\frac{v_1}{\delta_1} +\frac{\sigma_x^2 v_2}{\delta_2} + \frac{\sigma_x^2(v_2+\sigma_x^2)}{\delta_3}\right)\frac{\left(\sup_{z}|g(z)|^{2}+\sigma_{y}^{2}/n\right)}{\left(\mu_{n}(\mathbf{K})+\sigma_{x}^{2}+n\gamma\right)^{2}}.
\]
\end{thm}

\subsection{Probability of the event \texorpdfstring{$C_{p,n}^{(k)}$}{Cpnk}}

Our final task is to determine a bound for the probability of the event $C_{p,n}^{(k)}$.  If we define
\begin{equation}v_3 \coloneqq  \mb{E}\bigl[|\m{E}_{ij}|^4],\label{eq:v3_defn}
\end{equation}
(recalling from the definition of the LMM in Section \ref{sec:LMM_and_reg_model} that $\m{E}_{ij}$ are i.i.d. across all $i,j$) then we obtain the following bound, whose proof we present in Section \ref{sec:proof_of_C_unconditional}.
\begin{prop}\label{prop:prob_of_C_unconditional}
For any $0\leq k\leq n$ and any number $\phi_{k}(n)\geq0$,
\[
1-\mathbb{P}\left(C_{p,n}^{(k)}\right)\geq\left(1-\frac{24n^2}{p}\frac{(v_1+8\sigma_x^2v_2+2\sigma_x^4 v_3)}{(\phi_{k}(n)+\sigma_{x}^{2}+n\gamma)^{2}}\right)\mathbb{P}\left(\mu_{n}(\mathbf{K}_{>k})\geq \phi_{k}(n)\right).
\]
\end{prop}

\section{Interpretation and application of Theorems \ref{thm:VB_overall_bound}-\ref{thm:R_overall_bound}}\label{sec:interpretation}

We now look to apply the results of Section \ref{sec:main_theorems} to obtain more precise bounds on the prediction error under specific behaviours of the eigenvalues of the implicit kernel associated with the LMM and derive associated convergence rates. Throughout Section \ref{sec:interpretation} we assume that $p = p(n)$ and $\gamma = \gamma(n)$ are respectively non-decreasing and non-increasing functions of $n$, where we shall take $n\to\infty$.

We will restrict our focus to three illustrative eigenvalue behaviours: one in which the kernel has finite rank; one in which it has infinite rank and its eigenvalues decay at an exponential rate; and one in which it has infinite rank and its eigenvalues decay at a polynomial rate.  For each decay rate, we consider both the case in which we have explicit regularisation (in the sense that the regularisation parameter $\gamma > 0$) and also the case in which regularisation is provided by the presence of covariate noise in the LMM (so that $\sigma_x > 0$).  These scenarios are summarised in Figure \ref{fig:tree_intro}.

 Before proceeding, we establish some additional notation:

\begin{itemize}
\item We use $O(\cdot)$, $\Theta(\cdot)$, $\omega(\cdot)$  in the usual way to indicate  asymptotic behaviour: for two nonnegative sequences $(\kappa_n)_{n\geq1}$, $(l_n)_{n\geq1}$, $\kappa_n=O(l_n)$ means $\limsup_{n}\kappa_n/l_n<\infty$; $\kappa_n=\Theta(l_n)$ means that both $\kappa_n=O(l_n)$ and $l_n=O(\kappa_n)$; and $\kappa_n=\omega(l_n)$ means $\lim_n  \kappa_n/l_n=\infty$.  Subscript lower case $p$ on $O_{p}(\cdot)$, $\Theta_{p}(\cdot)$ is used to denote uniformity with respect to the dimensionality $p$. For example,
recalling from Section \ref{sec:assump_and_prop} that the eigenvalues $(\lambda_{k})_{k\geq1}$ depend
on $p$ in general, the statement ``$\lambda_{k}=O_{p}(e^{-ak})$
as $k\to\infty$'', means that there exists some finite constants
$c$ and $k_{0}$, such that for all $p\geq1$ and $k\geq k_{0}$,
$\lambda_{k}\leq ce^{-ak}$.
\item Some results in Section \ref{sec:interpretation} involve statements of the form: ``with probability
at least $1-\delta_{n}-O(a_{n})$, $X_{n}=O(\kappa_{n})$'' where
$(X_{n})_{n\geq1}$ is some sequence of random variables, and $(\delta_{n})_{n\geq1}$,
$(a_{n})_{n\geq1}$ and $(\kappa_{n})_{n\geq1}$ are deterministic
sequences. This means that there exist some finite constants $c_1,c_2, n_0$ such that for any $n\geq n_0$, $\mathbb{P}(|X_n|\leq c_1\kappa_n )\geq 1-\delta_n-c_2 a_n$.

\item $\OP(\cdot)$ denotes ``big oh in probability'' under
its usual definition; for a sequence of random variables $(X_{n})_{n\geq1}$
and some strictly positive sequence $(\kappa_{n})_{n\geq1}$, $X_{n}=O_{\mathbb{P}}(\kappa_{n})$
means that for any $\delta>0$ there exists constants $c(\delta)$
and $n_{0}(\delta)$ such for that for any $n\geq n_{0}(\delta)$,
$\mathbb{P}(|X_{n}|>\kappa_{n}c(\delta))<\delta$. 

\item Note that if for some decreasing sequence $a_n\searrow 0$, it holds that with probability at least $1-O(a_n)$, $X_n=O(\kappa_n)$, then $X_n=\OP(\kappa_n)$.

\end{itemize}

\subsection{Finite rank}

For our first example, we consider the situation in which the implicit kernel has only finitely many non-zero eigenvalues, that is, where $\rank < \infty$.  In this case, one can simply apply a union bound to combine the results of Theorems \ref{thm:VB_finite_rank} and \ref{thm:R_overall_bound} (with appropriate choice of $\delta_1$, $\delta_2$, $\delta_3$ in the latter) to obtain the following result:

\begin{thm}\label{thm:finite_rank_combined}
Assume that $\rank=O(1)$ and $\sup_j\beta_j=O(1)$ as $p\to\infty$, and that $\max(\sigma_x,\gamma)>0$. Then for any $\delta>0$ with probability at least $1-\delta-\exp\left[-\Theta\left(\frac{n}{\rank}\right)\right]-O\left(\frac{n^2}{p}\frac{\left(v_{1}+\sigma_{x}^{2}v_{2}+\sigma_{x}^{4}v_{3}\right)}{(\sigma_{x}^{2}+n\gamma)^{2}}\right)$
\begin{align*}
V&= O\left(\frac{\sigma_y^2\,\rank}{n}\right)\\
B&=O\left(\frac{\|\theta^*_{\leq\rank}\|^2_{\boldsymbol{\Lambda}_{\leq \rank}^{-1}}}{n^2}(\sigma_x^2+n\gamma)^2\right) \\
\sum_{i=1}^3 S_i &=O\left(\frac{n^2}{\delta p}\frac{\left(v_{1}+\sigma_{x}^{2}v_{2}+\sigma_{x}^{4}\right)\left(\sup_{z}|g(z)|^{2}+\sigma_{y}^{2}/n\right)}{(\sigma_{x}^{2}+n\gamma)^{2}}\right) 
\end{align*}
as $n \to \infty$.
\end{thm}

The stochastic convergence rates in Table \ref{tab:finite_rank} follow immediately from Theorem \ref{thm:finite_rank_combined} and have the following interpretations.  Consider first the case in which there is no covariate noise, i.e., $\sigma_x=0$, but we have regularisation with rate $\gamma = n^{-\epsilon/2}$ for some $\epsilon > 0$.  In order for the convergence rates to hold in probability we then require that the dimension grows according to $p = \omega\left(n^\epsilon v_1\right)$, which also forces the residual terms $S_i$ to go to zero.  We observe that the variance term $V$ decays at a rate proportional to $1/n$, independent of our choice of $\epsilon$, while the bias term $B$ decays at a rate proportional to $1/n^{\epsilon}$, thus giving us a trade-off in which increasing the rate of decay of the bias requires a corresponding increase in the dimension to ensure that the total prediction error goes to zero in probability.

On the other hand, when $\gamma=0$ and $\sigma_x^2>0$ is a constant, we require dimension grows as $p = \omega\left(n^2[v_1+\sigma_x^2v_2+\sigma_x^4v3]/\sigma_x^4\right)$ in order for the convergence rates to hold in probability, which again forces the residual terms $S_i$ to go to zero. In this case the $1/n$ convergence rate for $V$ still holds, while $B$ decays at a rate proportional to $\sigma_x^4/n^2$.  Intuitively, in this situation the additive noise is inducing some bias (as indicated by the $B$ term), whilst also making a useful contribution to implicit regularisation (as in the aforementioned growth condition on $p$ and the $S$ term).

Recall from \eqref{eq:v1_defn} and \eqref{eq:v2_defn} that $v_1$ and $v_2$ are related to the moments of the random functions $\psi_j$ in the LMM. In particular, if these random functions are actually deterministic (which does not violate our independence assumption \ref{ass:ind}, since any two a.s.-constant random variables are statistically independent) then $v_1=0$.   We see from Table \ref{tab:finite_rank} that  $v_1$ and $v_2$ being small is beneficial for convergence.

\begin{table}[h!]
\begin{centering}
\begin{adjustbox}{width=0.875\textwidth}
\small
\begin{tabularx}{\textwidth}{c*{2}{Y} }
\toprule
reg.&$\gamma = n^{-\epsilon/2},\enspace\epsilon > 0, \quad \sigma_x = 0$&$\gamma = 0, \quad \sigma_x > 0$\\
\midrule
dim.& $\omega\left(n^{1+\epsilon}[v_1+\sigma_x^2v_2+\sigma_x^4v_3]\right)$ & $\omega\left(\dfrac{n^2[v_1+\sigma_x^2v_2+\sigma_x^4v_3]}{\sigma_x^4}\right)$  \\
\midrule
$V$&$\OP\left(\dfrac{\sigma_y^2\,\rank}{n}\right)$ &$\OP\left(\dfrac{\sigma_y^2\,\rank}{n}\right)$ \\
\midrule 
$B$&$\OP\left(\Adjfrac{\|\theta^*_{\leq\rank}\|^2_{\boldsymbol{\Lambda}_{\leq \rank}^{-1}}}{n^\epsilon}\right)$&$\OP\left(\Adjfrac{\sigma_x^4\,\|\theta^*_{\leq\rank}\|^2_{\boldsymbol{\Lambda}_{\leq \rank}^{-1}}}{n^2}\right)$\\
\midrule 
$\sum S_i$&$\OP\left(\dfrac{v_1 n^\epsilon}{p}\left[\sup_z|g(z)|^2+\dfrac{\sigma_y^2}{n}\right]\right)$&$\OP\left(\dfrac{n^2 (v_1+\sigma_x^2v_2+\sigma_x^4)}{p\,\sigma_x^4}\left[\sup_z|g(z)|^2+\dfrac{\sigma_y^2}{n}\right]\right)$\\
\bottomrule
\end{tabularx}
\end{adjustbox}
\par\end{centering}
\caption{\label{tab:finite_rank} Convergence rates for the finite rank case, $\rank < \infty$, under given conditions on regularisation (reg.) and dimensionality (dim.). }
\end{table}

\subsection{Exponential eigenvalue decay}
For our second example, we consider the case in which $\rank=\infty$ but the eigenvalues of the implicit kernel decay to zero at an exponential rate.  In this situation, we obtain Theorem \ref{thm:fast_eig_decay}, the proof of which we present in Section \ref{sec:decay_proofs}:

\begin{thm}
\label{thm:fast_eig_decay}Assume that for some $a>0$, $\lambda_{k}=\Theta_{p}(e^{-ak})$ as $k\to\infty$.  Additionally, assume that $\sup_{j}|\theta_{j}^{*}|^{2}=O(1)$, 
$\sup_{j\geq 1}\beta_{j}=O(1)$ and $\lambda_{1}=\Theta(1)$ 
as $p\to\infty$, and that $\max(\sigma_{x},\gamma)>0$.  Then for any
$\delta,\delta_{\rho}\in(0,1)$, with probability at least $1-\delta-\delta_{\rho}-\exp\left[-\Theta\left(\frac{n}{\log n}\right)\right]-O\left(\frac{n^2}{p}\frac{\left(v_{1}+\sigma_{x}^{2}v_{2}+\sigma_{x}^{4}v_{3}\right)}{(\sigma_{x}^{2}+n\gamma)^{2}}\right)$,

\begin{align*}
V & =O\left(\frac{\sigma_y^2\,\log n}{n}\left(1+\frac{1}{\delta_{\rho}(\sigma_{x}^{2}+n\gamma)}\right)^{2}\right)\\
B & =O\left(\frac{\sup_{j}|\theta_{j}^{*}|^{2}}{n}\left(1+\frac{1}{\delta_{\rho}(\sigma_{x}^{2}+n\gamma)}\right)^{3}\left[\frac{1}{\delta}+\left(\frac{1}{n}+\sigma_{x}^{2}+n\gamma\right)^{2}\right]\right)\\
\sum_{i=1}^{3}S_{i} & =O\left(\frac{n^2}{\delta p}\frac{\left(v_{1}+\sigma_{x}^{2}v_{2}+\sigma_{x}^{4}\right)\left(\sup_{z}|g(z)|^{2}+\sigma_{y}^{2}/n\right)}{(\sigma_{x}^{2}+n\gamma)^{2}}\right), 
\end{align*}
as $n\to\infty$.
\end{thm}

The stochastic convergence rates in Table \ref{tab:fast_eig_decay} follow immediately from Theorem \ref{thm:fast_eig_decay} and have the following interpretations.  Consider first the case in which there is no covariate noise but we have regularisation with rate $\gamma = n^{-(1+\epsilon)/2}$ for some $\epsilon \in (0,1]$ (when working with an infinite-dimensional implicit kernel, we are far more restricted in our choices for $\gamma$, as if it decays too quickly then the concentration coefficient $\rho_{k,n}$ defined in \eqref{eq:rho_kn_defn} blows up).  In order for the convergence rates to hold in probability we then require that the dimension $p = \omega\left(n^{1+\epsilon}v_1\right)$, which also forces the residual terms $S_i$ to go to zero.  We observe that the variance term $V$ is again independent of our choice of $\epsilon$, but now decays at a slightly slower rate proportional to $\log n/n$, while the bias term $B$ again decays at a rate proportional to $1/n^{\epsilon}$, thus retaining the trade-off between the rate of decay of the bias and the dimension required for the total prediction error to go to zero in probability.

On the other hand, when $\gamma=0$ and $\sigma_x^2>0$ is a constant, we require that the dimension $p = \omega\left(n^2[v_1+\sigma_x^2v_2+\sigma_x^4v3]/\sigma_x^4\right)$ in order for the convergence rates to hold in probability, which again forces the residual terms $S_i$ to go to zero. In this case the convergence rate for $V$ still holds, while $B$ decays at a rate proportional to $\sigma_x^4/n$ (which we note is slower than the finite rank case by a factor of $1/n$).

\begin{table}[h!]
\begin{centering}
\begin{adjustbox}{width=0.875\textwidth}
\small
\begin{tabularx}{\textwidth}{c*{2}{Y} }
\toprule
reg. & $\gamma=n^{-(1+\epsilon)/2}$,\enspace$\epsilon\in(0,1]$,\quad$\sigma_x=0$ & $\gamma=0$,\quad$\sigma_x > 0$\\ 
\midrule
dim.& $\omega\left(n^{1+\epsilon}[v_1+\sigma_x^2v_2+\sigma_x^4v_3]\right)$ & $\omega\left(\dfrac{n^2[v_1+\sigma_x^2v_2+\sigma_x^4v_3]}{\sigma_x^4}\right)$  \\
\midrule
$V$&$\OP\left(\dfrac{\sigma_y^2\,\log n }{n}\right)$ & $\OP\left(\dfrac{\sigma_y^2\,\log n}{n}\right)$\\
\midrule 
$B$&$\OP\left(\adjfrac{\sup_j|\theta_j^*|^2}{n^\epsilon}\right)$ & $\OP\left(\adjfrac{\sigma_x^4\,\sup_j|\theta_j^*|^2}{n}\right)$\\
\midrule 
$\sum S_i$&$\OP\left(\dfrac{v_1 n^{1+\epsilon}}{p}\left[\sup_z|g(z)|^2+\dfrac{\sigma_y^2}{n}\right]\right)$&$\OP\left(\dfrac{n^2 (v_1+\sigma_x^2v_2+\sigma_x^4)}{p\,\sigma_x^4}\left[\sup_z|g(z)|^2+\dfrac{\sigma_y^2}{n}\right]\right)$\\
\bottomrule
\end{tabularx}
\end{adjustbox}
\par\end{centering}
\caption{\label{tab:fast_eig_decay} Convergence rates for the case of exponential eigenvalue decay,  $\lambda_k=\Theta(e^{-ak})$, under given conditions on regularisation (reg.) and dimensionality (dim.).}
\end{table}

\subsection{Polynomial eigenvalue decay}

For our third and final example, we consider the case in which $\rank=\infty$ but the eigenvalues of the implicit kernel decay to zero at a polynomial rate.  In this situation, we obtain Theorem \ref{thm:poly_eig_decay}, the proof of which we again present in Section \ref{sec:decay_proofs}:

\begin{thm}\label{thm:poly_eig_decay}Assume that for some $a > 0$, $\lambda_{k}=\Theta_{p}\left(k^{-(a+2)}\right)$ as $k \to \infty$.  Additionally, assume that $\alpha_k, \beta_k =\Theta(1)$ and $|\theta_j^*|^2 = o(j^{-1})$ as $p \to \infty$, and that $\max(\sigma_{x},\gamma) > 0$.  Then for any $\delta \in(0,1)$, with probability at least $1-\delta-O\left(\frac{1}{n}\right)-O\left(\frac{n^2}{p}\frac{\left(v_{1}+\sigma_{x}^{2}v_{2}+\sigma_{x}^{4}v_{3}\right)}{(\sigma_{x}^{2}
+n\gamma)^{2}}\right)$,

\begin{align*}
V & =O\left(\frac{\sigma_y^2}{n^\frac{a+1}{a+2}}\left(1+\frac{1}{\sigma_x^2+ n\gamma}\right)^2\right)\\[0.6em]
B &= O\left(\frac{\sup_{j}|\theta_{j}^{*}|^{2}}{ n}\left(1+\frac{1}{\sigma_x^2 + n\gamma}\right)^3\left[\frac{1}{\delta}+\left(\frac{1}{n^\frac{a+1}{a+2}}+\sigma_x^2+n\gamma\right)^2\right]\right)\\[0.6em]
\sum_{i=1}^{3}S_{i} & =O\left(\frac{n^2}{\delta p}\frac{\left(v_{1}+\sigma_{x}^{2}v_{2}+\sigma_{x}^{4}\right)\left(\sup_{z}|g(z)|^{2}+\sigma_{y}^{2}/n\right)}{(\sigma_{x}^{2}+n\gamma)^{2}}\right)
\end{align*}
as $n\to\infty$.
\end{thm}

Note that the conditions on the decay of the terms $\lambda_k$ and $|\theta_k^*|^2$  in Theorem \ref{thm:poly_eig_decay} are required in order to ensure that assumptions \ref{ass:g} and \ref{ass:eig_decay_basic} are met.

The stochastic convergence rates in Table \ref{tab:poly_decay} follow immediately from Theorem \ref{thm:poly_eig_decay} and have the following interpretations.  We note that almost all terms behave in the same way as the previous example, with the exception of the variance term $V$, which now decays at a rate proportional to $1/n^\frac{a+1}{a+2}$, which gets progressively closer to the equivalent rate of $\log n / n$ observed in the exponential case as $a$ grows.

\begin{table}[h!]
\begin{centering}
\begin{adjustbox}{width=0.875\textwidth}
\small
\begin{tabularx}{\textwidth}{c*{2}{Y} }
\toprule
reg.&$\gamma = n^{-(1+\epsilon)/2},\enspace\epsilon \in (0,1], \quad \sigma_x = 0$&$\gamma = 0, \quad \sigma_x > 0$\\
\midrule
dim.& $\omega\left(n^{1+\epsilon}[v_1+\sigma_x^2v_2+\sigma_x^4v_3]\right)$ & $\omega\left(\dfrac{n^2[v_1+\sigma_x^2v_2+\sigma_x^4v_3]}{\sigma_x^4}\right)$  \\
\midrule
$V$&$\OP\left(\adjfrac{\sigma_y^2}{n^\frac{a+1}{a+2}}\right)$ &$\OP\left(\adjfrac{\sigma_y^2}{n^\frac{a+1}{a+2}}\right)$ \\
\midrule
$B$&$\OP\left(\adjfrac{\sup_j|\theta_j^*|^2}{n^\epsilon}\right)$&$\OP\left(\adjfrac{\sigma_x^4\,\sup_j|\theta_j^*|^2}{n}\right)$\\
\midrule
$\sum S_i$&$\OP\left(\dfrac{v_1 n^{1+\epsilon}}{p}\left[\sup_z|g(z)|^2+\dfrac{\sigma_y^2}{n}\right]\right)$&$\OP\left(\dfrac{n^2 (v_1+\sigma_x^2v_2+\sigma_x^4)}{p\,\sigma_x^4}\left[\sup_z|g(z)|^2+\dfrac{\sigma_y^2}{n}\right]\right)$\\
\bottomrule
\end{tabularx}
\end{adjustbox}
\par\end{centering}
\caption{\label{tab:poly_decay} Convergence rates for the polynomial decay case, $\lambda_{k}=\Theta_{p}\left(k^{-(a+2)}\right)$, $a > 0$, under given conditions on regularisation (reg.) and dimensionality (dim.).}
\end{table}

\section{Numerical Results}\label{sec:numerical_results}

\subsection{Extending the cosines example of \texorpdfstring{\citet{tsigler2023benign}}{Tsigler and Barlett}}

As an illustrative example of the benign overfitting phenomenon, \citet[Figure 1]{tsigler2023benign} consider the problem of learning the function $z\mapsto\cos(3z)$ by linear regression, from data points $(z_i,y_i)_{i=1}^{60}$ where the $z_i$ are i.i.d. draws from the uniform distribution on $[0,\pi]$ and $y_i$ has a normal distribution with mean $\cos(3z_i)$ and standard deviation $0.4$. They consider different combinations of cosine features of the form $\cos(mz)$, $m=1,2,3,\ldots$, and numerically illustrate behaviour of the least-norm solution to the OLS problem. We now present an extension of their example in the setting of the LMM and our regression problem from Section \ref{sec:LMM_and_reg_model}, such that we recover linear regression with cosine features as per \citet{tsigler2023benign} in the $p\to\infty$ limit.

Consider an instance of the LMM with $\mathcal{Z}=[0,\pi]$ and $\mu$ being the uniform distribution. We construct a kernel as follows: we define $u_k(z)\coloneqq \sqrt{2}\cos(kz)$ for $k=1,2,\ldots$. These functions $u_k$ are orthonormal in $L_2(\mu)$, indeed using the identity $\cos(a)\cos(b)=[\cos(a+b)+\cos(a-b)]/2$, we have for $m\neq n$, 
\begin{align*}
\int_{0}^{\pi}\cos(mz)\cos(nz)\mathrm{d}z & =\frac{1}{2}\int_{0}^{\pi}\cos((m+n)z)\mathrm{d}z+\frac{1}{2}\int_{0}^{\pi}\cos((m-n)z)\mathrm{d}z\\
 & =\frac{1}{2}\left[\frac{\sin((m+n)z)}{m+n}\right]_{0}^{\pi}+\frac{1}{2}\left[\frac{\sin((m-n)z)}{m-n}\right]_{0}^{\pi}=0,
\end{align*}
and for $m=n>0$, 
\begin{align*}
\int_{0}^{\pi}\cos^{2}(mz)\mathrm{d}z & =\int_{0}^{\pi}\frac{1+\cos(2mz)}{2}\mathrm{d}z\\
 & =\frac{\pi}{2}+\frac{1}{2}\left[\frac{\sin(2mz)}{2m}\right]_{0}^{\pi}=\frac{\pi}{2}.
\end{align*}
Then for some nonnegative $\lambda_1\geq \lambda_2 \geq \lambda_3 \geq \cdots$ define the kernel:
$$
f(z,z^\prime)\coloneqq \sum_{k\geq 1}\lambda_k u_k(z)u_k(z^\prime) = 2\sum_{k\geq 1} \lambda_k\cos(kz)\cos(kz^\prime).
$$
By construction, the r.h.s. of the above equation is the Mercer expansion of the kernel $f$ with feature map:
\[
\phi(z)=\left[\begin{array}{c}
\sqrt{2\lambda_1}\cos(z)\\
\sqrt{2\lambda_2}\cos(2z)\\
\sqrt{2\lambda_3}\cos(3z)\\
\vdots
\end{array}\right].
\]
We take the random functions $\psi_j$ in the LMM to be i.i.d., zero-mean Gaussian processes with common covariance kernel $f$. Then by construction, $f$ is the implicit kernel of this LMM.

We take $\sigma_y = 0.4$ and $g(z)=\cos(3z)$ following the setup of \citet{tsigler2023benign}. 
Noting that kernel ridge regression can be viewed as linear regression onto covariates given by the feature map, we see that in the $p\to\infty$ limit, we will recover linear regression with features given by $\sqrt{2\lambda_k}\cos(kz)$, $k=1,2,...$.

We consider three distinct regimes for the eigenvalues $\lambda_k$, namely:
\begin{itemize}
\item \emph{Finite rank}: We set $\lambda_k = 1$ for $k = 1, \ldots, 20$, and $\lambda_k = 0$ otherwise.
\item \emph{Exponential decay}: We set $\lambda_k = \mr{exp}(-k)$ for all $k$.
\item \emph{Polynomial decay}: We set $\lambda_k = k^{-4}$ for all $k$.
\end{itemize}
(Note that for practical purposes, for the latter two regimes we set $\lambda_k = 0$ for $k$ sufficiently large, which in our case we take to be when $k > 10^4$).

\begin{figure}[h!]
  \centering
  \includegraphics[trim = 25 20 25 0, clip, scale=0.18]{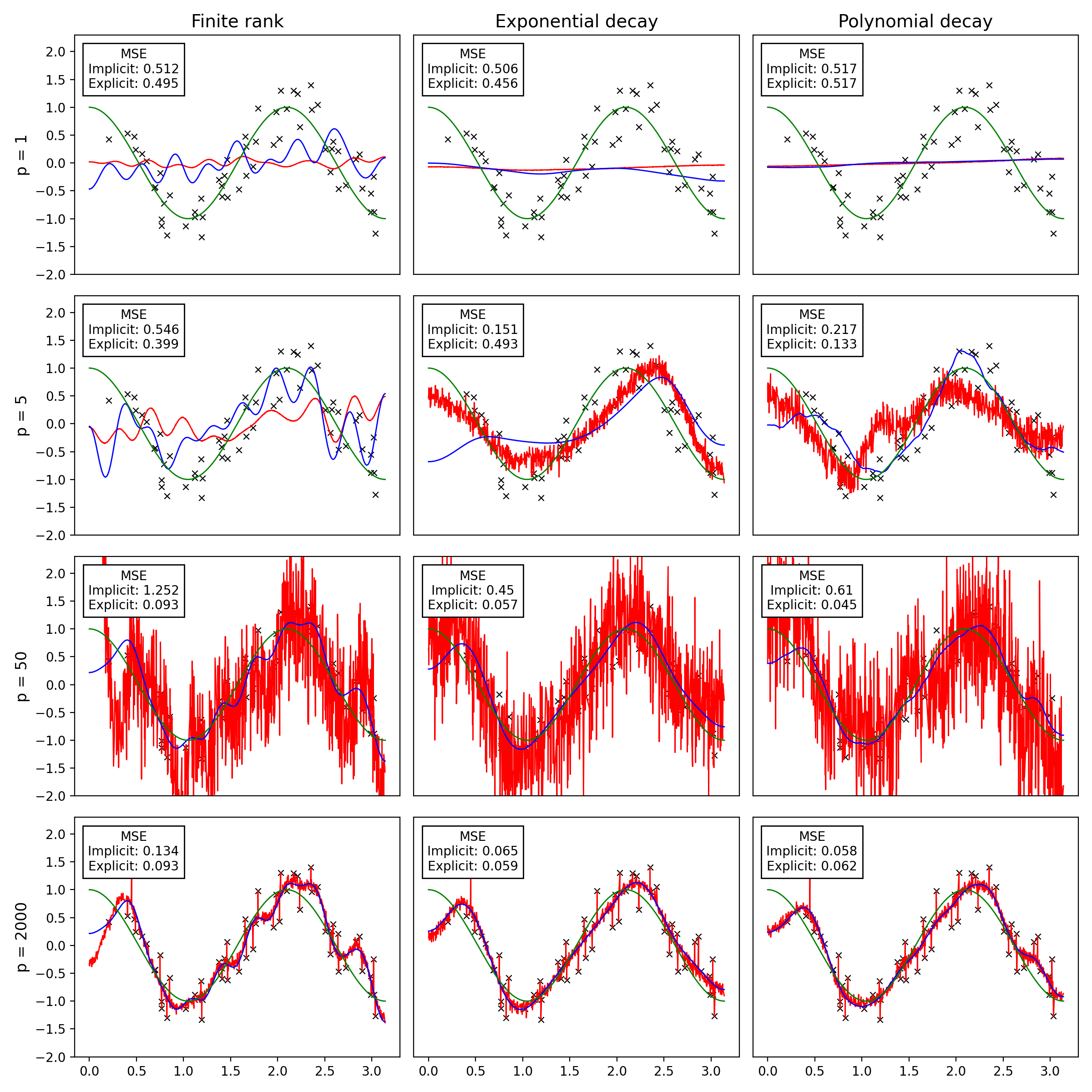}
  \caption{Demonstration of benign overfitting for different kernel eigenvalue regimes for the target function $g(z) = \cos(3z)$.  Black crosses denote the test points $\{g(z_i)\}_{i=1}^{60}$, the green curve denotes the target function, while the red and blue curves denote the predictions for $g(z_{test})$ for $1000$ test points under the explicit and implicit regularisation schemes respectively.}
  \label{fig_cosines_example}
\end{figure}

Figure \ref{fig_cosines_example} exhibits the behaviour of kernel ridge regression as $p$ increases in each of these regimes.  In each plot, the black crosses represent the training set of $n = 60$ points (which remains fixed across all examples) while the green curve represents the target function $g(z) = \cos(3z)$.  We apply two different types of regularisation: the blue curve represents \emph{explicit} regularisation, in which the ridge regularisation parameter $\gamma$ is non-zero but the additive covariate noise $\sigma_x$ in the latent metric model is zero, while the red curve represents \emph{implicit} regularisation, in which the roles of $\gamma$ and $\sigma_x$ are reversed (and we assume the covariate noise to be normally distributed).  For the former we set a value of $\gamma = 10^{-4}$, while for the latter we set $\sigma_x = (n\gamma)^{1/2}$, to keep their contributions roughly consistent as per Theorem \ref{thm:VB_overall_bound}.

To demonstrate the behaviour of the prediction error as the size $n$ of the training set grows, we consider two examples, in which we perform regression on the target function $g(z) = \cos(3z)$ with an implicit kernel of finite rank $\rank = 40$.  For the first example, we apply explicit regularisation by setting $\gamma = n^{(1+\epsilon)/2}$ for varying values of $\epsilon$, with $p = \lfloor n^{1.25}\rfloor$.  We assume that there is no covariate noise (so $\sigma_x = 0$) and that the observations $y_i$ have Gaussian noise with $\sigma_y = 0.4$. 

For each value of $\epsilon$ we ran $50$ independent trials each consisting of generating a training set of size $n$ and evaluating the mean squared error on a test set of $250$ points. These mean squared errors are shown in Figure \ref{fig:explicit_reg_example}, with the shaded areas on the plot denoting the $95\%$ error bounds across the $50$ trials.  We observe that increasing the regularisation rate $\gamma$ yields better performance, which is in line with our results in Theorem \ref{thm:finite_rank_combined}, from which we would expect the variance term $V$ to decay consistently across all examples, but the bias term $B$ to decay faster for higher values of $\epsilon$, while our choice of $p$ ensures that the residual terms should always decay at least as quickly as the variance.

\begin{figure}[ht!]
  \centering
  \includegraphics[trim = 0 0 0 0, clip, scale=0.15]{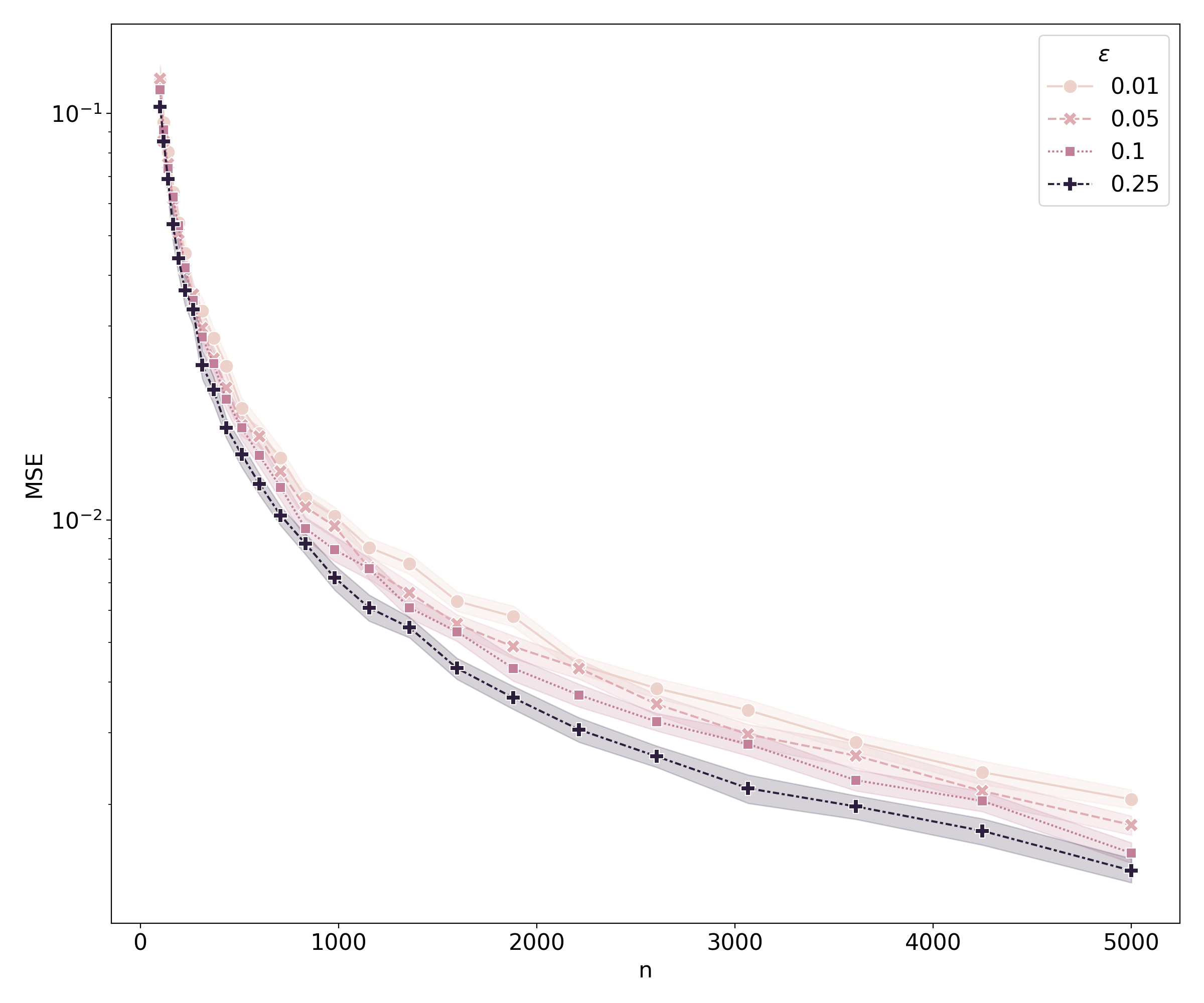}
  \caption{Asymptotic behaviour of the prediction error when performing regression on the target function $g(z) = \cos(3z)$ with explicit regularisation $\gamma = n^{(1+\epsilon)/2}$.  Solid lines indicate the average mean square error over $50$ independent trials for given values of $\epsilon$, with shaded areas denoting the corresponding $95\%$ error bounds.}
  \label{fig:explicit_reg_example}
\end{figure}

For the second example, we now consider the implicit regularisation provided by the covariate noise in the LMM, by setting $\gamma = 0$ and $\sigma_x = 0.1$ (we again assume that the observations $y_i$ have Gaussian noise with $\sigma_y = 0.4$).  In this case, we consider the effect of letting the dimension grow at different rates, by setting $p = \lfloor \sigma_x^2\,n^{1+\alpha}\rfloor$ for varying values of $\alpha$.

For each value of $\alpha$ we ran $50$ independent trials each consisting of generating a training set of size $n$ and evaluating the mean squared error on a test set of $250$ points. These mean squared errors are shown in Figure \ref{fig:implicit_reg_example}, with the shaded areas on the plot denoting the $95\%$ error bounds across the $50$ trials.  We observe that increasing the growth rate $\alpha$ yields better performance, which is in line with our results in Theorem \ref{thm:finite_rank_combined}, from which we would expect the bias and variance terms $B$ and $V$ to decay consistently across all examples, but the residual terms $S_i$ to decay faster for higher values of $\alpha$.

\begin{figure}[ht!]
  \centering
  \includegraphics[trim = 0 0 0 0, clip, scale=0.15]{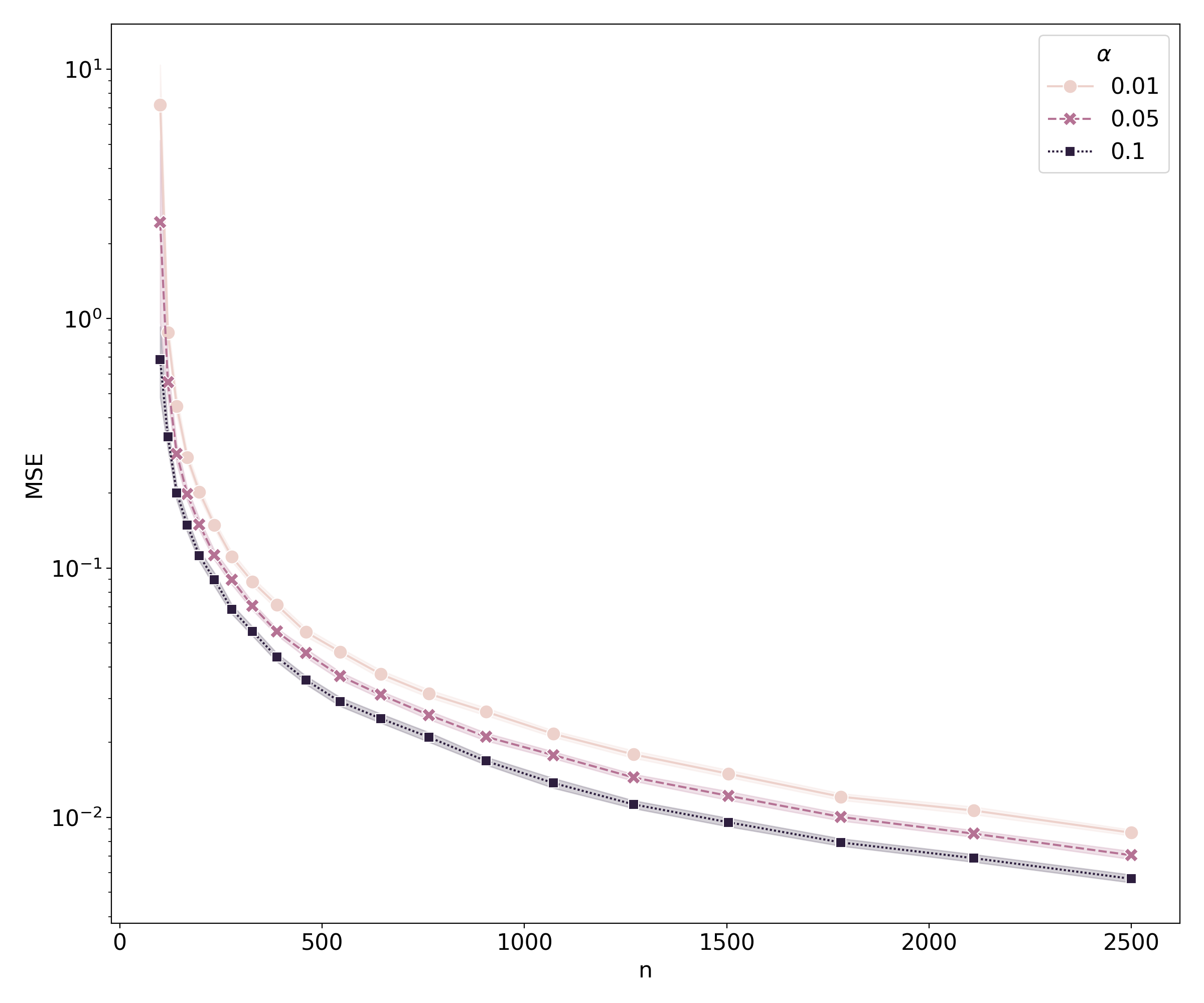}
  \caption{Asymptotic behaviour of the prediction error when performing regression on the target function $g(z) = \cos(3z)$ with implicit regularisation $\sigma_x = 0.0
  1$ and dimension growing at a rate proportional to $n^{1+\alpha}$.  Solid lines indicate the average mean square error over $50$ independent trials for given values of $\alpha$, with shaded areas denoting the corresponding $95\%$ error bounds.}
  \label{fig:implicit_reg_example}
\end{figure}

\subsection{Global temperatures example}

As an application of these ideas to a real-world dataset, we consider a set of time series of average daily temperatures in towns and cities across the world, originating from the Berkeley Earth project \cite{berkelyearth}.  The dataset comprises $2188$ such time series, each containing $p = 1450$ temperature recordings, split across five continents as shown in Table \ref{tab:time_series_distribution} (we note that there were a further $23$ time series for locations in Oceania, which were omitted from our study due to the prohibitively small sample size).

\begin{table}[h!]
\begin{centering}
\small
\begin{tabularx}{0.5\textwidth}{*{2}{Y} }
\toprule
Continent & Number of cities\\
\midrule
Africa & 223\\
Asia & 904\\
Europe & 393\\
North America & 380\\
South America & 288\\
\bottomrule
\end{tabularx}
\par\end{centering}
\caption{\label{tab:time_series_distribution} Number of towns and cities per continent.}
\end{table}

Using this data, we conducted a regression analysis to see if fluctuations in temperature can serve as an accurate predictor for the latitude of a given town or city.  For the $i$th such location, we choose our data vector $\m{x}_i \in \mb{R}^{1450}$ to contain the temperature recordings, standardized for that particular location.  For each continent, we ran 50 trials of unregularized regression, in which the scaled data vector $p^{-1/2}\m{X}$ was constructed for values of $p \in \{1,\ldots,1450\}$ using a random selection of $20\%$ of the cities as a training set for each trial, and we used the root mean square error of the predictions for the remaining $80\%$ of locations as our measure of accuracy.

\begin{figure}[ht!]
  \centering
  \includegraphics[trim = 0 20 0 0, clip, scale=0.25]{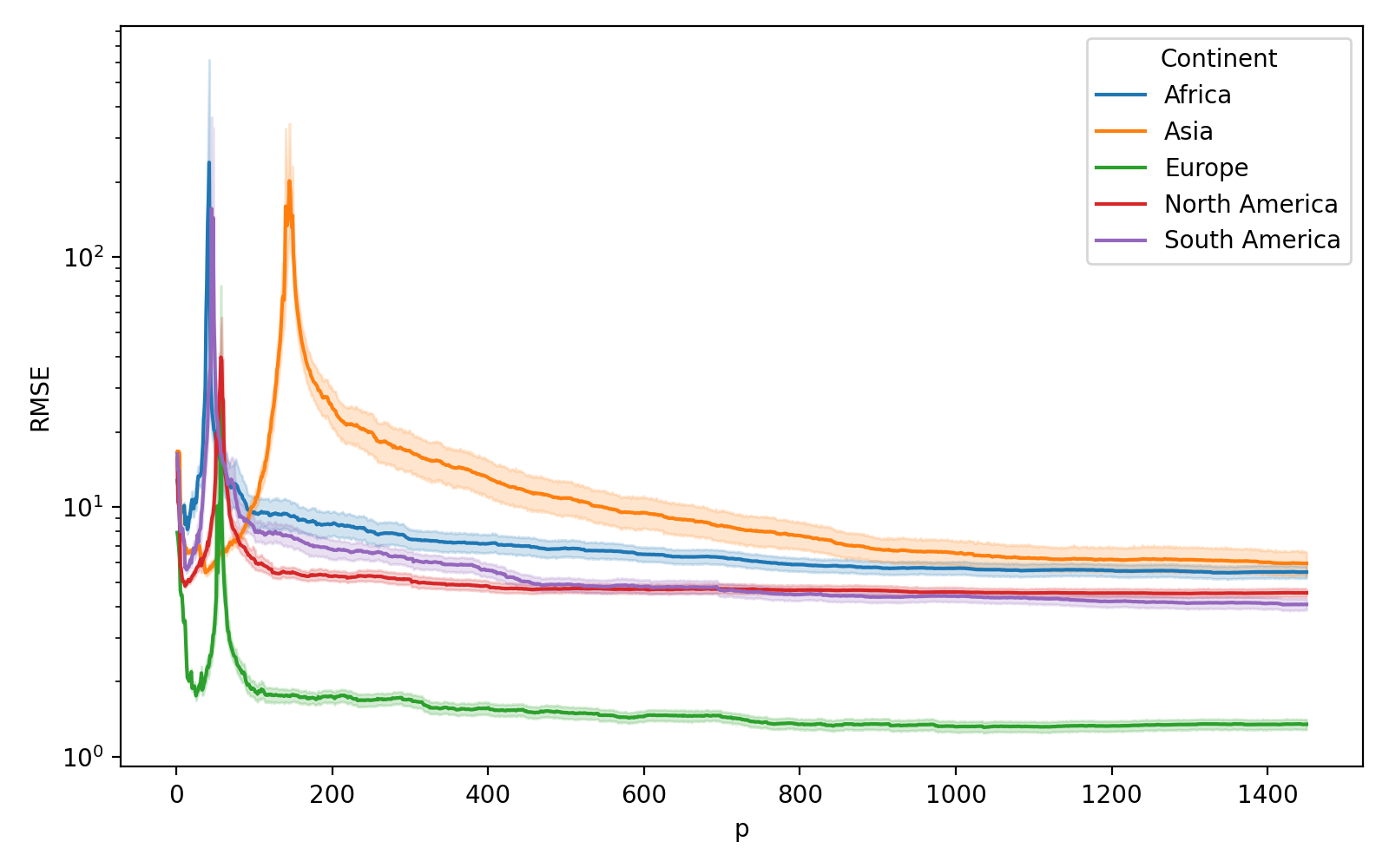}
  \caption{Results of regression analysis using temperature fluctuations to predict latitude of towns and cities in different continents.  Solid lines indicate the average root mean square error over $50$ trials for given values of $p$, with shaded areas denoting the corresponding $95\%$ error bounds.}
  \label{fig_temp_regression}
\end{figure}

The results of this analysis are shown in Figure \ref{fig_temp_regression}, in which the root mean square error for each continent is plotted against the number of temperature covariates $p$ used for the prediction.  In each case, we observe that following an initial peak, the error decreases before seemingly converging to a fixed value as the number of covariates increases.  

\begin{figure}[ht!]
  \centering
  \includegraphics[trim = 0 20 0 0, clip, scale=0.25]{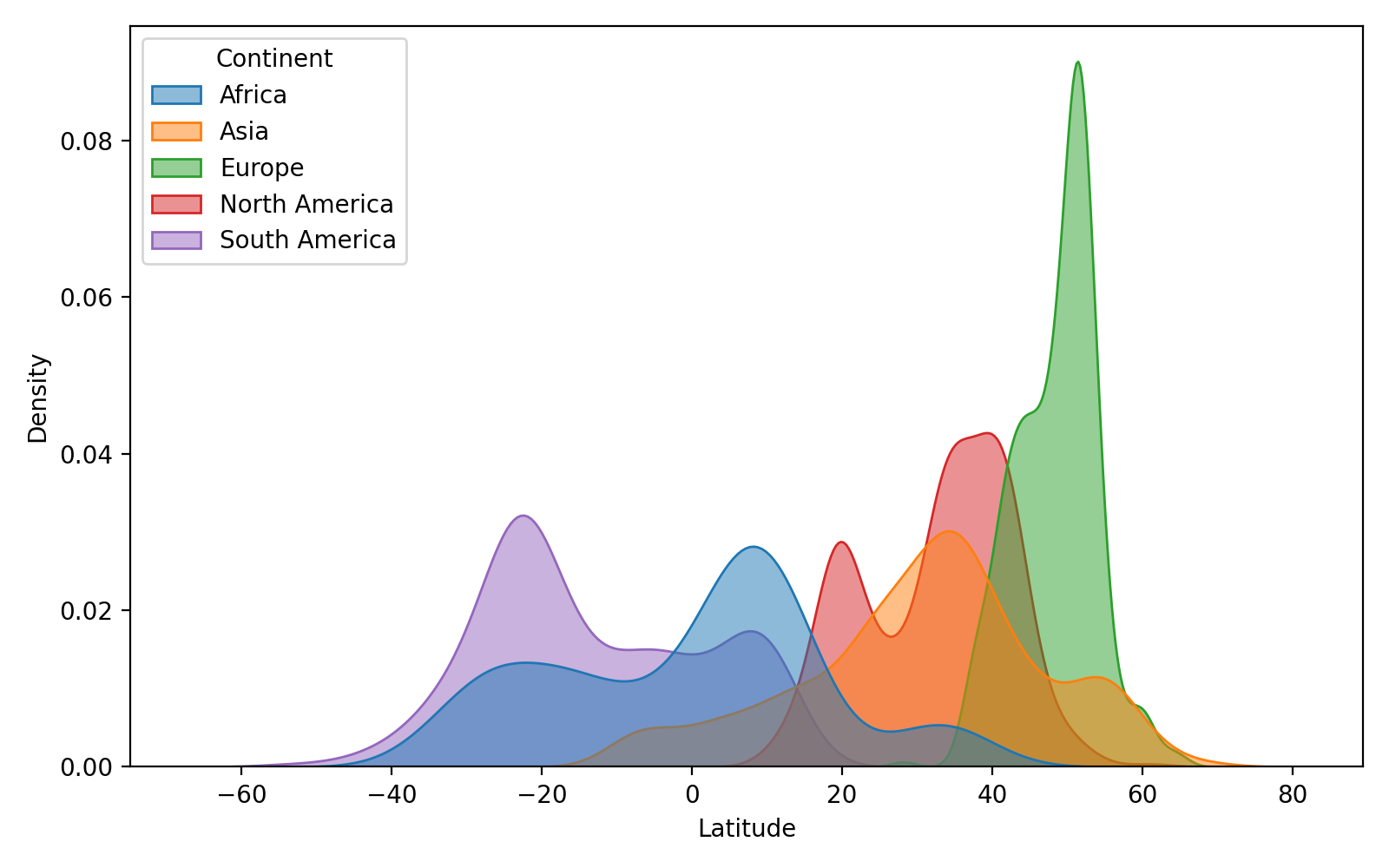}
  \caption{Kernel density estimates of the distribution of latitudes of towns and cities in each continent.}
  \label{fig_latitudes}
\end{figure}

By far the most accurate prediction is obtained by restricting our attention to European cities, which is perhaps unsurprising as the locations are distributed along a much narrower range of latitudes (as demonstrated in Figure \ref{fig_latitudes}) and so one would expect more stability among the temperature recordings than in other continents where the locations are more widely spread.

Figure \ref{fig_heat_maps} depicts the heat maps of the scaled inner product matrices $p^{-1}\m{X}\m{X}^\top$ with $p = 1450$ for each continent, in which the rows of each matrix are ordered according to the latitude of the corresponding town or city.  From this we note that the inner products between locations in Europe are largely consistent, while the other continents display a wider spread of values.

\begin{figure}[ht!]
  \centering
  \includegraphics[trim = 0 20 0 0, scale=0.25]{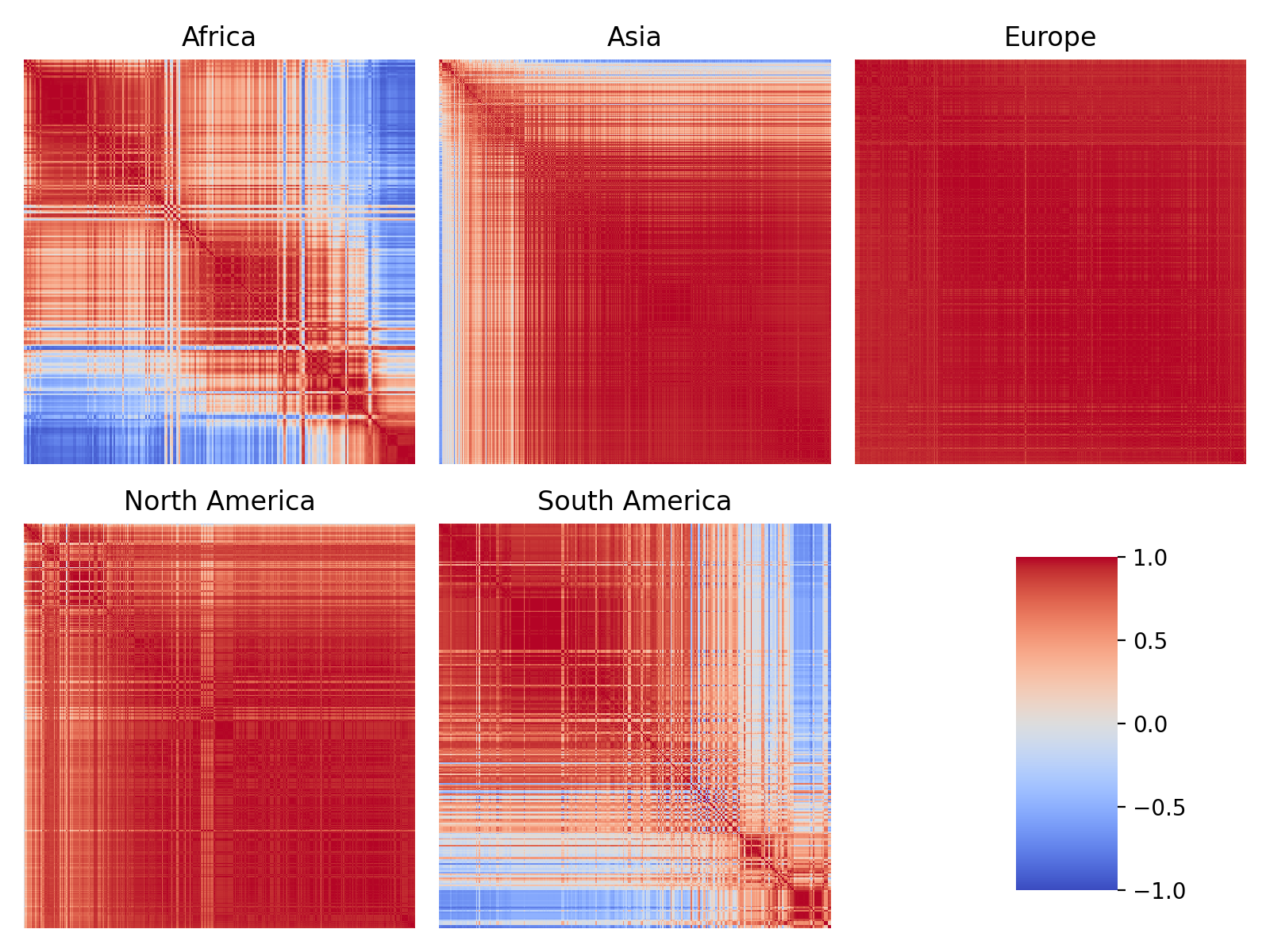}
  \caption{Heat maps of the scaled inner product matrices $p^{-1}\m{X}\m{X}^\top$ with $p = 1450$ for each continent.}
  \label{fig_heat_maps}
\end{figure}

\bibliographystyle{plainnat}
\bibliography{refs}

\clearpage
\newpage\appendix

\section{Proofs}\label{sec:proofs}

\subsection{Proof of Proposition \ref{prop:KL_expansion}}\label{sec:proof_of_KL_expansion}
\begin{proof}
Define 
\begin{equation}
\widetilde{\mathbf{W}}_{jk}\coloneqq\int_{\mathcal{Z}}\psi_{j}(z)u_{k}(z)\mu(\mathrm{d}z),\label{eq:W_tilde_defn}
\end{equation}
and note that
\begin{equation}
\widetilde{\mathbf{W}}_{jk}=p^{1/2}\lambda_{k}^{1/2}\mathbf{W}_{jk},\label{eq:W_tilde_identity}
\end{equation}
where $\mathbf{W}_{jk}$ is defined in \eqref{eq:W_jk_defn}.

Denote the implicit kernel in the LMM by $f(z,z^\prime)=p^{-1}\sum_{j=1}^p\mathbb{E}[\psi_j(z)\psi_j(z^\prime)]$. Recall $\rank\in\{1,2\ldots,\}\cup\{\infty\}$ is the number of nonzero eigenvalues $(\lambda_{k})_{k\geq1}$.  If $\rank<\infty$, pick any $r_{0} \leq \rank$, or if $\rank=\infty$, pick any $r_0<\infty$.  
We claim that, for any $z\in\mathcal{Z}$, the following equality
holds: 
\begin{equation}
\frac{1}{p}\sum_{j=1}^{p}\mathbb{E}\left[\left|\psi_{j}(z)-\sum_{k=1}^{r_{0}}u_{k}(z)\widetilde{\mathbf{W}}_{jk}\right|^{2}\right]=f(z,z)-\sum_{k=1}^{r_{0}}\lambda_{k}|u_{k}(z)|^{2}.\label{eq:L_2_decomp}
\end{equation}
To verify the equality (\ref{eq:L_2_decomp}), observe:
\begin{align}
 & \frac{1}{p}\sum_{j=1}^{p}\mathbb{E}\left[\left|\psi_{j}(z)-\sum_{k=1}^{r_{0}}u_{k}(z)\widetilde{\mathbf{W}}_{jk}\right|^{2}\right]\nonumber\\
 & =\frac{1}{p}\sum_{j=1}^{p}\mathbb{E}\left[\left|\psi_{j}(z)\right|^{2}\right]-\frac{2}{p}\sum_{j=1}^{p}\mathbb{E}\left[\psi_{j}(z)\sum_{k=1}^{r_{0}}u_{k}(z)\widetilde{\mathbf{W}}_{jk}\right]\nonumber\\
 & \quad+\frac{1}{p}\sum_{j=1}^{p}\sum_{k=1}^{r_{0}}\sum_{\ell=1}^{r_{0}}\mathbb{E}\left[\widetilde{\mathbf{W}}_{jk}\widetilde{\mathbf{W}}_{j\ell}\right]u_{k}(z)u_{\ell}(z)\nonumber\\
 & =f(z,z)-2\sum_{k=1}^{r_{0}}u_{k}(z)\int_{\mathcal{Z}}f(z,z^{\prime})u_{k}(z^{\prime})\mu(\mathrm{d}z^{\prime})\nonumber\\
 & \quad+\sum_{k=1}^{r_{0}}\sum_{\ell=1}^{r_{0}}u_{k}(z)u_{\ell}(z)\int_{\mathcal{Z}}\int_{\mathcal{Z}}f(z^{\prime},z^{\prime\prime})u_{k}(z^{\prime})u_{\ell}(z^{\prime\prime})\mu(\mathrm{d}z^{\prime})\mu(\mathrm{d}z^{\prime\prime})\nonumber\\
 & =f(z,z)-2\sum_{k=1}^{r_{0}}\lambda_{k}|u_{k}(z)|^{2}+\sum_{k=1}^{r_{0}}\lambda_{k}|u_{k}(z)|^{2}\nonumber\\
 & =f(z,z)-\sum_{k=1}^{r_{0}}\lambda_{k}|u_{k}(z)|^{2},\label{eq:kl_proof}
\end{align}
where the second equality uses (\ref{eq:W_tilde_defn}) and $f(z,z^{\prime})=p^{-1}\sum_{j=1}^{p}\mathbb{E}[X_{j}(z)X_{j}(z^{\prime})]$,
and the third equality uses the fact that $(u_{k} ,\lambda_{k})_{k\geq1}$,
by definition, are $L_{2}(\mu)$-orthonormal eigenfunctions and eigenvalues
of the integral operator associated with the kernel $f$ and the measure
$\mu$.

Now let $z_1,\ldots,z_n$ be the latent variables in the LMM, i.e., $z_1,\ldots,z_n$ are i.i.d. draws from $\mu$, and  $\psi_1,\ldots,\psi_p$ and $z_1,\ldots,z_n$  are mutually independent. Using this independence, \eqref{eq:kl_proof}, Mercer's theorem  and the fact that $\int_{\mathcal{Z}}|u_k(z)|^2\mu(\mathrm{d}z)=1$ for all $k\geq 1$, we have for $1\leq j\leq p$,
$$
\mathbb{E}\left[\left|\psi_j(z_i) - \sum_{k=1}^{r_0}u_k(z)\widetilde{\mathbf{W}}_{jk}\right|^2\right]\leq p \mathbb{E}\left[f(z_i,z_i)-\sum_{k=1}^{r_0}\lambda_k |u_k(z_i)|^2\right]= p \sum_{k>r_0}\lambda_k.
$$
Hence via Markov's inequality, for any $\delta>0$,
$$
\mathbb{P}\left(\left|\psi_j(z_i) - \sum_{k=1}^{r_0}u_k(z_i)\widetilde{\mathbf{W}}_{jk}\right|>\delta\right) <\frac{1}{\delta^2} \mathbb{E}\left[\left|\psi_j(z_i) - \sum_{k=1}^{r_0}u_k(z_i)\widetilde{\mathbf{W}}_{jk}\right|^2\right] \leq p\sum_{k>r_0}\lambda_k
$$
Since $\sum_{k\geq 1}k \lambda_k = \sum_{k\geq 1}\sum_{\ell\geq k} \lambda_k$, \ref{ass:eig_decay_basic} implies  
$$
\sum_{r_0=1}^{\infty} \mathbb{P}\left(\left|\psi_j(z_i) - \sum_{k=1}^{r_0}u_k(z_i)\widetilde{\mathbf{W}}_{jk}\right|>\delta\right)<\infty,
$$
so by the Borel-Cantelli lemma, $\lim_{r_0\to\infty}\sum_{k=1}^{r_0}u_k(z_i)\widetilde{\mathbf{W}}_{jk} = \psi_j(z_i)$, a.s. Substituting \eqref{eq:W_tilde_identity} and using \eqref{eq:X_model} and the definition of the feature map $\phi$, we have established $\mathbf{X}=p^{1/2}\boldsymbol{\Phi}\mathbf{W}^\top + \sigma_x\mathbf{E}$, a.s. The second almost sure equality in the statement of the proposition holds by the same arguments, since $z_{test}\sim \mu $ is independent of all other random variables. 

The proof of third equality in the statement of the proposition follows by exactly the arguments of \citet[proof of Prop. 1]{whiteley2022statistical}, so the details are omitted.

\end{proof}

\subsection{Proof of Lemma \ref{lem:decomp}}\label{sec:proof_of_decomp}
\begin{proof} Using Proposition \ref{prop:KL_expansion},
\begin{align*}
\mathbf{x}_{test}^{\top}\mathbf{X}^{\top} & =p\phi(z_{test})^{\top}\mathbf{W}^{\top}\mathbf{W}\boldsymbol{\Phi}^{\top}\\
 & \quad+p^{1/2}\phi(z_{test})^{\top}\mathbf{W}^{\top}\sigma_{x}\mathbf{E}^{\top}\\
 & \quad+\sigma_{x}\mathbf{e}_{test}^{\top}\mathbf{X}^{\top}
\end{align*}
and so 

\begin{align*}
p^{-1/2}\mathbf{x}_{test}^{\top}\hat{\beta}(\mathbf{y}) & =p^{-1}\mathbf{x}_{test}^{\top}\mathbf{X}^{\top}\mathbf{A}^{-1}\mathbf{y}\\
 & =\left(\phi(z_{test})^{\top}\boldsymbol{\Phi}^{\top}+p^{-1}\mathbf{x}_{test}^{\top}\mathbf{X}^{\top}-\phi(z_{test})^{\top}\boldsymbol{\Phi}^{\top}\right)\mathbf{A}^{-1}\mathbf{y}\\
 & =\phi(z_{test})^{\top}\boldsymbol{\Phi}^{\top}\mathbf{A}^{-1}\mathbf{y}\\
 & \quad+\phi(z_{test})^{\top}\left(\mathbf{W}^{\top}\mathbf{W}-\mathbf{I}_{r}\right)\boldsymbol{\Phi}^{\top}\mathbf{A}^{-1}\mathbf{y}\\
 & \quad+p^{-1/2}\phi(z_{test})^{\top}\mathbf{W}^{\top}\sigma_{x}\mathbf{E}^{\top}\mathbf{A}^{-1}\mathbf{y}\\
 & \quad+\sigma_{x}p^{-1}\mathbf{e}_{test}^{\top}\mathbf{X}^{\top}\mathbf{A}^{-1}\mathbf{y}.
\end{align*}
By definition of the response model (\ref{eq:response_model}) we
have $\mathbf{y}=\boldsymbol{\Phi}\beta^{*}+\sigma_{y}\boldsymbol{\epsilon}$;
by Assumption \ref{ass:g} we have $g(z_{test})=\phi(z_{test})^{\top}\theta^{*}$;
integrating out $z_{test}$ we have $\mathbb{E}_{z_{test}}\left[\phi(z_{test})\phi(z_{test})^{\top}\right]=\boldsymbol{\Lambda},$
hence :
\[
\mathbb{E}_{z_{test}}\left[\left|\phi(z_{test})^{\top}\left(\boldsymbol{\Phi}^{\top}\mathbf{A}^{-1}\boldsymbol{\Phi}-\mathbf{I}_{r}\right)\theta^{*}\right|^{2}\right]=\left\Vert \hat{\theta}(\boldsymbol{\Phi}\theta^{*})-\theta^{*}\right\Vert _{\boldsymbol{\Lambda}}^{2}
\]
and
\[
\mathbb{E}_{z_{test},\boldsymbol{\epsilon}}\left[\left|\phi(z_{test})^{\top}\boldsymbol{\Phi}^{\top}\mathbf{A}^{-1}\boldsymbol{\epsilon}\right|^{2}\right]=\mathbb{E}_{\boldsymbol{\epsilon}}\left[\left\Vert \boldsymbol{\Phi}^{\top}\mathbf{A}^{-1}\boldsymbol{\epsilon}\right\Vert _{\boldsymbol{\Lambda}}^{2}\right]=\mathbb{E}_{\boldsymbol{\epsilon}}\left[\left\Vert \hat{\theta}(\boldsymbol{\epsilon})\right\Vert _{\boldsymbol{\Lambda}}^{2}\right];
\]
and using these identities together with the property that $\boldsymbol{\epsilon}$
is zero mean and independent of all other random variables we obtain:
\begin{align*}
\frac{1}{4}\mathbb{E}_{z_{test},\mathbf{e}_{test},\boldsymbol{\epsilon}}\left[\left|p^{-1/2}\mathbf{x}_{test}^{\top}\hat{\beta}(\mathbf{y})-g(z{}_{test})\right|^{2}\right] & \leq B+V\\
 & \quad+S_{1}+S_{2}+S_{3}
\end{align*}
 as in the statement of the lemma. 
 
\end{proof}

\subsection{Proof of Theorem \ref{thm:VB_overall_bound}}\label{sec:proof_of_VB_overall_bound}
\begin{proof} 
The proof involves applying Lemmas \ref{lem:V_prob_bound}
and \ref{lem:B_prob_bound} to bound $V$ and $B$, then performing
some manipulations of the quantities appearing in these bounds. We
carry out some preliminary computations to prepare for these bounds
following very closely the line of argument in \citep[proof of Theorem 2]{barzilai2023generalization},
but with some additional numerical constants appearing. 

For any $1\leq i\leq n$ we have $\mu_{i}(\mathbf{A}_{k}-\boldsymbol{\Delta})=\mu_{i}(\mathbf{K}_{>k})+\sigma_{x}^{2}+n\gamma,$
and, by Weyl's inequality, $\max_{1\leq i\leq n}\left|\mu_{i}(\mathbf{A}_{k})-\mu_{i}(\mathbf{A}_{k}-\boldsymbol{\Delta})\right|\leq\|\boldsymbol{\Delta}\|_{2}$.
Combined with the condition: $2\|\boldsymbol{\Delta}\|_{2}\leq\mu_{n}(\mathbf{K}_{>k})+\sigma_{x}^{2}+n\gamma$,
which holds on the complement of the event $C_{p,n}^{(k)}$ , we therefore
have for $1\leq i\leq n$,
\begin{align*}
\frac{2}{3}\mu_{i}\left(\frac{1}{n}\mathbf{A}_{k}\right) & \leq\mu_{i}\left(\frac{1}{n}\mathbf{K}_{>k}\right)+\sigma_{x}^{2}/n+\gamma,\\
2\mu_{i}\left(\frac{1}{n}\mathbf{A}_{k}\right) & \geq\mu_{i}\left(\frac{1}{n}\mathbf{K}_{>k}\right)+\sigma_{x}^{2}/n+\gamma.
\end{align*}
Applying these inequalities we obtain
\begin{equation}
\frac{\mu_{1}(\frac{1}{n}\mathbf{A}_{k})^{2}}{\mu_{n}(\frac{1}{n}\mathbf{A}_{k})^{2}}\leq9\left(\frac{\mu_{1}(\frac{1}{n}\mathbf{K}_{>k})+\sigma_{x}^{2}/n+\gamma}{\mu_{n}(\frac{1}{n}\mathbf{K}_{>k})+\sigma_{x}^{2}/n+\gamma}\right)^{2}\leq9\rho_{k,n}^{2},\label{eq:VB_prop_1}
\end{equation}
\begin{equation}
\frac{\|\boldsymbol{\Lambda}_{>k}\|_{2}}{\mu_{n}(\frac{1}{n}\mathbf{A}_{k})}\leq2\frac{\|\boldsymbol{\Lambda}_{>k}\|_{2}}{\mu_{n}(\frac{1}{n}\mathbf{K}_{>k})+\sigma_{x}^{2}/n+\gamma}\leq2\rho_{k,n},\label{eq:VB_prop_2}
\end{equation}
\begin{equation}
\frac{1}{n}\frac{\sum_{i>k}\lambda_{i}^{2}}{\mu_{n}(\frac{1}{n}\mathbf{A}_{k})^{2}}=\frac{\|\boldsymbol{\Lambda}_{>k}\|_{2}^{2}}{\mu_{n}(\frac{1}{n}\mathbf{A}_{k})^{2}}\frac{r_{k}(\boldsymbol{\Lambda}^{2})}{n}\leq4\rho_{k,n}^{2}\frac{r_{k}(\boldsymbol{\Lambda}^{2})}{n}.\label{eq:VB_prop_3}
\end{equation}

Furthermore, using the fact that $\mu_{n}\left(\frac{1}{n}\mathbf{A}_{k}\right)\leq\frac{1}{n}\mathrm{tr}\left(\frac{1}{n}\mathbf{A}_{k}\right)$, 
\begin{multline}
\mu_{1}\left(\frac{1}{n}\mathbf{A}_{k}\right)^{2}=\frac{\mu_{1}(\frac{1}{n}\mathbf{A}_{k})^{2}}{\mu_{n}(\frac{1}{n}\mathbf{A}_{k})^{2}}\mu_{n}\left(\frac{1}{n}\mathbf{A}_{k}\right)^{2}\leq9\rho_{k,n}^{2}\left[\frac{1}{n}\mathrm{tr}\left(\frac{1}{n}\mathbf{A}_{k}\right)\right]^{2}\\
\leq9\rho_{k,n}^{2}\frac{9}{4}\left(\frac{1}{n^{2}}\sum_{i=1}^{n}\sum_{j>k}\lambda_{j}|u_{j}(z_{i})|^{2}+\frac{\sigma_{x}^{2}}{n}+\gamma\right)^{2}\leq21\rho_{k,n}^{2}\left(\frac{\beta_{k}\mathrm{tr}(\boldsymbol{\Lambda}_{>k})}{n}+\frac{\sigma_{x}^{2}}{n}+\gamma\right)^{2},\label{eq:VB_prop_4}
\end{multline}
similarly
\begin{multline}
\mu_{n}\left(\frac{1}{n}\mathbf{A}_{k}\right)^{2}=\frac{\mu_{n}(\frac{1}{n}\mathbf{A}_{k})^{2}}{\mu_{1}(\frac{1}{n}\mathbf{A}_{k})^{2}}\mu_{1}\left(\frac{1}{n}\mathbf{A}_{k}\right)^{2}\geq\frac{1}{9}\frac{1}{\rho_{k,n}^{2}}\left[\frac{1}{n}\mathrm{tr}\left(\frac{1}{n}\mathbf{A}_{k}\right)\right]^{2}\\
\geq\frac{1}{9}\frac{1}{\rho_{k,n}^{2}}\frac{1}{2}\left(\frac{1}{n^{2}}\sum_{i=1}^{n}\sum_{j>k}\lambda_{j}u_{j}(z_{i})^{2}+\frac{\sigma_{x}^{2}}{n}+\gamma\right)^{2}\geq\frac{1}{18}\frac{1}{\rho_{k,n}^{2}}\left(\frac{\alpha_{k}\mathrm{tr}(\boldsymbol{\Lambda}_{>k})}{n}+\frac{\sigma_{x}^{2}}{n}+\gamma\right)^{2},\label{eq:VB_prop_5}
\end{multline}
and so
\begin{equation}
\frac{1}{n}\frac{\sum_{i>k}\lambda_{i}^{2}}{\mu_{n}(\frac{1}{n}\mathbf{A}_{k})^{2}}\leq18 \rho_{k,n}^{2}\frac{n\sum_{i>k}\lambda_{i}^{2}}{\left(\alpha_{k}\mathrm{tr}(\boldsymbol{\Lambda}_{>k})+\sigma_{x}^{2}+n\gamma\right)^{2}}\leq18\rho_{k,n}^{2}\frac{n}{R_{k}(\boldsymbol{\Lambda})}\frac{\mathrm{tr}(\boldsymbol{\Lambda}_{>k})^2}{\left(\alpha_k\mathrm{tr}(\boldsymbol{\Lambda}_{>k})+\sigma_x^2+n\gamma\right)^2}.\label{eq:VB_prop_6}
\end{equation}

Combining Lemma \ref{lem:V_prob_bound} with (\ref{eq:VB_prop_1}),
(\ref{eq:VB_prop_3}) and (\ref{eq:VB_prop_6}), we have with probability
at least $1-8\exp\left(-\frac{c^{\prime}}{\beta_{k}^{2}}\frac{n}{k}\right)-\mathbb{P}(C_{p,n}^{(k)})-\mathbb{P}(D_{n}^{(k)})$,
\[
V\leq c_{1}\rho_{k,n}^{2}\sigma_{y}^{2}\left[\frac{k}{n}+\min\left\{\frac{r_{k}(\boldsymbol{\Lambda}^{2})}{n},\left(\frac{n}{R_{k}(\boldsymbol{\Lambda})}\frac{\mathrm{tr}(\boldsymbol{\Lambda}_{>k})^2}{\left(\alpha_k\mathrm{tr}(\boldsymbol{\Lambda}_{>k})+\sigma_x^2+n\gamma\right)^2}\right)\right\}\right],
\]
with the numerical constants in (\ref{eq:VB_prop_1}), (\ref{eq:VB_prop_3})
and (\ref{eq:VB_prop_6}) absorbed into $c_{1}$.

Combining Lemma \ref{lem:B_prob_bound} with (\ref{eq:VB_prop_1}),
(\ref{eq:VB_prop_2}), (\ref{eq:VB_prop_4}), and using $\rho_{k,n}>1$
we have that with probability at least $1-\delta-8\exp\left(-\frac{c^{\prime}}{\beta_{k}^{2}}\frac{n}{k}\right)-\mathbb{P}(C_{p,n}^{(k)})-\mathbb{P}(D_{n}^{(k)})$,
\[
B\leq c_{2}\rho_{k,n}^{3}\left[\frac{1}{\delta}\|\theta_{>k}^{*}\|_{\boldsymbol{\Lambda}_{>k}}^{2}+\|\theta_{\leq k}^{*}\|_{\boldsymbol{\Lambda}_{\leq k}^{-1}}^{2}\left(\frac{\beta_{k}\mathrm{tr}(\boldsymbol{\Lambda}_{>k})}{n}+\frac{\sigma_{x}^{2}}{n}+\gamma\right)^{2}\right],
\]
with the numerical constants in (\ref{eq:VB_prop_1}), (\ref{eq:VB_prop_2}),
(\ref{eq:VB_prop_4})  absorbed into $c_{2}$. The proof is completed
by a union bound.
\end{proof}
\begin{lem}
\label{lem:A_k_pos_def}For any $0\leq k<n$,
\[
\overline{C_{p,n}^{(k)}}\cap\overline{D_{n}^{(k)}}\subseteq\left\{ \mu_{n}(\mathbf{A}_{k})>0\right\} ,
\]
with the convention that $\mathbf{A}_{0}\equiv\mathbf{A}$. 
\end{lem}
\begin{proof}
By an application of Weyl's inequality,
\[
\mu_{n}(\mathbf{A}_{k})\geq\mu_{n}(\mathbf{A}_{k}-\boldsymbol{\Delta})-\|\boldsymbol{\Delta}\|_{2}=\mu_{n}(\mathbf{K}_{>k})+\sigma_{x}^{2}+n\gamma-\|\boldsymbol{\Delta}\|_{2}.
\]
It follows from the definitions of the events $C_{p,n}^{(k)}$ and
$D_{n}^{(k)}$ that: 
\begin{align*}
\overline{C_{p,n}^{(k)}}\cap\overline{D_{n}^{(k)}} & \subseteq\left\{ \mu_{n}(\mathbf{A}_{k})>0\right\} .
\end{align*}
\end{proof}
\begin{lem}
\label{lem:=00005Ctheta_hat_identity}If for some $1\leq k<n$, $\mathbf{A}_{k}$
is positive definite, then for any $\mathbf{y}\in\mathbb{R}^{n},$
\[
\hat{\theta}(\mathbf{y})_{\leq k}-\boldsymbol{\Phi}_{\leq k}^{\top}\mathbf{A}_{k}^{-1}\boldsymbol{\Phi}_{\leq k}\hat{\theta}(\mathbf{y})_{\leq k}=\boldsymbol{\Phi}_{\leq k}^{\top}\mathbf{A}_{k}^{-1}\mathbf{y}.
\]
\end{lem}
\begin{proof}
Noting that $\mathbf{A}_{k}\succ\mathbf{0}$ implies $\mathbf{A}=\mathbf{K}_{\leq k}+\mathbf{A}_{k}\succ\mathbf{0}$,
we have
\begin{align*}
 & \hat{\theta}(\mathbf{y})_{\leq k}-\boldsymbol{\Phi}_{\leq k}^{\top}\mathbf{A}_{k}^{-1}\boldsymbol{\Phi}_{\leq k}\hat{\theta}(\mathbf{y})_{\leq k}\\
 & =\boldsymbol{\Phi}_{\leq k}^{\top}(\mathbf{K}_{\leq k}+\mathbf{A}_{k})^{-1}\mathbf{y}-\boldsymbol{\Phi}_{\leq k}^{\top}\mathbf{A}_{k}^{-1}\boldsymbol{\Phi}_{\leq k}\boldsymbol{\Phi}_{\leq k}^{\top}(\mathbf{K}_{\leq k}+\mathbf{A}_{k})^{-1}\mathbf{y}\\
 & =\boldsymbol{\Phi}_{\leq k}^{\top}\mathbf{A}_{k}^{-1}\left(\mathbf{A}_{k}+\boldsymbol{\Phi}_{\leq k}\boldsymbol{\Phi}_{\leq k}^{\top}\right)(\mathbf{K}_{\leq k}+\mathbf{A}_{k})^{-1}\mathbf{y}\\
 & =\boldsymbol{\Phi}_{\leq k}^{\top}\mathbf{A}_{k}^{-1}\mathbf{y},
\end{align*}
where the final equality uses $\mathbf{K}_{\leq k}=\boldsymbol{\Phi}_{\leq k}\boldsymbol{\Phi}_{\leq k}^{\top}.$ 
\end{proof}
\begin{lem}
\label{lem:V_determ_bound}If for some $k<n$, the matrix $\mathbf{A}_{k}$
is positive definite, then
\[
V\leq\sigma_{y}^{2}\left(\frac{\mu_{1}(\mathbf{A}_{k}^{-1})\mathrm{tr}\left(\mathbf{U}_{\leq k}\mathbf{U}_{\leq k}^{\top}\right)}{\mu_{n}(\mathbf{A}_{k}^{-1})\mu_{k}(\mathbf{U}_{\leq k}^{\top}\mathbf{U}_{\leq k})^{2}}+\mu_{1}(\mathbf{A}_{k}^{-1})^{2}\mathrm{tr}\left(\boldsymbol{\Phi}_{>k}\boldsymbol{\Lambda}_{>k}\boldsymbol{\Phi}_{>k}^{\top}\right)\right).
\]
\end{lem}
\begin{proof}
The proof follows exactly the same manipulations as \citep[proof of Lemma 13]{barzilai2023generalization},
which apply unchanged with our definitions of $\mathbf{A}_{k}$ and
$\mathbf{A}$ in place, making use of our Lemma \ref{lem:=00005Ctheta_hat_identity}
in place of \citep[Lemma 11]{barzilai2023generalization}.
\end{proof}
\begin{lem}
\label{lem:V_prob_bound}There exist some absolute constants $c$,
$c^{\prime}$, and $c_{1}$ such that for any $k<\rank$ with
$c\beta_{k}k\log(k)\leq n,$ it holds with probability at least $1-8\exp\left(-\frac{c^{\prime}}{\beta_{k}^{2}}\frac{n}{k}\right)-\mathbb{P}(C_{p,n}^{(k)})-\mathbb{P}(D_{n}^{(k)})$
that 
\[
V\leq c_{1}\sigma_{y}^{2}\left(\frac{\mu_{1}(\frac{1}{n}\mathbf{A}_{k})}{\mu_{n}(\frac{1}{n}\mathbf{A}_{k})}\frac{k}{n}+\frac{1}{n}\frac{\sum_{i>k}\lambda_{i}^{2}}{\mu_{n}(\frac{1}{n}\mathbf{A}_{k})^{2}}\right).
\]
\end{lem}
\begin{proof}
By Lemma \ref{lem:A_k_pos_def}, on the intersection of the complements
of $C_{n,p}^{(k)}$ and $D_{n}^{(k)}$ the matrix $\mathbf{A}_{k}$
is positive definite and thus the bound in the statement of Lemma
\ref{lem:V_determ_bound} holds with probability at least $1-\mathbb{P}(C_{p,n}^{(k)})-\mathbb{P}(D_{n}^{(k)})$.
The terms $\mathrm{tr}\left(\mathbf{U}_{\leq k}\mathbf{U}_{\leq k}^{\top}\right)/\mu_{k}(\mathbf{U}_{\leq k}^{\top}\mathbf{U}_{\leq k})^{2}$
and $\mathrm{tr}\left(\boldsymbol{\Phi}_{>k}\boldsymbol{\Lambda}_{>k}\boldsymbol{\Phi}_{>k}^{\top}\right)$
in this bound are controlled using exactly the same method as in \citep[proof of lemma 9]{barzilai2023generalization},
namely the probabilistic inequalities \citep[Lemma 4]{barzilai2023generalization},
which combined with a union bound completes the proof. We note that
in \citep[proof of lemma 9]{barzilai2023generalization} it is assumed
that $\gamma>0$, we do not need this assumption because we work on
the intersection of the complements of $C_{n,p}^{(k)}$ and $D_{n}^{(k)}$.
\end{proof}
\begin{lem}
\label{lem:B_determ_bound}If for some $k< n $, the matrix $\mathbf{A}_{k}$
is positive definite and $2\|\boldsymbol{\Delta}\|_{2}\leq\mu_{n}(\mathbf{K}_{>k})+\sigma_{x}^{2}+n\gamma$,
then
\begin{align*}
\frac{1}{6}B & \leq\frac{\mu_{1}(\mathbf{A}_{k}^{-1})^{2}}{\mu_{n}(\mathbf{A}_{k}^{-1})^{2}}\frac{\mu_{1}(\mathbf{U}_{\leq k}^{\top}\mathbf{U}_{\leq k})}{\mu_{k}(\mathbf{U}_{\leq k}^{\top}\mathbf{U}_{\leq k})^{2}}\left\Vert \boldsymbol{\Phi}_{>k}\theta_{>k}^{*}\right\Vert _{2}^{2}
+\frac{\|\theta_{\leq k}^{*}\|_{\boldsymbol{\Lambda}_{\leq k}^{-1}}^{2}}{\mu_{n}(\mathbf{A}_{k}^{-1})^{2}\mu_k\left(\mathbf{U}_{\leq k}^{\top}\mathbf{U}_{\leq k}\right)^{2}}\\
 & \quad+\|\theta_{>k}^{*}\|_{\boldsymbol{\Lambda}_{>k}}^{2}\\
 & \quad+\|\boldsymbol{\Lambda}_{>k}\|_{2}\mu_{1}(\mathbf{A}_{k}^{-1})\left\Vert \boldsymbol{\Phi}_{>k}\theta_{>k}^{*}\right\Vert _{2}^{2}\\
 & \quad+\|\boldsymbol{\Lambda}_{>k}\|_{2}\frac{\mu_{1}(\mathbf{A}_{k}^{-1})}{\mu_{n}(\mathbf{A}_{k}^{-1})^{2}}\frac{\mu_{1}(\mathbf{U}_{\leq k}^{\top}\mathbf{U}_{\leq k})}{\mu_{k}(\mathbf{U}_{\leq k}^{\top}\mathbf{U}_{\leq k})^{2}}\|\theta_{\leq k}^{*}\|_{\boldsymbol{\Lambda}_{\leq k}^{-1}}^{2}.
\end{align*}
\end{lem}
\begin{proof}
The proof largely follows that of \citep[proof of Lemma 14]{barzilai2023generalization},
starting from the decomposition 
\[
B=\|\hat{\theta}_{\leq k}(\boldsymbol{\Phi}\theta^{*})-\theta_{\leq k}^{*}\|_{\boldsymbol{\Lambda}_{\leq k}}^{2}+\|\hat{\theta}_{>k}(\boldsymbol{\Phi}\theta^{*})-\theta_{>k}^{*}\|_{\boldsymbol{\Lambda}_{>k}}^{2}.
\]
The first two terms in the bound on $B$ in the statement of the lemma
to be proved are derived from bounding $\|\hat{\theta}_{\leq k}(\boldsymbol{\Phi}\theta^{*})-\theta_{\leq k}^{*}\|_{\boldsymbol{\Lambda}_{\leq k}}^{2}$
using exactly the same arguments as in \citep[proof of Lemma 14]{barzilai2023generalization}
and applying our Lemma \ref{lem:=00005Ctheta_hat_identity}, so the
details are omitted.

The term $\|\hat{\theta}_{>k}(\boldsymbol{\Phi}\theta^{*})-\theta_{>k}^{*}\|_{\boldsymbol{\Lambda}_{>k}}^{2}$
is bounded using almost exactly the same arguments as in \citep[proof of Lemma 14]{barzilai2023generalization}.
We start from the elementary upper-bound:
\begin{align}
\frac{1}{3}\|\hat{\theta}_{>k}(\boldsymbol{\Phi}\theta^{*})-\theta_{>k}^{*}\|_{\boldsymbol{\Lambda}_{>k}}^{2} & \leq\|\theta_{>k}^{*}\|_{\boldsymbol{\Lambda}_{>k}}^{2}\nonumber \\
 & \quad+\|\boldsymbol{\Phi}_{>k}^{\top}\mathbf{A}^{-1}\boldsymbol{\Phi}_{>k}\theta_{>k}^{*}\|_{\boldsymbol{\Lambda}_{>k}}^{2}+\|\boldsymbol{\Phi}_{>k}^{\top}\mathbf{A}^{-1}\boldsymbol{\Phi}_{\leq k}\theta_{\leq k}^{*}\|_{\boldsymbol{\Lambda}_{>k}}^{2}.\label{eq:bias_decomp}
\end{align}
The first term on the r.h.s. of (\ref{eq:bias_decomp}) gives the
term $\|\theta_{>k}^{*}\|_{\boldsymbol{\Lambda}_{>k}}^{2}$ appearing
in the bound on $B$ in the statement of the lemma to be proved. For
the second term in (\ref{eq:bias_decomp}), 
\begin{align}
\|\boldsymbol{\Phi}_{>k}^{\top}\mathbf{A}^{-1}\boldsymbol{\Phi}_{>k}\theta_{>k}^{*}\|_{\boldsymbol{\Lambda}_{>k}}^{2} & \leq\|\boldsymbol{\Lambda}_{>k}\|_{2}\left\Vert \boldsymbol{\Phi}_{>k}^{\top}\mathbf{A}^{-1}\boldsymbol{\Phi}_{>k}\theta_{>k}^{*}\right\Vert _{2}^{2}\label{eq:bias_decomp_inter_2}\\
 & =\|\boldsymbol{\Lambda}_{>k}\|_{2}(\theta_{>k}^{*})^{\top}\boldsymbol{\Phi}_{>k}^{\top}\mathbf{A}^{-1}\boldsymbol{\Phi}_{>k}\boldsymbol{\Phi}_{>k}^{\top}\mathbf{A}^{-1}\boldsymbol{\Phi}_{>k}\theta_{>k}^{*}\nonumber 
\end{align}
and using 
\[
\boldsymbol{\Phi}_{>k}\boldsymbol{\Phi}_{>k}^{\top}=\mathbf{A}-\mathbf{K}_{\leq k}-(\sigma_{x}^{2}+n\gamma)\mathbf{I}_{n}-\boldsymbol{\Delta},
\]
together with $\mathbf{K}_{\leq k}\succeq\mathbf{0}$ and $\mu_{1}(\mathbf{A}^{-1})=\mu_{n}(\mathbf{A})^{-1}\leq\mu_{n}(\mathbf{A}_{k})^{-1}$
we have:
\begin{align}
(\theta_{>k}^{*})^{\top}\boldsymbol{\Phi}_{>k}^{\top}\mathbf{A}^{-1}\boldsymbol{\Phi}_{>k}\boldsymbol{\Phi}_{>k}^{\top}\mathbf{A}^{-1}\boldsymbol{\Phi}_{>k}\theta_{>k}^{*} & \leq(\theta_{>k}^{*})^{\top}\boldsymbol{\Phi}_{>k}^{\top}\mathbf{A}^{-1}\boldsymbol{\Phi}_{>k}\theta_{>k}^{*}+\|\mathbf{A}^{-1}\boldsymbol{\Phi}_{>k}\theta_{>k}^{*}\|_{2}^{2}\|\boldsymbol{\Delta}\|_{2}\nonumber \\
 & \leq\frac{1}{\mu_{n}(\mathbf{A}_{k})}\|\boldsymbol{\Phi}_{>k}\theta_{>k}^{*}\|_{2}^{2}+\frac{1}{\mu_{n}(\mathbf{A}_{k})}\frac{\|\boldsymbol{\Delta}\|_{2}}{\mu_{n}(\mathbf{A})}\|\boldsymbol{\Phi}_{>k}\theta_{>k}^{*}\|_{2}^{2}.\label{eq:bias_decomp_inter}
\end{align}
Now by application of Weyl's inequality, $\mu_{n}(\mathbf{A})\geq\mu_{n}(\mathbf{K})+\sigma_{x}^{2}+n\gamma-\|\boldsymbol{\Delta}\|_{2}$,
furthermore $\mu_{n}(\mathbf{K})\geq\mu_{n}(\mathbf{K}_{>k})$ and
by assumption of the lemma, $2\|\boldsymbol{\Delta}\|_{2}\leq\mu_{n}(\mathbf{K}_{>k})+\sigma_{x}^{2}+n\gamma$,
therefore 
\[
\frac{\|\boldsymbol{\Delta}\|_{2}}{\mu_{n}(\mathbf{A})}\leq2\frac{\|\boldsymbol{\Delta}\|_{2}}{\mu_{n}(\mathbf{K}_{>k})+\sigma_{x}^{2}+n\gamma}\leq1.
\]
Substituting into (\ref{eq:bias_decomp_inter}) and returning to (\ref{eq:bias_decomp_inter_2})
we have established:
\[
\|\boldsymbol{\Phi}_{>k}^{\top}\mathbf{A}^{-1}\boldsymbol{\Phi}_{>k}\theta_{>k}^{*}\|_{\boldsymbol{\Lambda}_{>k}}^{2}\leq2\frac{\|\boldsymbol{\Lambda}_{>k}\|_{2}}{\mu_{n}(\mathbf{A}_{k})}\|\boldsymbol{\Phi}_{>k}\theta_{>k}^{*}\|_{2}^{2},
\]
which is the fourth term in the bound on $B$ in the statement of
the lemma. 

The third term on the r.h.s. of (\ref{eq:bias_decomp}) is dealt with
by very similar manipulations to \citep[proof of Lemma 14]{barzilai2023generalization},
so we just highlight the key differences. They use a Sherman-Morrison
argument together with the identity $\mathbf{A}=\mathbf{K}_{\leq k}+\mathbf{A}_{k}$,
which holds in our setting too, and some elementary properties of
norms to derive:
\begin{align}
\|\boldsymbol{\Phi}_{>k}^{\top}\mathbf{A}^{-1}\boldsymbol{\Phi}_{\leq k}\theta_{\leq k}^{*}\|_{\boldsymbol{\Lambda}_{>k}}^{2} & \leq\|\boldsymbol{\Lambda}_{>k}\|_{2}\left\Vert \mathbf{A}_{k}^{-1/2}\boldsymbol{\Phi}_{>k}\boldsymbol{\Phi}_{>k}^{\top}\mathbf{A}_{k}^{-1/2}\right\Vert _{2}\frac{\mu_{1}(\mathbf{A}_{k}^{-1})}{\mu_{n}(\mathbf{A}_{k}^{-1})^{2}}\frac{\mu_{1}(\mathbf{U}_{\leq k}^{\top}\mathbf{U}_{\leq k})}{\mu_{k}(\mathbf{U}_{\leq k}^{\top}\mathbf{U}_{\leq k})^{2}}\|\theta_{\leq k}^{*}\|_{\boldsymbol{\Lambda}_{\leq k}^{-1}}^{2}.\label{eq:bias_bound_penultimate}
\end{align}
In order to bound the term $\left\Vert \mathbf{A}_{k}^{-1/2}\boldsymbol{\Phi}_{>k}\boldsymbol{\Phi}_{>k}^{\top}\mathbf{A}_{k}^{-1/2}\right\Vert _{2}$,
we use the definition of $\mathbf{A}_{k}$, that is $\boldsymbol{\Phi}_{>k}\boldsymbol{\Phi}_{>k}^{\top}=\mathbf{K}_{>k}=\mathbf{A}_{k}-(\sigma_{x}^{2}+n\gamma)\mathbf{I}_{n}-\boldsymbol{\Delta}$,
to give
\begin{equation}
\left\Vert \mathbf{A}_{k}^{-1/2}\boldsymbol{\Phi}_{>k}\boldsymbol{\Phi}_{>k}^{\top}\mathbf{A}_{k}^{-1/2}\right\Vert _{2}\leq\left\Vert \mathbf{I}_{n}-(\sigma_{x}^{2}+n\gamma)\mathbf{A}_{k}^{-1}\right\Vert _{2}+\frac{\left\Vert \boldsymbol{\Delta}\right\Vert _{2}}{\mu_{n}(\mathbf{A}_{k})}.\label{eq:bias_bound_final}
\end{equation}
Using the assumption of the lemma that $\mathbf{A}_{k}$ is positive
definite together with the definition of $\mathbf{A}_{k}$ we have
$\left\Vert \mathbf{I}_{n}-(\sigma_{x}^{2}+n\gamma)\mathbf{A}_{k}^{-1}\right\Vert _{2}\leq1$.
Using the assumption of the lemma that $2\left\Vert \boldsymbol{\Delta}\right\Vert _{2}\leq\mu_{n}(\mathbf{K}_{>k})+\sigma_{x}^{2}+n\gamma$,
we have via Weyl's inequality that $2\mu_{n}(\mathbf{A}_{k})\geq\mu_{n}(\mathbf{A}_{k}-\boldsymbol{\Delta})-\|\boldsymbol{\Delta}\|_{2}\geq\mu_{n}(\mathbf{K}_{>k})+\sigma_{x}^{2}+n\gamma$,
hence $\left\Vert \boldsymbol{\Delta}\right\Vert _{2}/\mu_{n}(\mathbf{A}_{k})\leq1$.
Therefore (\ref{eq:bias_bound_final}) yields
\[
\left\Vert \mathbf{A}_{k}^{-1/2}\boldsymbol{\Phi}_{>k}\boldsymbol{\Phi}_{>k}^{\top}\mathbf{A}_{k}^{-1/2}\right\Vert _{2}\leq2.
\]
 Combining the above bounds and returning to (\ref{eq:bias_bound_penultimate})
gives
\[
\|\boldsymbol{\Phi}_{>k}^{\top}\mathbf{A}^{-1}\boldsymbol{\Phi}_{\leq k}\theta_{\leq k}^{*}\|_{\boldsymbol{\Lambda}_{>k}}^{2}\leq2\|\boldsymbol{\Lambda}_{>k}\|_{2}\frac{\mu_{1}(\mathbf{A}_{k}^{-1})}{\mu_{n}(\mathbf{A}_{k}^{-1})^{2}}\frac{\mu_{1}(\mathbf{U}_{\leq k}^{\top}\mathbf{U}_{\leq k})}{\mu_{k}(\mathbf{U}_{\leq k}^{\top}\mathbf{U}_{\leq k})^{2}}\|\theta_{\leq k}^{*}\|_{\boldsymbol{\Lambda}_{\leq k}^{-1}}^{2},
\]
thus bounding the third term on the r.h.s. of (\ref{eq:bias_decomp}),
which in turn yields the final term in the bound on $B$ in the statement
of the lemma.
\end{proof}
\begin{lem}
\label{lem:B_prob_bound}There exist absolute constants $c,c^{\prime}$
and $c_{2}>0$ such that for any $k<\rank$ with $c\beta_{k}k\log k\leq n$
and $\delta>0$, it holds with probability at least $1-\delta-8\exp\left(-\frac{c^{\prime}}{\beta_{k}^{2}}\frac{n}{k}\right)-\mathbb{P}(C_{p,n}^{(k)})-\mathbb{P}(D_{n}^{(k)})$
that 
\[
B\leq c_{2}\left(\|\theta_{>k}^{*}\|_{\boldsymbol{\Lambda}_{>k}}^{2}\left[1+\frac{1}{\delta}\left(\frac{\mu_{1}(\mathbf{A}_{k}^{-1})^{2}}{\mu_{n}(\mathbf{A}_{k}^{-1})^{2}}+\frac{\|\boldsymbol{\Lambda}_{>k}\|_{2}}{\mu_{n}(\frac{1}{n}\mathbf{A}_{k})}\right)\right]+\|\theta_{\leq k}^{*}\|_{\boldsymbol{\Lambda}_{\leq k}^{-1}}^{2}\left[\mu_{1}\left(\frac{1}{n}\mathbf{A}_{k}\right)^{2}\left(1+\frac{\|\boldsymbol{\Lambda}_{>k}\|_{2}}{\mu_{n}\left(\frac{1}{n}\mathbf{A}_{k}\right)}\right)\right]\right).
\]
\end{lem}
\begin{proof}
Using Lemma \ref{lem:A_k_pos_def}, the bound of Lemma \ref{lem:B_determ_bound}
holds with probability at least $1-\mathbb{P}(C_{p,n}^{(k)})-\mathbb{P}(D_{n}^{(k)})$.
The proof is completed following exactly the same steps as in \citep[proof of lemma 10]{barzilai2023generalization},
namely by applying the bounds of \citep[lemmas 3 and 4]{barzilai2023generalization},
and then taking a union bound.
\end{proof}

\subsection{Proof of Theorem \ref{thm:VB_finite_rank}}\label{sec:proof_of_VB_finite_rank}
\begin{proof}
Throughout the proof $c,c^\prime,c_1,c_2$ are numerical constants whose
value may change on each appearance. By Lemma \ref{lem:V_determ_bound} with $k=\rank$,
if the matrix $\mathbf{A}_{\rank}=(\sigma_{x}^{2}+n\gamma)\mathbf{I}_{n}+\boldsymbol{\Delta}$
is positive definite then
\begin{equation}
V\leq\sigma_{y}^{2}\frac{\mu_{1}(\mathbf{A}_{\rank}^{-1})\mathrm{tr}(\mathbf{U}_{\leq \rank}^{\top}\mathbf{U}_{\leq \rank})}{\mu_{n}(\mathbf{A}_{\rank}^{-1})\mu_{\rank}(\mathbf{U}_{\leq \rank}^{\top}\mathbf{U}_{\leq \rank})^{2}}.\label{eq:finite_k_max_V_bound}
\end{equation}
Since by assumption $\max(\sigma_{x},\gamma)>0$, the event $D_{n}^{(\rank)}$
occurs with probability zero. On the complement of the event $C_{p,n}^{(\rank)}$
the matrix $\mathbf{A}_{\rank}$ is positive definite. Therefore
with probability at least $1-\mathbb{P}(C_{p,n}^{(\rank)})$, (\ref{eq:finite_k_max_V_bound})
holds and simultaneously, using the definition of $C_{p,n}^{(\rank)}$
\eqref{eq:C_k_p_n_defn} and Weyl's inequality,  $2\mu_{n}(\mathbf{A}_{\rank})\geq\sigma_{x}^{2}+n\gamma$
and $\mu_{1}(\mathbf{A}_{\rank})\leq\frac{3}{2}(\sigma_{x}^{2}+n\gamma)$.
Then arguing as in the proof of Lemma \ref{lem:V_prob_bound} to bound the ratio $\mathrm{tr}(\mathbf{U}_{\leq \rank}^{\top}\mathbf{U}_{\leq \rank})/\mu_{\rank}(\mathbf{U}_{\leq \rank}^{\top}\mathbf{U}_{\leq \rank})^{2}$,
there exist absolute constants $c$, $c^{\prime}$, $c_{1}$ such
that if $c\beta_{\rank}\rank\log \rank\leq n$, then with probability
at least $1-8\exp[-c^{\prime}n/(\beta_{\rank}^{2}\rank)]-\mathbb{P}(C_{p,n}^{(\rank)})$
, 
\begin{equation}\label{eq:V_bound_finite_rank2}
V\leq c_1\sigma_{y}^{2}\frac{\mu_{1}(\mathbf{A}_{\rank})}{\mu_{n}(\mathbf{A}_{\rank})}\frac{\rank}{n}\leq c_1\sigma_{y}^{2}\frac{\rank}{n}.
\end{equation}
For the $B$ term, we argue similarly to as in Lemma \ref{lem:B_determ_bound} and \ref{lem:B_prob_bound},
 with $k=\rank$. Indeed if $\text{\ensuremath{\mathbf{A}_{\rank}}}$
is positive definite, then 
\[
B\leq\frac{\|\theta_{\leq \rank}^{*}\|_{\boldsymbol{\Lambda}_{\leq k}^{-1}}^{2}}{\mu_{n}(\mathbf{A}_{\rank}^{-1})^{2}\mu_{\rank}\left(\mathbf{U}_{\leq \rank}^{\top}\mathbf{U}_{\leq \rank}\right)^{2}},
\]
and with probability at least $1-8\exp\left(-\frac{c^{\prime}n}{\beta_{\rank}^{2}\rank}\right),$
\[
\mu_{\rank}\left(\mathbf{U}_{\leq \rank}^{\top}\mathbf{U}_{\leq \rank}\right)\geq cn.
\]
With probability at least $1-\mathbb{P}(C_{p,n}^{(k_{max})})$, we
have $\mu_{1}(\mathbf{A}_{k_{max}})\leq\frac{3}{2}(\sigma_{x}^{2}+n\gamma)$.
Therefore with probability at least $1-8\exp\left(-\frac{c^{\prime}n}{\beta_{kmax}^{2}k_{max}}\right)-\mathbb{P}(C_{p,n}^{(k_{max})})$,
\begin{equation}\label{eq:B_bound_finite_rank2}
B\leq c_{2}\left(\frac{\sigma_{x}^{2}}{n}+\gamma\right)^{2}\|\theta_{\leq k_{max}}^{*}\|_{\boldsymbol{\Lambda}_{\leq k}^{-1}}^{2}.
\end{equation}
The proof of the theorem is completed by using a union bound to combine \eqref{eq:V_bound_finite_rank2} and \eqref{eq:B_bound_finite_rank2}.
\end{proof}

\subsection{Proof of Theorem \ref{thm:R_overall_bound}}\label{sec:proof_of_R_overall_bound}

\begin{proof}
The first claim is proved by using a union bound to combine the results
of Lemmas \ref{lem:R1}-\ref{lem:R3}. For the second claim of the
proposition, note that $\mu_{n}(p^{-1}\mathbf{X}\mathbf{X}^{\top})+n\gamma=\mu_{n}(\mathbf{A})$,
so an application of Weyl's inequality gives
\[
\left|\mu_{n}(\mathbf{A})-\mu_{n}(\mathbf{A}-\boldsymbol{\Delta})\right|\leq\|\boldsymbol{\Delta}\|_{2}
\]
and since $\mu_{n}(\mathbf{A}-\boldsymbol{\Delta})=\mu_{n}(\mathbf{K})+\sigma_{x}^{2}+n\gamma$
we obtain.
\[
\mu_{n}(\mathbf{A})\geq\mu_{n}(\mathbf{K})+\sigma_{x}^{2}+n\gamma-\|\boldsymbol{\Delta}\|_{2}.
\]
Therefore 
\[
\mathbb{P}\left(2\mu_{n}(\mathbf{A})\geq\mu_{n}(\mathbf{K})+\sigma_{x}^{2}+n\gamma\right)\geq1-\mathbb{P}\left(C_{p,n}^{(0)}\right),
\]
whilst $\mu_{n}(\mathbf{K})+\sigma_{x}^{2}+n\gamma$ is strictly positive
on the complement of the event $D_{n}^{(0)}$. Using a union bound
to combine these facts with the first claim of the proposition complements
the second claim of the proposition.
\end{proof}

The following preliminary lemma will be used when bounding $S_{1}$-$S_{3}$
in the proofs of Lemmas \ref{lem:R1}-\ref{lem:R3} below.
\begin{lem}
\label{lem:CD}If $\mathbf{C}$ and $\mathbf{D}$ are any symmetric,
positive semi-definite matrices such that the product $\mathbf{C}\mathbf{D}$
is well-defined, $\mathrm{tr}(\mathbf{C}\mathbf{D})\leq\|\mathbf{C}\|_{2}\mathrm{tr}(\mathbf{D})$.
\end{lem}
\begin{proof}
We have $\mu_{1}(\mathbf{C})\mathbf{I}-\mathbf{C}\succeq0$, hence
\[
\mu_{1}(\mathbf{C})\mathrm{tr}(\mathbf{D})-\mathrm{tr}(\mathbf{C}\mathbf{D})=\mathrm{tr}([\mu_{1}(\mathbf{C})\mathbf{I}-\mathbf{C}]\mathbf{D})=\mathrm{tr}(\mathbf{D}^{1/2}[\mu_{1}(\mathbf{C})\mathbf{I}-\mathbf{C}]\mathbf{D}^{1/2})\geq0.
\]
\end{proof}

\subsubsection{Bounding \texorpdfstring{$S_{1}$}{S1}}
\begin{lem}
\label{lem:R1}For any $\delta_1\in(0,1)$, with probability at least $1-\delta_1$,
\[
S_{1}\leq\frac{1}{\delta_1}\frac{v_1n^2}{p}\frac{\left(\sup_{z}|g(z)|^{2}+\sigma_{y}^{2}/n\right)}{\left(\mu_{n}(p^{-1}\mathbf{X}\mathbf{X}^{\top})+n\gamma\right)^{2}}.
\]
\end{lem}
\begin{proof}
We have 
\begin{align*}
 & \left|\phi(z_{test})^{\top}\left(\mathbf{W}^{\top}\mathbf{W}-\mathbf{I}_{r}\right)\boldsymbol{\Phi}^{\top}\mathbf{A}^{-1}\mathbf{y}\right|^{2}\\
 & =\mathbf{y}^{\top}\mathbf{A}^{-1}\boldsymbol{\Phi}\left(\mathbf{W}^{\top}\mathbf{W}-\mathbf{I}_{r}\right)\phi(z_{test})\phi(z_{test})^{\top}\left(\mathbf{W}^{\top}\mathbf{W}-\mathbf{I}_{r}\right)\boldsymbol{\Phi}^{\top}\mathbf{A}^{-1}\mathbf{y}
\end{align*}
and $\mathbb{E}\left[\phi(z_{test})\phi(z_{test})^{\top}\right]=\boldsymbol{\Lambda}$,
hence with the shorthand $\mathbf{B}\coloneqq\boldsymbol{\Lambda}^{1/2}\left(\mathbf{W}^{\top}\mathbf{W}-\mathbf{I}_{r}\right)\boldsymbol{\Phi}^{\top}$,
\[
S_{1}=\mathbb{E}_{\boldsymbol{\epsilon}}\left[\mathbf{y}^{\top}\mathbf{A}^{-1}\mathbf{B}^{\top}\mathbf{B}\mathbf{A}^{-1}\mathbf{y}\right].
\]
The term inside the expectation may be re-written:
\[
\mathbf{y}^{\top}\mathbf{A}^{-1}\mathbf{B}^{\top}\mathbf{B}\mathbf{A}^{-1}\mathbf{y}=\mathrm{tr}\left(\mathbf{y}\mathbf{y}^{\top}\mathbf{A}^{-1}\mathbf{B}^{\top}\mathbf{B}\mathbf{A}^{-1}\right).
\]
Using $\mathbf{y}=\boldsymbol{\Phi}\theta^{*}+\boldsymbol{\epsilon}$,
the independence of $\boldsymbol{\Phi}\theta^{*}$ and $\boldsymbol{\epsilon}$,
$\mathbb{E}[\boldsymbol{\epsilon}]=\boldsymbol{0}$ and $\mathbb{E}[\boldsymbol{\epsilon}\boldsymbol{\epsilon}^{\top}]=\sigma_{y}^{2}\mathbf{I}_{n}$,
and the linearity of trace,
\begin{align*}
S_{1} & =\mathrm{tr}\left(\left[\boldsymbol{\Phi}\theta^{*}(\boldsymbol{\Phi}\theta^{*})^{\top}+\sigma_{y}^{2}\mathbf{I}_{n}\right]\mathbf{A}^{-1}\mathbf{B}^{\top}\mathbf{B}\mathbf{A}^{-1}\right)\\
 & =\mathrm{tr}\left(\boldsymbol{\Phi}\theta^{*}(\boldsymbol{\Phi}\theta^{*})^{\top}\mathbf{A}^{-1}\mathbf{B}^{\top}\mathbf{B}\mathbf{A}^{-1}\right)+\sigma_{y}^{2}\mathrm{tr}\left(\mathbf{A}^{-1}\mathbf{B}^{\top}\mathbf{B}\mathbf{A}^{-1}\right).
\end{align*}
By the cyclic property of trace and an application of Lemma \ref{lem:CD}
with $\mathbf{C}=(\mathbf{A}^{-1})^{2}$ and $\mathbf{D}=\mathbf{B}^{\top}\mathbf{B}$,
\[
\mathrm{tr}\left(\mathbf{A}^{-1}\mathbf{B}^{\top}\mathbf{B}\mathbf{A}^{-1}\right)=\mathrm{tr}\left(\mathbf{A}^{-1}\mathbf{A}^{-1}\mathbf{B}^{\top}\mathbf{B}\right)\leq\mu_{1}(\mathbf{A}^{-1})^{2}\mathrm{tr}\left(\mathbf{B}^{\top}\mathbf{B}\right)=\mu_{1}(\mathbf{A}^{-1})^{2}\left\Vert \mathbf{B}\right\Vert _{F}^{2},
\]
and similarly,
\[
\mathrm{tr}\left(\boldsymbol{\Phi}\theta^{*}(\boldsymbol{\Phi}\theta^{*})^{\top}\mathbf{A}^{-1}\mathbf{B}^{\top}\mathbf{B}\mathbf{A}^{-1}\right)\leq\|\boldsymbol{\Phi}\theta^{*}\|_{2}^{2}\mathrm{tr}\left(\mathbf{A}^{-1}\mathbf{B}^{\top}\mathbf{B}\mathbf{A}^{-1}\right)\leq\|\boldsymbol{\Phi}\theta^{*}\|_{2}^{2}\mu_{1}(\mathbf{A}^{-1})^{2}\left\Vert \mathbf{B}\right\Vert _{F}^{2}.
\]
Combining the above trace bounds with the identities: $\|\boldsymbol{\Phi}\theta^{*}\|_{2}^{2}=\sum_{i=1}^{n}|g(z_{i})|^{2}$
, $\mu_{1}(\mathbf{A}^{-1})^{-1}=\mu_{n}(\mathbf{A})=\mu_{n}(p^{-1}\mathbf{X}\mathbf{X}^{\top})+n\gamma$,
gives:
\begin{align}
S_{1} & \leq \left(\sup_{z}|g(z)|^{2}+\sigma_{y}^{2}/n\right)\frac{n\left\Vert \mathbf{B}\right\Vert _{F}^{2}}{\mu_{n}(\mathbf{A})^{2}}\nonumber \\
 & =\left(\sup_{z}|g(z)|^{2}+\sigma_{y}^{2}/n\right)\frac{n\left\Vert \mathbf{B}\right\Vert _{F}^{2}}{\left(\mu_{n}(p^{-1}\mathbf{X}\mathbf{X}^{\top})+n\gamma\right)^{2}}.\label{eq:S_1_trace_bound}
\end{align}
In order to write out the Frobenius norm term in (\ref{eq:S_1_trace_bound})
more explicitly, recall from Proposition \ref{prop:KL_expansion}
that $\mathbb{E}[\mathbf{W}^{\top}\mathbf{W}]=\mathbf{I}_{r}$, so
that with $\mathbf{W}\equiv[W_{1}|\cdots|W_{p}]^{\top}$, we have:
\[
\mathbf{W}^{\top}\mathbf{W}-\mathbf{I}_{r}=\sum_{j=1}^{p}W_{j}W_{j}^{\top}-\mathbb{E}\left[W_{j}W_{j}^{\top}\right]
\]
 and with 
\[
\xi_{j}(z_{i})\coloneqq\boldsymbol{\Lambda}^{1/2}W_{j}W_{j}^{\top}\phi(z_{i})-\boldsymbol{\Lambda}^{1/2}\mathbb{E}\left[W_{j}W_{j}^{\top}\right]\phi(z_{i}),
\]
we have
\[
\left\Vert \mathbf{B}\right\Vert _{F}^{2}=\sum_{i=1}^{n}\left\Vert \boldsymbol{\Lambda}^{1/2}\left(\mathbf{W}^{\top}\mathbf{W}-\mathbf{I}_{r}\right)\phi(z_{i})\right\Vert _{2}^{2}=\sum_{i=1}^{n}\left\Vert \sum_{j=1}^{p}\xi_{j}(z_{i})\right\Vert _{2}^{2}.
\]

By Markov's inequality, for any $\delta>0$, 
\begin{align}
\mathbb{P}\left(\left\Vert \mathbf{B}\right\Vert _{F}^{2}\geq\delta\right) & \leq\frac{1}{\delta}\mathbb{E}\left[\left\Vert \mathbf{B}\right\Vert _{F}^{2}\right]\nonumber \\
 & =\frac{n}{\delta}\mathbb{E}\left[\mathbb{E}\left[\left.\left\Vert \sum_{j=1}^{p}\xi_{j}(z_{1})\right\Vert _{2}^{2}\right|z_{1}\right]\right]\nonumber \\
 & =\frac{n}{\delta}\sum_{j=1}^{p}\mathbb{E}\left[\left\Vert \xi_{j}(z_{1})\right\Vert _{2}^{2}\right],\label{eq:S_1_Markov}
\end{align}
where the first equality uses the fact that $z_{1},\ldots,z_{n}$
are identically distributed, and the third equality uses the fact
that $\mathbb{E}\left[\left.\xi_{j}(z_{1})\right|z_{1}\right]=\mathbf{0}$
and that by Assumption \ref{ass:ind} the $\xi_{j}(z_{1})$ are conditionally
independent given $z_{1}$.

Furthermore, using from Proposition \ref{prop:KL_expansion} $W_{j}^{\top}\phi(z_{1})=p^{-1/2}\psi_{j}(z_{1})$
and the definition of $\mathbf{W}$,
\begin{align*}
\left\Vert \xi_{j}(z_{1})\right\Vert _{2}^{2} & =\sum_{k}\lambda_{k}\left(\mathbf{W}_{jk}W_{j}^{\top}\phi(z_{1})-\mathbb{E}\left[\mathbf{W}_{jk}W_{j}^{\top}\right]\phi(z_{1})\right)^{2}\\
 & =\frac{1}{p^{2}}\sum_{k}\left(\int_{\mathcal{Z}}\psi_{j}(z)u_{k}(z)\mu(\mathrm{d}z)\psi_{j}(z_{1})-\mathbb{E}\left[\left.\int_{\mathcal{Z}}\psi_{j}(z)u_{k}(z)\mu(\mathrm{d}z)\psi_{j}(z_{1})\right|z_{1}\right]\right)^{2}\\
 & =\frac{1}{p^{2}}\sum_{k}\left(\int_{\mathcal{Z}}\psi_{j}(z)\psi_{j}(z_{1})-\mathbb{E}\left[\left.\psi_{j}(z)\psi_{j}(z_{1})\right|z_{1}\right]u_{k}(z)\mu(\mathrm{d}z)\right)^{2}\\
 & \leq\frac{1}{p^{2}}\left\Vert \psi_{j}(\cdot)\psi_{j}(z_{1})-\mathbb{E}\left[\left.\psi_{j}(\cdot)\psi_{j}(z_{1})\right|z_{1}\right]\right\Vert _{L_{2}(\mu)}^{2}\\
 & =\frac{1}{p^{2}}\int_{\mathcal{Z}}\left|\psi_{j}(z)\psi_{j}(z_{1})-\mathbb{E}\left[\left.\psi_{j}(z)\psi_{j}(z_{1})\right|z_{1}\right]\right|^{2}\mu(\mathrm{d}z),
\end{align*}
where the inequality is Bessel's inequality. Taking expectation and
using the independence of $z_{!}$ and $\psi_{j}$, 
\begin{align*}
\mathbb{E}\left[\left\Vert \xi_{j}(z_{1})\right\Vert _{2}^{2}\right] & \leq\frac{1}{p^{2}}\mathbb{E}\left[\int_{\mathcal{Z}}\int_{\mathcal{Z}}\left|\psi_{j}(z)\psi_{j}(z_{1})-\mathbb{E}\left[\psi_{j}(z)\psi_{j}(z^{\prime})\right]\right|^{2}\mu(\mathrm{d}z)\mu(\mathrm{d}z^{\prime})\right]\\
 & =\frac{1}{p^{2}}\int_{\mathcal{Z}}\int_{\mathcal{Z}}\mathbb{E}\left[\left|\psi_{j}(z)\psi_{j}(z^{\prime})-\mathbb{E}\left[\psi_{j}(z)\psi_{j}(z^{\prime})\right]\right|^{2}\right]\mu(\mathrm{d}z)\mu(\mathrm{d}z^{\prime}).
\end{align*}
Returning to (\ref{eq:S_1_Markov}) and recalling the definition of
$v_{1}$ in (\ref{eq:v1_defn}), we have shown:
\begin{align*}
\mathbb{P}\left(\left\Vert \mathbf{B}\right\Vert _{F}^{2}\geq\delta\right) & \leq\frac{n}{\delta p}\frac{1}{p}\sum_{j=1}^{p}\int_{\mathcal{Z}}\int_{\mathcal{Z}}\mathbb{E}\left[\left|\psi_{j}(z)\psi_{j}(z^{\prime})-\mathbb{E}\left[\psi_{j}(z)\psi_{j}(z^{\prime})\right]\right|^{2}\right]\mu(\mathrm{d}z)\mu(\mathrm{d}z^{\prime})\\
 & \leq\frac{nv_{1}}{\delta p},
\end{align*}
and combined with (\ref{eq:S_1_trace_bound}), we have shown that,
with probability at least $1-\frac{v_{1}n}{\delta p}$, 
\begin{equation}
S_{1}\leq\delta n\cdot\frac{\sup_{z}|g(z)|^{2}+\sigma_{y}^{2}/n}{\left(\mu_{n}(p^{-1}\mathbf{X}\mathbf{X}^{\top})+n\gamma\right)^{2}}.\label{eq:S_1_proof_final}
\end{equation}
The proof is completed by choosing any $\delta_1\in(0,1)$ and setting $\delta=nv_1/\delta_1 p$.
\end{proof}

\subsubsection{Bounding \texorpdfstring{$S_{2}$}{S2}}
\begin{lem}
\label{lem:R2}For any $\delta_2\in(0,1)$, with probability at least $1-\delta_2$,
\[
S_{2}\leq\frac{1}{\delta_2}\frac{\sigma_{x}^{2}v_2n^2}{p}\frac{\left(\sup_{z}|g(z)|^{2}+\sigma_{y}^{2}\right)}{\left(\mu_{n}(p^{-1}\mathbf{X}\mathbf{X}^{\top})+n\gamma\right)^{2}}.
\]
\end{lem}
\begin{proof}
We have
\begin{align}
 & \mathbb{E}_{z_{test},\boldsymbol{\epsilon}}\left[\left|\phi(z_{test})^{\top}\mathbf{W}^{\top}\mathbf{E}^{\top}\mathbf{A}^{-1}\mathbf{y}\right|^{2}\right]\nonumber \\
 & =\mathbb{E}_{\boldsymbol{\epsilon}}\left[\mathbf{y}^{\top}\mathbf{A}^{-1}\mathbf{E}\mathbf{W}\boldsymbol{\Lambda}\mathbf{W}^{\top}\mathbf{E}^{\top}\mathbf{A}^{-1}\mathbf{y}\right]\\
 & =\mathbb{E}_{\boldsymbol{\epsilon}}\left[\left\Vert \boldsymbol{\Lambda}^{1/2}\mathbf{W}^{\top}\mathbf{E}^{\top}\mathbf{A}^{-1}\mathbf{y}\right\Vert _{2}^{2}\right]\nonumber \\
 & =\left\Vert \boldsymbol{\Lambda}^{1/2}\mathbf{W}^{\top}\mathbf{E}^{\top}\mathbf{A}^{-1}\boldsymbol{\Phi}\theta^{*}\right\Vert _{2}^{2}+\mathbb{E}_{\boldsymbol{\epsilon}}\left[\left\Vert \boldsymbol{\Lambda}^{1/2}\mathbf{W}^{\top}\mathbf{E}^{\top}\mathbf{A}^{-1}\boldsymbol{\epsilon}\right\Vert ^{2}\right]\nonumber \\
 & =\left\Vert \boldsymbol{\Lambda}^{1/2}\mathbf{W}^{\top}\mathbf{E}^{\top}\mathbf{A}^{-1}\boldsymbol{\Phi}\theta^{*}\right\Vert _{2}^{2}+\mathbb{E}_{\boldsymbol{\epsilon}}\left[\mathrm{tr}\left(\boldsymbol{\epsilon}\boldsymbol{\epsilon}^{\top}\mathbf{A}^{-1}\mathbf{E}\mathbf{W}\boldsymbol{\Lambda}\mathbf{W}^{\top}\mathbf{E}^{\top}\mathbf{A}^{-1}\right)\right]\nonumber \\
 & =\left\Vert \boldsymbol{\Lambda}^{1/2}\mathbf{W}^{\top}\mathbf{E}^{\top}\mathbf{A}^{-1}\boldsymbol{\Phi}\theta^{*}\right\Vert _{2}^{2}+\mathrm{tr}\left(\mathbb{E}\left[\boldsymbol{\epsilon}\boldsymbol{\epsilon}^{\top}\right]\mathbf{A}^{-1}\mathbf{E}\mathbf{W}\boldsymbol{\Lambda}\mathbf{W}^{\top}\mathbf{E}^{\top}\mathbf{A}^{-1}\right)\nonumber \\
 & =\left\Vert \boldsymbol{\Lambda}^{1/2}\mathbf{W}^{\top}\mathbf{E}^{\top}\mathbf{A}^{-1}\boldsymbol{\Phi}\theta^{*}\right\Vert _{2}^{2}+\sigma_{y}^{2}\left\Vert \boldsymbol{\Lambda}^{1/2}\mathbf{W}^{\top}\mathbf{E}^{\top}\mathbf{A}^{-1}\right\Vert _{F}^{2}\nonumber \\
 & \leq\left\Vert \boldsymbol{\Lambda}^{1/2}\mathbf{W}^{\top}\mathbf{E}^{\top}\mathbf{A}^{-1}\right\Vert _{F}^{2}\left(\|\boldsymbol{\Phi}\theta^{*}\|_{2}^{2}+\sigma_{y}^{2}\right),\label{eq:S_2_proof_bound}
\end{align}
where the first inequality uses the fact that $\boldsymbol{\epsilon}$
and $z_{test}$ are independent of each other and all other random
variables; the third equality uses independence and $\mathbb{E}[\boldsymbol{\epsilon}]=\mathbf{0}$;
the fourth equality uses the cyclic property of trace; the fifth equality
uses independence and linearity of trace; the sixth equality uses
$\mathbb{E}\left[\boldsymbol{\epsilon}\boldsymbol{\epsilon}^{\top}\right]=\sigma_{y}^{2}\mathbf{I}_{n}$.

Now using the cyclic property of trace and Lemma \ref{lem:CD},

\begin{multline}
\left\Vert \boldsymbol{\Lambda}^{1/2}\mathbf{W}^{\top}\mathbf{E}^{\top}\mathbf{A}^{-1}\right\Vert _{F}^{2}=\mathrm{tr}\left(\mathbf{A}^{-1}\mathbf{E}\mathbf{W}\boldsymbol{\Lambda}\mathbf{W}^{\top}\mathbf{E}^{\top}\mathbf{A}^{-1}\right)=\mathrm{tr}\left(\mathbf{A}^{-2}\mathbf{E}\mathbf{W}\boldsymbol{\Lambda}\mathbf{W}^{\top}\mathbf{E}^{\top}\right)\\
\leq\|\mathbf{A}^{-1}\|_{2}^{2}\mathrm{tr}\left(\mathbf{A}^{-2}\mathbf{E}\mathbf{W}\boldsymbol{\Lambda}\mathbf{W}^{\top}\mathbf{E}^{\top}\right)=\|\mathbf{A}^{-1}\|_{2}^{2}\|\boldsymbol{\Lambda}^{1/2}\mathbf{W}^{\top}\mathbf{E}^{\top}\|_{F}^{2}=\|\mathbf{A}^{-1}\|_{2}^{2}\sum_{i=1}^{n}\|\boldsymbol{\Lambda}^{1/2}\mathbf{W}^{\top}\mathbf{e}_{i}\|_{2}^{2},\label{eq:S_2_proof_bound_2}
\end{multline}
and then using independence, $\mathbb{E}\left[\mathbf{e}_{i}\mathbf{e}_{i}^{\top}\right]=\mathbf{I}_{p}$,
the cyclic property and linearity of trace, and the identity $\mathbb{E}\left[\mathbf{W}^{\top}\mathbf{W}\right]=\mathbf{I}_{r}$
from Proposition \ref{prop:KL_expansion},

\begin{multline*}
\mathbb{E}\left[\|\boldsymbol{\Lambda}^{1/2}\mathbf{W}^{\top}\mathbf{e}_{i}\|_{2}^{2}\right]=\mathbb{E}\left[\mathbb{E}\left[\left.\mathrm{tr}\left(\mathbf{e}_{i}\mathbf{e}_{i}^{\top}\mathbf{W}\boldsymbol{\Lambda}\mathbf{W}^{\top}\right)\right|\mathbf{W}\right]\right]=\mathbb{E}\left[\mathrm{tr}\left(\mathbb{E}\left[\mathbf{e}_{i}\mathbf{e}_{i}^{\top}\right]\mathbf{W}\boldsymbol{\Lambda}\mathbf{W}^{\top}\right)\right]\\
=\mathbb{E}\left[\mathrm{tr}\left(\mathbf{W}\boldsymbol{\Lambda}\mathbf{W}^{\top}\right)\right]=\mathrm{tr}\left(\mathbb{E}\left[\mathbf{W}^{\top}\mathbf{W}\right]\boldsymbol{\Lambda}\right)=\mathrm{tr}\left(\boldsymbol{\Lambda}\right)=\sum_{k=1}^\infty\lambda_k \int_{\mathcal{Z}}|u_k(z)|^2\mu(\mathrm{d}z)\leq v_2.
\end{multline*}
Therefore by Markov's inequality, for any $\delta>0$,
\[
\mathbb{P}\left(\frac{\sigma_{x}^{2}}{p}\sum_{i=1}^{n}\|\boldsymbol{\Lambda}^{1/2}\mathbf{W}^{\top}\mathbf{e}_{i}\|_{2}^{2}\geq\delta\right)\leq\frac{\sigma_{x}^{2}}{\delta p}\mathbb{E}\left[\sum_{i=1}^{n}\|\boldsymbol{\Lambda}^{1/2}\mathbf{W}^{\top}\mathbf{e}_{i}\|_{2}^{2}\right]\leq\frac{n}{p}\frac{\sigma_{x}^{2}v_2}{\delta},
\]
and returning to (\ref{eq:S_2_proof_bound}) and using $\|\boldsymbol{\Phi}\theta^{*}\|_{2}^{2}+\sigma_{y}^{2}\leq n\sup_{z}|g(z)|^{2}+\sigma_{y}^{2}$,
we have shown that with probability at least $1-\frac{\sigma_{x}^{2}v_2}{\delta}\frac{n}{p}$,
\[
S_{2}=\frac{\sigma_{x}^{2}}{p}\mathbb{E}_{z_{test},\boldsymbol{\epsilon}}\left[\left|\phi(z_{test})^{\top}\mathbf{W}^{\top}\mathbf{E}^{\top}\mathbf{A}^{-1}\mathbf{y}\right|^{2}\right]\leq\delta n \cdot\frac{\left(\sup_{z}|g(z)|^{2}+\sigma_{y}^{2}/n\right)}{\left(\mu_{n}(p^{-1}\mathbf{X}\mathbf{X}^{\top})+n\gamma\right)^{2}}.
\]
The proof is completed by choosing any $\delta_2\in(0,1)$ and setting $\delta=\sigma_x^2 v_2 n/(p\delta_2)$.
\end{proof}

\subsubsection{Bounding \texorpdfstring{$S_{3}$}{S3}}
\begin{lem}
\label{lem:R3} For any $\delta_3\in(0,1)$, with probability at least $1-\delta_3$,
\[
S_{3}\leq\frac{1}{\delta_3}\frac{\sigma_{x}^{2}(v_{2}+\sigma_{x}^{2})n^2}{p}\frac{\left(\sup_{z}\left|g(z)\right|^{2}+\sigma_{y}^{2}/n\right)}{\left(\mu_{n}(p^{-1}\mathbf{X}\mathbf{X}^{\top})+n\gamma\right)^{2}}.
\]
\end{lem}
\begin{proof}
We have
\begin{align*}
\mathbb{E}_{\mathbf{e}_{test},\boldsymbol{\epsilon}}\left[\left|\mathbf{e}_{test}^{\top}\mathbf{X}^{\top}\mathbf{A}^{-1}\mathbf{y}\right|^{2}\right] & =\mathbb{E}_{\mathbf{e}_{test},\boldsymbol{\epsilon}}\left[\mathbf{y}^{\top}\mathbf{A}^{-1}\mathbf{X}\mathbf{e}_{test}\mathbf{e}_{test}^{\top}\mathbf{X}^{\top}\mathbf{A}^{-1}\mathbf{y}\right]\\
 & =\mathbb{E}_{\boldsymbol{\epsilon}}\left[\mathbf{y}^{\top}\mathbf{A}^{-1}\mathbf{X}\mathbb{E}\left[\mathbf{e}_{test}\mathbf{e}_{test}^{\top}\right]\mathbf{X}^{\top}\mathbf{A}^{-1}\mathbf{y}\right]\\
 & =\mathbb{E}_{\boldsymbol{\epsilon}}\left[\left\Vert \mathbf{X}^{\top}\mathbf{A}^{-1}\mathbf{y}\right\Vert _{2}^{2}\right]\\
 & =\left\Vert \mathbf{X}^{\top}\mathbf{A}^{-1}\boldsymbol{\Phi}\theta^{*}\right\Vert _{2}^{2}+\mathbb{E}_{\boldsymbol{\epsilon}}\left[\left\Vert \mathbf{X}^{\top}\mathbf{A}^{-1}\boldsymbol{\epsilon}\right\Vert _{2}^{2}\right]\\
 & =\left\Vert \mathbf{X}^{\top}\mathbf{A}^{-1}\boldsymbol{\Phi}\theta^{*}\right\Vert _{2}^{2}+\mathbb{E}_{\boldsymbol{\epsilon}}\left[\mathrm{tr}\left(\boldsymbol{\epsilon}\boldsymbol{\epsilon}^{\top}\mathbf{A}^{-1}\mathbf{X}\mathbf{X}^{\top}\mathbf{A}^{-1}\right)\right]\\
 & =\left\Vert \mathbf{X}^{\top}\mathbf{A}^{-1}\boldsymbol{\Phi}\theta^{*}\right\Vert _{2}^{2}+\mathrm{tr}\left(\mathbb{E}[\boldsymbol{\epsilon}\boldsymbol{\epsilon}^{\top}]\mathbf{A}^{-1}\mathbf{X}\mathbf{X}^{\top}\mathbf{A}^{-1}\right)\\
 & =\left\Vert \mathbf{X}^{\top}\mathbf{A}^{-1}\boldsymbol{\Phi}\theta^{*}\right\Vert _{2}^{2}+\sigma_{y}^{2}\left\Vert \mathbf{X}^{\top}\mathbf{A}^{-1}\right\Vert _{F}^{2}\\
 & \leq\left\Vert \mathbf{X}^{\top}\mathbf{A}^{-1}\right\Vert _{F}^{2}\left(n\sup_{z}\left|g(z)\right|^{2}+\sigma_{y}^{2}\right),
\end{align*}
where the second equality holds because $\mathbf{e}_{test}$ is independent
of all other random variables; the third holds because $\mathbb{E}\left[\mathbf{e}_{test}\mathbf{e}_{test}^{\top}\right]=\mathbf{I}_{p}$;
the fourth uses $\mathbf{y}=\boldsymbol{\Phi}\theta^{*}+\boldsymbol{\epsilon}$,
the independence between $\boldsymbol{\epsilon}$ and all other random
variables, and $\mathbb{E}[\boldsymbol{\epsilon}]=0$; the fifth holds
by the cyclic property of trace the sixth uses independence; the seventh
uses $\mathbb{E}\left[\boldsymbol{\epsilon}\boldsymbol{\epsilon}^{\top}\right]=\sigma_{y}^{2}\mathbf{I}_{n}$;
the final inequality uses $\left\Vert \mathbf{X}^{\top}\mathbf{A}^{-1}\boldsymbol{\Phi}\theta^{*}\right\Vert _{2}^{2}\leq\left\Vert \mathbf{X}^{\top}\mathbf{A}^{-1}\right\Vert _{2}^{2}\left\Vert \boldsymbol{\Phi}\theta^{*}\right\Vert _{2}^{2}\leq\left\Vert \mathbf{X}^{\top}\mathbf{A}^{-1}\right\Vert _{F}^{2}n\sup_{z}\left|g(z)\right|^{2}$.
We have therefore established that
\begin{equation}
S_{3}=\frac{\sigma_{x}^{2}}{p^{2}}\mathbb{E}_{\mathbf{e}_{test},\boldsymbol{\epsilon}}\left[\left|\mathbf{e}_{test}^{\top}\mathbf{X}^{\top}\mathbf{A}^{-1}\mathbf{y}\right|^{2}\right]\leq\frac{\sigma_{x}^{2}}{p^{2}}\left\Vert \mathbf{X}^{\top}\mathbf{A}^{-1}\right\Vert _{F}^{2}\left(n\sup_{z}\left|g(z)\right|^{2}+\sigma_{y}^{2}\right)\label{eq:R_3_expectation_bound}
\end{equation}

Now using the cyclic property of trace and Lemma \ref{lem:CD},
\begin{equation}
\left\Vert \mathbf{X}^{\top}\mathbf{A}^{-1}\right\Vert _{F}^{2}=\mathrm{tr}\left(\mathbf{A}^{-1}\mathbf{X}\mathbf{X}^{\top}\mathbf{A}^{-1}\right)=\mathrm{tr}\left(\mathbf{A}^{-2}\mathbf{X}\mathbf{X}^{\top}\right)\leq\left\Vert \mathbf{\mathbf{A}}^{-1}\right\Vert _{2}^{2}\mathrm{tr}(\mathbf{\mathbf{X}}\mathbf{X}^{\top})=\frac{\left\Vert \mathbf{X}\right\Vert _{F}^{2}}{\mu_{n}(\mathbf{A})^{2}},\label{eq:X^TA_frob_bound}
\end{equation}
and substituting into (\ref{eq:R_3_expectation_bound}) gives:
\begin{equation}
S_{3}\leq\sigma_{x}^{2}\| \mathbf{X}\|_{F}^{2}\cdot\frac{n}{p^{2}}\frac{\left(\sup_{z}\left|g(z)\right|^{2}+\sigma_{y}^{2}/n\right)}{\mu_{n}(\mathbf{A})^{2}}.\label{eq:R_3_expectation_bound_2}
\end{equation}
Using independence, and the properties $\mathbb{E}[\mathbf{E}_{ij}]=0$
and $\mathbb{E}\left[\left|\mathbf{E}_{ij}\right|^{2}\right]=1$,
\begin{align*}
\mathbb{E}\left[\left\Vert \mathbf{X}\right\Vert _{F}^{2}\right] & =\sum_{j=1}^{p}\sum_{i=1}^{n}\mathbb{E}\left[\left|\psi_{j}(z_{i})+\sigma_{x}\mathbf{E}_{ij}\right|^{2}\right]\\
 & =\sum_{j=1}^{p}n\mathbb{E}\left[\int\left|\psi_{j}(z)\right|^{2}\mu(\mathrm{d}z)\right]+\sum_{i=1}^{n}\sum_{j=1}^{p}\sigma_{x}^{2}\mathbb{E}\left[\left|\mathbf{E}_{ij}\right|^{2}\right]\\
 & =n\sum_{j=1}^{p}\int\mathbb{E}\left[\left|\psi_{j}(z)\right|^{2}\right]\mu(\mathrm{d}z)+np\sigma_{x}^{2}\leq np(v_{2}+\sigma_{x}^{2}),
\end{align*}
where $v_{2}$ is defined in (\ref{eq:v2_defn}), hence by Markov's
inequality
\[
\mathbb{P}\left(\frac{\sigma_{x}^{2}}{p^{2}}\left\Vert \mathbf{X}\right\Vert _{F}^{2}\geq\delta\right)\leq\frac{n}{p}\frac{\sigma_{x}^{2}(v_{2}+\sigma_{x}^{2})}{\delta}.
\]
Applying this estimate to
(\ref{eq:R_3_expectation_bound_2}), choosing any $\delta_3\in(0,1)$ and setting $\delta=n\sigma_x^2(v_2+\sigma_x^2)/(\delta_3 p)$ completes the proof.
\end{proof}

\subsection{Proof of Proposition \ref{prop:prob_of_C_unconditional}}\label{sec:proof_of_C_unconditional}
\begin{prop}
\label{prop:event_C} For any $0\leq k\leq n$,
\[
\mathbb{P}\left(\left.C_{p,n}^{(k)}\right|z_{1},\ldots,z_{n}\right)\leq\frac{24n^2}{p}\frac{v_1+8\sigma_x^2v_2+2\sigma_x^4v_3}{\left(\mu_{n}(\mathbf{K}_{>k})+\sigma_{x}^{2}+n\gamma\right)^{2}},\qquad a.s.,
\]
with the convention that $\mathbf{K}_{>0}\equiv\mathbf{K}$.
\end{prop}

\begin{proof}
Write $\mbs{\Delta} = \mbs{\Delta}^{(1)} + p^{-1/2}\sigma_x\mbs{\Delta}^{(2)} + p^{-1}\sigma_x^2\mbs{\Delta}^{(3)}$, where
\begin{itemize}
\item $\mbs{\Delta}^{(1)} \coloneqq \mbs{\Phi}(\m{W}^\top\m{W} - \m{I}_r)\mbs{\Phi}^\top$
\item $\mbs{\Delta}^{(2)} \coloneqq \mbs{\Phi}\m{W}^\top\m{E}^\top + \m{E}\m{W}\mbs{\Phi}^\top$
\item $\mbs{\Delta}^{(3)} \coloneqq \m{E}\m{E}^\top - p\,\m{I}_n$.
\end{itemize}
Conditioning on the latent variables, Chebyshev's inequality (\cite{tropp2016inequalities}, Proposition 3.1) shows that for any $t > 0$
\begin{equation}\mb{P}\bigl(\|\mbs{\Delta}\|_2 \geq t\,|\,\m{z}\bigr) \leq \frac{1}{t^2}\,\mb{E}\bigl[\|\mbs{\Delta}\|^2_{S_2}\,|\,\m{z}\bigr]\end{equation} where $\|\cdot\|_{S_2}$ denotes the Schatten 2-norm, and for brevity we write $\m{z}$ to denote the latent variables $z_1, \ldots, z_n$.  Applying the triangle inequality, the generalized mean inequality and the fact that $\m{E}$ is independent of $\m{z}$, we find that
\begin{equation}\mb{P}(\|\Delta\|_2 \geq t\,|\,\m{z}) \leq \frac{3}{t^2}\Bigl(\mb{E}\bigl[\|\mbs{\Delta}^{(1)}\|_{S_2}^2\,|\,\m{z}\bigr] + p^{-1}\sigma_x^2\,\mb{E}\bigl[\|\mbs{\Delta}^{(2)}\|_{S_2}^2\,|\,\m{z}\bigr] + p^{-2}\sigma_x^4\,\mb{E}\bigl[\|\mbs{\Delta}^{(3)}\|_{S_2}^2\bigr]\Bigr).\end{equation}
To bound the first term, observe that \begin{equation}\Delta^{(1)} = \frac{1}{p}\sum_{j=1}^p \Bigl(\mbs{\psi}_j\mbs{\psi}_j^\top - \mb{E}\bigl[\mbs{\psi}_j\mbs{\psi}_j^\top\,|\,\m{z}\bigr]\Bigr)\end{equation} where $\mbs{\psi}_j \equiv \mbs{\psi}_j(\m{z}) \equiv [\psi_j(z_1)|\cdots|\psi_j(z_n)]^\top$.  By our assumptions, the vectors $\mbs{\psi}_j(\cdot)$ are mutually conditionally independent given $\m{z}$, and so by the matrix Efron-Stein inequality (\cite{tropp2016inequalities}, Theorem 4.2):
\begin{equation}\mb{E}\bigl[\|\mbs{\Delta}^{(1)}\|_{S_2}^2\,|\,\m{z}\bigr] \leq 2\,\mb{E}\bigl[\|\mbs{\Sigma}^{(1)}\|_{S_1}\,|\,\m{z}\bigr]\end{equation} where $\mbs{\Sigma}^{(1)}$ is the variance proxy
\begin{equation}
  \mbs{\Sigma}^{(1)} = \frac{1}{2p^2}\sum_{j=1}^p \mb{E}\Bigl[\bigl(\mbs{\psi}_j\mbs{\psi}_j^\top - \til{\mbs{\psi}}_j\til{\mbs{\psi}}_j^\top\bigr)^2 \,|\, \mbs{\psi}_j,\m{z}\Bigr]\end{equation}
(here the vector $\til{\mbs{\psi}}_j = [\til{\psi}_j(z_1) | \cdots | \til{\psi}_j(z_n) ]^T$, where $\til{\psi}_j(\cdot)$ is an independent copy of $\psi_j(\cdot)$).
Now, an identical argument to the proof of Lemma 9 in \cite{whiteley2022statistical} shows that
\begin{align}\mb{E}\bigl[\|\mbs{\Sigma}^{(1)}\|_{S_1}\,|\,\m{z}\bigr] &\leq \frac{1}{2p^2}\sum_{j=1}^p \mb{E}\bigl[\|\mbs{\psi}_j\mbs{\psi}_j^\top - \til{\mbs{\psi}}_j\til{\mbs{\psi}}_j^\top\|_{S_2}^2\,|\,\m{z}\bigl]\\
 &= \frac{1}{2p^2}\sum_{j=1}^p\sum_{i=1}^n\sum_{k=1}^n \mb{E}\bigl[|\psi_j(z_i)\psi_j(z_k) - \til{\psi}_j(z_i)\til{\psi}_j(z_k)|^2 \bigr] \\
 &= \frac{1}{p^2}\sum_{j=1}^p\sum_{i=1}^n\sum_{k=1}^n \mr{Var}\bigl[\psi_j(z_i)\psi_j(z_k)] \\
   &\leq \frac{n^2}{p} \sup_{z,z' \in \mc{Z}}\frac{1}{p}\sum_{j=1}^p\mr{Var}\bigl[\psi_j(z)\psi_j(z')]\\
   &= \frac{n^2v_1}{p}
\end{align}
where the first equality uses the fact that the Schatten 2-norm is equal to the Frobenius norm.

For the second term, observe first that
\begin{equation}\|\mbs{\Delta}^{(2)}\|_{S_2}^2 \leq 4\,\|\mbs{\Phi}\m{W}^\top\m{E}^\top\|_{S_2}^2\end{equation} by the triangle inequality and the fact that the Frobenius norm is invariant under transposes.  Letting $\m{E}_j$ denote the $j$th column of $\m{E}$, we see that
\begin{equation}\mbs{\Phi}\m{W}^\top\m{E}^\top = \frac{1}{p^{1/2}}\sum_{j=1}^p \mbs{\psi}_j\m{E}_j^\top\end{equation}
and so applying the matrix Efron-Stein inequality again we find that
\begin{equation}\mb{E}\bigl[\|\mbs{\Delta}^{(2)}\|_{S_2}^2\,|\,\m{z}\bigr] \leq 8\,\mb{E}\bigl[\|\mbs{\Sigma}^{(2)}\|_{S_1}\,|\,\m{z}\bigr]\end{equation} where $\mbs{\Sigma}^{(2)}$ is the variance proxy
\begin{equation}
  \mbs{\Sigma}^{(2)} = \frac{1}{2p}\sum_{j=1}^p \mb{E}\Bigl[\bigl(\mbs{\psi}_j\m{E}_j^\top - \til{\mbs{\psi}}_j\til{\m{E}}_j^\top\bigr)^2 \,|\, \mbs{\psi}_j,\m{E}_j,\m{z}\Bigr]\end{equation}
where $\til{\mbs{\psi}}_j$ is defined as before, and $\til{\m{E}}_j = [\til{\m{E}}_{1j}|\cdots|\til{\m{E}}_{nj}]^\top$ where each $\til{\m{E}}_{ij}$ is an identical copy of $\m{E}_{ij}$.  Note that, conditional on $\m{z}$, each matrix $\mbs{\psi}_j\m{E}_j^\top$ has zero mean.
Applying again an analogous argument to the proof of Lemma 9 in \cite{whiteley2022statistical} we see that
\begin{align}
  \mb{E}\bigl[\|\mbs{\Sigma}^{(2)}\|_{S_1}\,|\,\m{z}\bigr] &\leq \frac{1}{2p}\sum_{j=1}^p \mb{E}\bigl[\|\mbs{\psi}_j\m{E}_j^\top - \til{\mbs{\psi}}_j\til{\m{E}}_j^\top\|_{S_2}^2\,|\,\m{z}\bigl]\\
    & \leq \frac{2}{p}\sum_{j=1}^p \mb{E}\bigl[\|\mbs{\psi}_j\m{E}_j^\top\|_{S_2}^2 \,|\,\m{z}\bigr]\\
    &\leq \frac{2}{p}\sum_{j=1}^p\mb{E}\bigl[\|\mbs{\psi}_j\|^2\,|\,\m{z}\bigr]\mb{E}\bigl[\|\m{E}_j\|^2\bigr]\\
    &= \frac{2n}{p}\sum_{i=1}^n\sum_{j=1}^p\mb{E}\bigl[|\psi_j(z_i)|^2\bigr]\\
    &\leq 2n^2\sup_{z \in \mc{Z}} \mb{E}\bigl[|\psi_j(z)|^2\bigr]\\
    &= 2n^2v_2
\end{align}
recalling that the entries $\m{E}_{ij}$ have zero mean and unit variance, and thus $ \mb{E}\bigl[|\m{E}_{ij}|^2\bigr] = 1$ for all $i$ and $j$.

For the third and final term, observe that
\begin{equation}\mbs{\Delta}^{(3)} = \sum_{j=1}^p \Bigl(\m{E}_j\m{E}_j^\top - \mb{E}\bigl[ \m{E}_j\m{E}_j^\top\bigr]\Bigr)\end{equation}
and thus (applying identical arguments to before)
\begin{align}
  \mb{E}\bigl[\|\mbs{\Delta}^{(3)}\|_{S_2}^2\bigr] &\leq 4 \sum_{j=1}^p \mb{E}\bigl[\|\m{E}_j\m{E}_j^\top\|_{S_2}^2\bigr] \\
  &\leq 4n^2p \sup_{i,j\geq 1} \mb{E}\bigl[|\m{E}_{ij}|^4]\\
  &= 4n^2pv_3.
\end{align}

Combining all three results, we find that
\begin{equation}
  \mb{P}\bigl(\|\mbs{\Delta}\|_2 \geq t \,|\,\m{z}\bigr) \leq  \frac{6n^2}{pt^2}(v_1 + 8\sigma_x^2 v_2 + 2\sigma_x^4 v_3)
\end{equation}
from which the result follows.
\end{proof}

\begin{proof}
[Proof of Proposition \ref{prop:prob_of_C_unconditional}]
\begin{align*}
1-\mathbb{P}\left(C_{p,n}^{(k)}\right) & =\mathbb{P}\left(2\|\boldsymbol{\Delta}\|_{2}<\mu_{n}(\mathbf{K}_{>k})+\sigma_{x}^{2}+n\gamma\right)\\
 & \geq\mathbb{P}\left(\left\{ 2\|\boldsymbol{\Delta}\|_{2}<\mu_{n}(\mathbf{K}_{>k})+\sigma_{x}^{2}+n\gamma\right\} \cap\left\{ \mu_{n}(\mathbf{K}_{>k})\geq\phi_{k}(n)\right\} \right)\\
 & =\mathbb{E}\left[\mathbb{P}\left(\left.2\|\boldsymbol{\Delta}\|_{2}<\mu_{n}(\mathbf{K}_{>k})+\sigma_{x}^{2}+n\gamma\,\right|\,z_{1},\ldots,z_{n}\right)\mathbb{I}\left\{\mu_{n}(\mathbf{K}_{>k})\geq\phi_{k}(n)\right\} \right]\\
 & \geq\mathbb{E}\left[\left(1-\frac{24 n^{2}}{p}\frac{(v_1+8\sigma_x^2v_2+2\sigma_x^4 v_3)}{(\mu_{n}(\mathbf{K}_{>k})+\sigma_{x}^{2}+n\gamma)^{2}}\right)\mathbb{I}\left\{ \mu_{n}(\mathbf{K}_{>k})\geq\phi_{k}(n)\right\} \right]\\
 & \geq\mathbb{E}\left[\left(1-\frac{24n^2}{p}\frac{(v_1+8\sigma_x^2v_2+2\sigma_x^4 v_3)}{(\phi_{k}(n)+\sigma_{x}^{2}+n\gamma)^{2}}\right)\mathbb{I}\left\{ \mu_{n}(\mathbf{K}_{>k})\geq\phi_{k}(n)\right\} \right]\\
 & =\left(1-\frac{24n^2}{p}\frac{(v_1+8\sigma_x^2v_2+2\sigma_x^4 v_3)}{(\phi_{k}(n)+\sigma_{x}^{2}+n\gamma)^{2}}\right)\mathbb{P}\left(\mu_{n}(\mathbf{K}_{>k})\geq\phi_{k}(n)\right),
\end{align*}
where the second inequality uses Proposition \ref{prop:event_C}.
\end{proof}

\subsection{Proofs of Theorems \ref{thm:fast_eig_decay} and \ref{thm:poly_eig_decay}}\label{sec:decay_proofs}

\begin{lem}[\citealt{barzilai2023generalization}, Lemma 8]\label{lem:eigevalues_bound}
For any $\delta>0$ and any $k<\rank$, it holds that with probability at least $1-\delta$, for all $1\leq i\leq n$,
$$
\alpha_k\frac{1}{n}\text{tr}(\boldsymbol{\Lambda}_{>k})\left(1-\frac{1}{\delta}\sqrt{\frac{n^2}{R_k(\boldsymbol{\Lambda})}}\right)
\leq \mu_i\left(\frac{1}{n}\mathbf{K}_{>k}\right)\leq \beta_k\frac{1}{n}\text{tr}(\boldsymbol{\Lambda}_{>k})\left(1+\frac{1}{\delta}\sqrt{\frac{n^2}{R_k(\boldsymbol{\Lambda})}}\right).
$$
\end{lem}

\begin{proof}[
Proof of Theorem \ref{thm:fast_eig_decay}]
Note that under the assumptions of the theorem $\rank=\infty$.
In order to apply Theorems \ref{thm:VB_overall_bound} and \ref{thm:R_overall_bound}, let us first quantify the probabilities of the events $C_{p,n}^{(k)}$
and $D_{n}^{(k)}$, defined in \eqref{eq:C_k_p_n_defn}-\eqref{eq:D_k_n_defn}. By Proposition \ref{prop:prob_of_C_unconditional} applied with
$\phi_{k}(n)\coloneqq0$, 
\begin{align}
\min_{k\in\{0,\ldots,n\}}1-\mathbb{P}\left(C_{p,n}^{(k)}\right) & \geq1-\frac{24n^2}{p}\frac{\left(v_{1}+8\sigma_{x}^{2}v_{2}+2\sigma_{x}^{4}v_{3}\right)}{(\sigma_{x}^{2}+n\gamma)^{2}}\nonumber \\
 & =1-O\left(\frac{n^2}{p}\frac{\left(v_{1}+\sigma_{x}^{2}v_{2}+\sigma_{x}^{4}v_{3}\right)}{(\sigma_{x}^{2}+n\gamma)^{2}}\right),\label{eq:P(C)_fast_eig_decay}
\end{align}
where the asymptotic here and throughout the proof is for some non-decreasing
sequence $p=p(n)$, some non-increasing sequence $\gamma=\gamma(n)$ and $n\to\infty$. Under the assumption of the
theorem that $\max(\sigma_{x}^{2},\gamma)>0$, for all $0\leq k\leq n$, 
\begin{equation}
\mathbb{P}\left(D_{n}^{(k)}\right)=0.\label{eq:P(D)_fast_eig_decay}
\end{equation}

Our next step is to bound the quantity $\rho_{k,n}$ defined in \eqref{eq:rho_kn_defn}
and which appears in Theorem \ref{thm:VB_overall_bound}. By Lemma \ref{lem:eigevalues_bound}, for any $\delta_{\rho}>0$,
with probability at least $1-\delta_{\rho}$, for any $1\leq k\leq n$,
\begin{align}
\mu_{1}\left(\frac{1}{n}\mathbf{K}_{>k}\right) & \leq\beta_{k}\frac{1}{n}\mathrm{tr}(\boldsymbol{\Lambda}_{>k})+\beta_{k}\frac{\mathrm{tr}(\boldsymbol{\Lambda}_{>k})}{\sqrt{R_{k}(\boldsymbol{\Lambda})}}\frac{1}{\delta_{\rho}}\nonumber \\
 & =\beta_{k}\frac{1}{n}\mathrm{tr}(\boldsymbol{\Lambda}_{>k})+\beta_{k}\mathrm{tr}(\boldsymbol{\Lambda}_{>k})\frac{\sqrt{\mathrm{tr}(\boldsymbol{\Lambda}_{>k}^{2})}}{\mathrm{tr}(\boldsymbol{\Lambda}_{>k})}\frac{1}{\delta_{\rho}}\nonumber \\
 & \leq\left(\frac{1}{n}\mathrm{tr}(\boldsymbol{\Lambda}_{>k})+\frac{1}{\delta_{\rho}}\sqrt{\mathrm{tr}(\boldsymbol{\Lambda}_{>k}^{2})}\right)\sup_{j}\beta_{j}.\label{eq:mu_1(K)_upper_buond}
\end{align}
Throughout the remainder of the proof we take $k=k(n)\coloneqq\left\lceil (\log n)/a\right\rceil $ and all
asymptotics are as $n\to\infty$. 

It follows from (\ref{eq:mu_1(K)_upper_buond})
that with probability at least $1-\delta_{\rho}$, 

\begin{align}
\rho_{k,n} & =\frac{\|\boldsymbol{\Lambda}_{>k}\|_{2}+\mu_{1}(\frac{1}{n}\mathbf{K}_{>k})+\sigma_{x}^{2}/n+\gamma}{\mu_{n}(\frac{1}{n}\mathbf{K}_{>k})+\sigma_{x}^{2}/n+\gamma}\nonumber \\
 & \leq\frac{n\lambda_{k+1}+\left(\mathrm{tr}(\boldsymbol{\Lambda}_{>k})+\frac{n}{\delta_{\rho}}\sqrt{\mathrm{tr}(\boldsymbol{\Lambda}_{>k}^{2})}\right)\sup_{j}\beta_{j}+\sigma_{x}^{2}+n\gamma}{\sigma_{x}^{2}+n\gamma}\nonumber \\
 & =\left(O(ne^{-a(k+1)})+O\left(\frac{e^{-a(k+1)}}{1-e^{-a}}\right)+\frac{n}{\delta_{\rho}}O\left(\frac{e^{-a(k+1)}}{(1-e^{-2a})^{1/2}}\right)\right)\frac{1}{\sigma_{x}^{2}+n\gamma}+1\nonumber \\
 & =\frac{1}{\delta_{\rho}}\frac{O(1)}{\left(\sigma_{x}^{2}+n\gamma\right)}+1,\label{eq:rho_kn_exp_eig_decay}
\end{align}
where the second equality uses the assumptions of the theorem that
$\lambda_{k}=\Theta_{p}(e^{-ak})$ and $\sup_{j}\beta_{j}=O(1)$.

In its definition \eqref{eq:beta_k_defn}, $\beta_{k}$ is required only to be an upper
bound on the quantities appearing there. This upper-bound remains valid if $\beta_{k}$ is replaced by $\sup_{j}\beta_{j}$, and under the assumptions of the theorem $\sup_{j}\beta_{j}=O(1)$.
Then, using (\ref{eq:P(C)_fast_eig_decay}), (\ref{eq:P(D)_fast_eig_decay})
and $k=\left\lceil (\log n)/a\right\rceil $, the bound on $\rho_{k,n}$
in (\ref{eq:rho_kn_exp_eig_decay}) and simultaneously the bounds
on $V$ and $B$ in Theorem \ref{thm:VB_overall_bound} hold with probability at least:
\begin{equation}
1-\delta-\delta_{\rho}-\exp(-\Theta(n/k))-O\left(\frac{1}{p}\frac{\left(v_{1}+\sigma_{x}^{2}v_{2}+\sigma_{x}^{4}v_{3}\right)}{(\sigma_{x}^{2}/n+\gamma)^{2}}\right).\label{eq:fast_eig_decay_prob_proof}
\end{equation}
In particular, using (\ref{eq:rho_kn_exp_eig_decay}) and the assumption
of the theorem that $\lambda_{k}=\Theta(e^{-ak})$, the bound on $V$
from Theorem \ref{thm:VB_overall_bound}  yields:
\begin{align}
V & \leq c_{1}\rho_{k,n}^{2}\sigma_{y}^{2}\left[\frac{k}{n}+\min\left\{ \frac{r_{k}(\Lambda^{2})}{n},\left(\frac{n}{R_{k}(\Lambda)}\frac{\mathrm{tr}(\Lambda_{>k})^{2}}{\left(\alpha_{k}\mathrm{tr}(\Lambda_{>k})+\sigma_{x}^{2}+n\gamma\right)^{2}}\right)\right\} \right]\nonumber \\
 & \leq c_{1}\rho_{k,n}^{2}\sigma_{y}^{2}\left[\frac{k}{n}+\min\left\{ \frac{r_{k}(\Lambda^{2})}{n},\left(\frac{n}{R_{k}(\Lambda)}\frac{\mathrm{tr}(\Lambda_{>k})^{2}}{(\sigma_{x}^{2}+n\gamma)^{2}}\right)\right\} \right]\nonumber \\
 & =c_{1}\rho_{k,n}^{2}\sigma_{y}^{2}\left[\frac{k}{n}+\min\left\{ \frac{1}{n}\frac{\mathrm{tr}(\Lambda_{>k}^{2})}{\lambda_{k+1}^{2}},\left(n\frac{\mathrm{tr}(\Lambda_{>k}^{2})}{(\sigma_{x}^{2}+n\gamma)^{2}}\right)\right\} \right]\nonumber \\
 & =c_{1}\rho_{k,n}^{2}\sigma_{y}^{2}\left[\frac{k}{n}+\mathrm{tr}(\Lambda_{>k}^{2})\min\left\{ \frac{1}{n\lambda_{k+1}^{2}},\left(\frac{n}{(\sigma_{x}^{2}+n\gamma)^{2}}\right)\right\} \right]\nonumber \\
 & =O\left(\sigma_{y}^{2}\left(\frac{1}{\delta_{\rho}\left(\sigma_{x}^{2}+n\gamma\right)}+1\right)^{2}\left[\frac{\log n}{n}+\min\left\{ \frac{1}{n},\frac{n e^{-2ak}}{(\sigma_{x}^{2}+n\gamma)^{2}}\right\} \right]\right)\nonumber \\
 & =O\left(\sigma_{y}^{2}\left(\frac{1}{\delta_{\rho}\left(\sigma_{x}^{2}+n\gamma\right)}+1\right)^{2}\frac{1}{n}\left[\log n+\min\left\{ 1,\frac{1}{(\sigma_{x}^{2}+n\gamma)^{2}}\right\} \right]\right),\label{eq:fast_eig_decay_V_bound}
\end{align}
where the third equality uses $k=\left\lceil (\log n)/a\right\rceil $.

Using the assumptions of the theorem that $\lambda_1=\Theta(1)$ and $\lambda_{k}=\Theta_{p}(e^{-ak})$
we have 
\[
\|\theta_{>k}^{*}\|_{\Lambda_{>k}}^{2}=\sup_{j}|\theta_{j}^{*}|^{2}O_{p}(e^{-ak}),\qquad\|\theta_{\leq k}^{*}\|_{\Lambda_{\leq k}^{-1}}^{2}=\sup_{j}|\theta_{j}^{*}|^{2}O_{p}(e^{ak}),
\]
 and using (\ref{eq:rho_kn_exp_eig_decay}), together with $\beta_{k}=O_{p}(1)$,
and again $k=\left\lceil (\log n)/a\right\rceil $, the bound on $B$
from Theorem \ref{thm:VB_overall_bound} yields: 
\begin{align}
B & \leq c_{2}\rho_{k,n}^{3}\left[\frac{1}{\delta}\|\theta_{>k}^{*}\|_{\Lambda_{>k}}^{2}+\|\theta_{\leq k}^{*}\|_{\Lambda_{\leq k}^{-1}}^{2}\left(\frac{\beta_{k}\mathrm{tr}(\Lambda_{>k})}{n}+\frac{\sigma_{x}^{2}}{n}+\gamma\right)^{2}\right]\nonumber \\
 & =O\left(\sup_{j}|\theta_{j}^{*}|^{2}\left(\frac{1}{\delta_{\rho}(\sigma_{x}^{2}+n\gamma)}+1\right)^{3}\left[\frac{1}{\delta}\frac{1}{n}+n\left(\frac{e^{-ak}}{n}+\frac{\sigma_{x}^{2}}{n}+\gamma\right)^{2}\right]\right)\nonumber \\
 & =O\left(\sup_{j}|\theta_{j}^{*}|^{2}\left(\frac{1}{\delta_{\rho}(\sigma_{x}^{2}+n\gamma)}+1\right)^{3}\left[\frac{1}{\delta}\frac{1}{n}+n\left(\frac{1}{n^{2}}+\frac{\sigma_{x}^{2}}{n}+\gamma\right)^{2}\right]\right).\label{eq:fast_eig_decay_B_bound}
\end{align}

Using (\ref{eq:P(C)_fast_eig_decay}) and (\ref{eq:P(D)_fast_eig_decay}),
by Theorem \ref{thm:R_overall_bound} we have that for any $\delta_{i}>0,$ $i=1,2,3$, with probability
at least
\begin{equation}
1-\sum_{i=1}^{3}\delta_{i}-\mathbb{P}(C_{p,n}^{(0)})-\mathbb{P}(D_{n}^{(0)})\geq1-\sum_{i=1}^{3}\delta_{i}-O\left(\frac{n}{p}\frac{\left(v_{1}+8\sigma_{x}^{2}v_{2}+2\sigma_{x}^{4}v_{3}\right)}{(\sigma_{x}^{2}n^{-1/2}+n^{1/2}\gamma)^{2}}\right),\label{eq:fast_eig_decay_prob_S}
\end{equation}
the following holds:
\begin{align}
S_{1}+S_{2}+S_{3} & =O\left(\frac{v_{1}}{\delta_{1}}+\frac{\sigma_{x}^{2}v_{2}}{\delta_{2}}+\frac{\sigma_{x}^{2}(v_{2}+\sigma_{x}^{2})}{\delta_{3}}\right)\frac{\left(\sup_{z}|g(z)|^{2}+\sigma_{y}^{2}/n\right)}{(\sigma_{x}^{2}/n+\gamma)^{2}}.\label{eq:fast_eig_decay_S_bound}
\end{align}
The proof of the theorem is completed by using a union bound to combine
(\ref{eq:fast_eig_decay_prob_proof}), (\ref{eq:fast_eig_decay_V_bound}),
(\ref{eq:fast_eig_decay_B_bound}), (\ref{eq:fast_eig_decay_prob_S}) and
(\ref{eq:fast_eig_decay_S_bound}) with appropriate choice of $\delta_1,\delta_2,\delta_3$.
\end{proof}

\begin{proof}[Proof of Theorem \ref{thm:poly_eig_decay}] We begin by noting that Lemma 15 of \citealt{barzilai2023generalization} tells us that in the regime of polynomial eigenvalue decay we have both $r_k(\mbs{\Lambda}), r_k(\mbs{\Lambda}^2) = \Theta_p(k)$, and so using the fact that $R_k(\mbs{\Lambda}) = \frac{r_k(\mbs{\Lambda})^2}{r_k(\mbs{\Lambda}^2)}$ we find that $R_k(\mbs{\Lambda}) = \Theta_p(k)$ also.  Moreover, since by definition $\mr{tr}(\mbs{\Lambda}_{>k}) = r_k(\mbs{\Lambda})\cdot \lambda_{k+1}$, we see that $\mr{tr}(\mbs{\Lambda}_{>k}) = \Theta_p(k^{-(a+1)})$.  We shall rely on these bounds throughout the following arguments.

Next, observe that by the same arguments presented in the proof of Theorem \ref{thm:fast_eig_decay} we find that
\begin{align}
\min_{k\in\{0,\ldots,n\}}1-\mathbb{P}\left(C_{p,n}^{(k)}\right) & \geq1-\frac{24}{p}\frac{\left(v_{1}+8\sigma_{x}^{2}v_{2}+2\sigma_{x}^{4}v_{3}\right)}{(\sigma_{x}^{2}/n+\gamma)^{2}}\nonumber \\
 & =1-O\left(\frac{1}{p}\frac{\left(v_{1}+\sigma_{x}^{2}v_{2}+\sigma_{x}^{4}v_{3}\right)}{(\sigma_{x}^{2}/n+\gamma)^{2}}\right),\label{eq:P(C)_slow_eig_decay}
\end{align}
and for all $0\leq k\leq n$, 
\begin{equation}
\mathbb{P}\left(D_{p,n}^{(k)}\right)=0.\label{eq:P(D)_slow_eig_decay}
\end{equation}
We now proceed to bound the term $\rho_{k,n}$.  For the remainder of the proof, we shall assume that $k = \lceil n^{\frac{1}{a+2}} \rceil$ which trivially satisfies the conditions of \ref{thm:VB_overall_bound} for $n$ sufficiently large.

To establish a bound for $\rho_{k,n}$, note that Lemma 16 of \citealt{barzilai2023generalization} tells us that with probability $1 - O\left(\frac{1}{k^3}\right)\exp\left(-\frac{n}{k}\right)$ we have 
\begin{equation}\mu_1\left(\frac{1}{n}\m{K}_{>k}\right) = O(\lambda_{k+1}),\end{equation}
and consequently 
\begin{align}
\rho_{k,n} &= \frac{\|\mbs{\Lambda}_{>k}\|_2 + \mu_1\left(\frac{1}{n}\m{K}_{>k}\right) + \sigma_x^2/n + \gamma}{\mu_n\left(\frac{1}{n}\m{K}_{>k}\right) + \sigma_x^2/n + \gamma} \\[0.6em]
& \leq \frac{\lambda_{k+1} + \mu_1\left(\frac{1}{n}\m{K}_{>k}\right) + \sigma_x^2/n + \gamma}{\sigma_x^2/n +\gamma} \\[0.6em]
&= O\left(1+\frac{\lambda_{k+1}}{\sigma_x^2/n + \gamma}\right) \\[0.6em]
& = O\left(1+\frac{1}{\sigma_x^2 + n\gamma}\right)
\end{align}
with probability at least $1 - O\left(\frac{1}{n}\right)$ (since $\frac{n}{k} = n^\frac{a+1}{a+2} = \omega(\log(n))$.
Consequently, the bound on $V$ from Theorem \ref{thm:VB_overall_bound} yields:
\begin{align}
V &\leq c_{1}\rho_{k,n}^{2}\sigma_{y}^{2}\left[\frac{k}{n}+\min\left\{\frac{r_{k}(\boldsymbol{\Lambda}^{2})}{n},\frac{n}{R_k(\boldsymbol{\Lambda})}\frac{\mr{tr}(\boldsymbol{\Lambda}_{>k})^2}{(\alpha_k\mr{tr}(\boldsymbol{\Lambda}_{>k}) + \sigma_x^2+n\gamma)^2}\right\}\right]\\[0.6em]
&= O\left(\sigma_y^2\left(1+\frac{1}{\sigma_x^2 + n\gamma}\right)^2\left[\frac{k}{n}+\min\left\{\frac{k}{n},\frac{nk\lambda_{k+1}^2}{(\sigma_x^2+n\gamma)^2}\right\}\right]\right)\\[0.6em]
& = O\left(\frac{\sigma_y^2}{n^\frac{a+1}{a+2}}\left(1+\frac{1}{\sigma_x^2+ n\gamma}\right)^2\right).
\end{align}

For $B$, we note that the bound from Theorem \ref{thm:VB_overall_bound} yields:
\begin{align}
 B &\leq c_{2}\rho_{k,n}^{3}\left[\frac{1}{\delta}\|\theta_{>k}^{*}\|_{\boldsymbol{\Lambda}_{>k}}^{2}+\frac{1}{n^2}\|\theta_{\leq k}^{*}\|_{\boldsymbol{\Lambda}_{\leq k}^{-1}}^{2}\left(\beta_{k}\mathrm{tr}(\boldsymbol{\Lambda}_{>k}) +\sigma_x^2+n\gamma\right)^{2}\right] \\[0.6em]
 &= O\left(\left(1+\frac{1}{\sigma_x^2 + n\gamma}\right)^3\left(\frac{1}{\delta}\|\theta^*_{>k}\|^2_{\mbs{\Lambda}_{>k}} + \frac{1}{n^2}\|\theta^*_{\leq k}\|^2_{\mbs{\Lambda}_{\leq k}^{-1}}\left(k\lambda_{k+1} +\sigma_x^2+n\gamma\right)^2\right)\right)
\end{align}
Now, 
\begin{equation}
    \|\theta_{>k}^*\|^2_{\mbs{\Lambda}_{>k}} = \sum_{j=k+1}^\infty |\theta_j^*|^2\lambda_j 
     = O\left(\sup_{j}|\theta_{j}^{*}|^{2}\int_k^\infty x^{-(a+3)}\,dx\right) = O\left(\frac{\sup_{j}|\theta_{j}^{*}|^{2}}{n}\right)
\end{equation}
while
\begin{equation}
\|\theta^*_{\leq k}\|^2_{\mbs{\Lambda}_{\leq k}^{-1}} = \sum_{j=1}^k |\theta_j^*|^2\lambda_j^{-1} = O\left(\sup_{j}|\theta_{j}^{*}|^{2}\sum_{i=1}^k i^{a+1}\right) = O\left(\sup_{j}|\theta_{j}^{*}|^{2}n\right) 
\end{equation}
by applying Lemma 17 of \citealt{barzilai2023generalization}, and thus 
\begin{equation}
 B = O\left(\sup_{j}|\theta_{j}^{*}|^{2}\left(1+\frac{1}{\sigma_x^2 + n\gamma}\right)^3\left(\frac{1}{\delta}\frac{1}{n}+\frac{1}{n}\left(\frac{1}{n^\frac{a+1}{a+2}}+\sigma_x^2+n\gamma\right)^2\right)\right).\\[0.6em]
\end{equation}
Finally, we apply an identical argument to that found in the proof of \ref{thm:fast_eig_decay} to bound the residual terms $S_i$, and our final result follows by taking a union bound to combine these results.
\end{proof}

\end{document}